%% file: main.tex
\documentclass{article}
\usepackage{CJKutf8}

\usepackage{PRIMEarxiv}

\usepackage[utf8]{inputenc}
\usepackage[T1]{fontenc}
\usepackage{hyperref}
\usepackage{url}
\usepackage{booktabs}
\usepackage{amsfonts}
\usepackage{amssymb}
\usepackage{nicefrac}
\usepackage{amsbsy}
\usepackage{microtype}
\usepackage{lipsum}
\usepackage{fancyhdr}
\usepackage{graphicx}
\graphicspath{{media/}{figures/}}
\usepackage{makecell}
\usepackage{subcaption}
\usepackage{comment}
\usepackage{float}
\usepackage{multirow}
\usepackage{listings}
\usepackage{color}
\usepackage{colortbl}

\usepackage{graphicx}
\usepackage{booktabs}
\usepackage[table]{xcolor}

\newcommand\YAMLcolonstyle{\color{red}\mdseries}
\newcommand\YAMLkeystyle{\color{black}\bfseries}
\newcommand\YAMLvaluestyle{\color{blue}\mdseries}

\makeatletter

\newcommand\language@yaml{yaml}

\newcommand{\Edc}{\frac{1}{(1{-}\gamma)\Gamma}\sum_{s\in\mathcal S[\mathcal X]} \frac{\Gamma(s)}{T_s}\sum_{a\in\mathcal X_s}}

\expandafter\expandafter\expandafter\lstdefinelanguage
\expandafter{\language@yaml}
{
  keywords={true,false,null,y,n},
  keywordstyle=\color{darkgray}\bfseries,
  basicstyle=\YAMLkeystyle,
  sensitive=false,
  comment=[l]{\#},
  morecomment=[s]{/*}{*/},
  commentstyle=\color{purple}\ttfamily,
  stringstyle=\YAMLvaluestyle\ttfamily,
  moredelim=[l][\color{orange}]{\&},
  moredelim=[l][\color{magenta}]{*},
  moredelim=**[il][\YAMLcolonstyle{:}\YAMLvaluestyle]{:},
  morestring=[b]',
  morestring=[b]",
  literate =    {---}{{\ProcessThreeDashes}}3
                {>}{{\textcolor{red}\textgreater}}1     
                {|}{{\textcolor{red}\textbar}}1 
                {\ -\ }{{\mdseries\ -\ }}3,
}

\usepackage{enumitem}
\usepackage{amsmath}
\usepackage{amsthm}
\usepackage{mathrsfs}
\newtheorem{theorem}{Theorem}[section]
\newtheorem{lemma}[theorem]{Lemma}
\newtheorem{proposition}[theorem]{Proposition}

\theoremstyle{definition}
\newtheorem{definition}[theorem]{Definition}
\theoremstyle{remark}
\newtheorem{remark}[theorem]{Remark}
\newtheorem{claim}[theorem]{Claim}
\newtheorem{property}[theorem]{Property}

\title{Fibration Policy Optimization
}

\author{
Chang Li$^{1}$ \quad Tshihao Tsu$^{2}$ \quad Yaren Zhang$^{2}$ \quad Chao Xue$^{1}$ \quad Xiaodong He$^{1}$ \\
$^{1}$JD Explore Academy \quad $^{2}$Carleton University \\
\texttt{\{lichang93, xuechao19, xiaodong.he\}@jd.com} \\
\texttt{\{zhihaoxu, yarenzhang\}@cmail.carleton.ca}
}

\begin{document}
\maketitle

\begin{abstract}
  Large language models are increasingly trained as heterogeneous
  systems spanning multiple domains, expert partitions, and
  agentic pipelines, yet prevalent proximal objectives operate at
  a single scale and lack a principled mechanism for coupling
  token-level, trajectory-level, and higher-level hierarchical
  stability control. To bridge this gap, we derive the
  \emph{Aggregational Policy Censoring Objective} (APC-Obj), the
  first exact unconstrained reformulation of sample-based
  TV-TRPO, establishing that clipping-based surrogate design and
  trust-region optimization are dual formulations of the same
  problem. Building on this foundation, we develop \emph{Fiber
    Bundle Gating} (FBG), an algebraic framework that organizes
  sampled RL data as a fiber bundle and decomposes ratio gating
  into a base-level gate on trajectory aggregates and a
  fiber-level gate on per-token residuals, with provable
  first-order agreement with the true RL objective near
  on-policy. From APC-Obj and FBG we derive
  Fibration Policy Optimization (or simply, FiberPO), a concrete
  objective whose Jacobian is block-diagonal over trajectories,
  reduces to identity at on-policy, and provides better update
  direction thus improving token efficiency. The compositional
  nature of the framework extends beyond the trajectory-token
  case: fibrations compose algebraically into a \emph{Fibration
    Gating Hierarchy} (FGH) that scales the same gating mechanism
  to arbitrary hierarchical depth without new primitives, as
  demonstrated by \emph{FiberPO-Domain}, a four-level
  instantiation with independent trust-region budgets at the
  domain, prompt group, trajectory, and token levels. Together,
  these results connect the trust-region theory, a compositional
  algebraic structure, and practical multi-scale stability
  control into a unified framework for LLM policy optimization.
\end{abstract}

\section{Introduction}
\label{sec:intro}

\input{rlvr_intro}
\input{rlvr_trpo}
\input{rlvr_acpo}
\input{rlvr_fiberpo}
\input{rlvr_fiberpo_sec3}

\input{rlvr_fiberpo_sec4}
\input{rlvr_conclusion}

\appendix
\newpage
\input{rlvr_appendix_notations}

\input{rlvr_appendix_trpo_acpo}
\input{rlvr_appendix_acpo_equiv}
\input{rlvr_appendix_fiberpo}

\bibliographystyle{unsrt}
\bibliography{references}

\end{document}

%% file: rlvr_intro.tex
Large language models are no longer single, monolithic policies:
they are increasingly deployed and trained as heterogeneous
systems—agentic pipelines spanning domains and tools,
mixture-of-experts (MoE) architectures with conditional routing,
and distributed/asynchronous training stacks where optimization
noise and data nonstationarity are structural rather than
incidental. In this regime, alignment via
RLHF~\cite{ouyang2022training} must simultaneously handle
multi-scale instability: token-level stochasticity,
trajectory-level drift, and system-level heterogeneity
(domains/experts/agents) interacting in the same update. Existing
PPO-style ``proximal''
objectives~\cite{schulman2017proximal,shao2024deepseekmath,yu2025dapo}
provide only coarse local controls (mostly per-token clipping)
and limited diagnostics when failures arise from global structure
(e.g., a drifting subset of trajectories, an expert partition, or
a domain slice). This motivates importing more expressive
mathematical structure, beyond new loss heuristics, to build
controllers that can allocate stability budgets across the
relevant global contexts, composably across arbitrarily many
hierarchical levels (tokens, trajectories, prompt groups,
domains etc). In this work, we develop Fiber Bundle Gating
(FBG), an algebraic framework grounded in fiber bundle theory, and
derive FiberPO-Trajectory (or simply, FiberPO) from it, a
concrete policy optimization objective that decomposes
trust-region maintenance into compositional global and local
components, providing multi-scale stability control with
first-order fidelity to the true RL objective near on-policy and
a restorative gradient structure rarely explored in existing
methods (see also ~\cite{wang2020truly}).

The development proceeds through four stages. We begin from
TRPO~\cite{schulman2015trust}, whose trust-region radius depends
on the discount factor $\gamma$. LLM RL effectively requires
$\gamma = 1$ (rewards are sparse and determined only at
completion), and we prove that both TV- and KL-based TRPO trust
regions then collapse to the reference policy, permitting only
trivial updates (Theorem~\ref{thm:trpo_vanishing}). This does not
preclude trust-region-style stabilization in the $\gamma = 1$
setting, but it shows that the classical radius from TRPO cannot
be used as-is, and a useful trust-region design must decouple the
trust-region maintaining mechanism from the particular radius prescribed by
the monotonic improvement guarantee.

To make this decoupling precise, we observe that a more
fundamental gap must first be addressed. Clipping-based
surrogates such as PPO~\cite{schulman2017proximal} were
originally motivated as tractable approximations to TRPO's
constraint, yet the precise relationship between ratio clipping
and trust-region maintenance remains insufficiently understood,
leaving open whether clipping merely imitates a trust region or
can exactly reproduce one. We attempt to close this gap by
deriving Aggregational Policy Censoring Objective (APC-Obj), an
unconstrained clipping-based surrogate whose per-action clip
bound $B_{s,a}$ decomposes the TV trust-region constraint into an
explicit, cross-action-coupled form. We prove a formal
equivalence (Theorem~\ref{thm:apc_obj_trpo_equiv}): under standard
function approximation assumptions, maximizing APC-Obj yields the
same policy update as sample-based TV-TRPO, establishing that
clipping-based surrogate design and trust-region policy
optimization are dual formulations of the same optimization
problem. Although APC-Obj itself yields trivial updates at
$\gamma{=}1$ (the same vanishing as TRPO), its contribution is
structural rather than algorithmic: it separates the trust-region
maintenance mechanism (cross-action-coupled clipping) from the
specific radius prescribed by the classical bound, so that the
mechanism remains well-defined at any $\delta > 0$. This
separation provides a reliable analytical anchor from which
PPO~\cite{schulman2017proximal},
GRPO~\cite{shao2024deepseekmath}, and GSPO~\cite{zheng2025group}
can each be formally derived via identified relaxation steps,
using a shared Ratio Gating Formalism (RGF), making each method's
departure from the trust-region optimum explicit and traceable.

The APC-Obj-based taxonomy reveals a structural gap shared by all
existing methods: token-wise objectives (PPO, GRPO) gate each
ratio independently without directly bounding trajectory-level
drift, while sequence-wise objectives (GSPO) collapse each
trajectory to a single aggregate, suppressing within-trajectory
variation. The relevant stability criterion, trajectory-level TV
divergence, is inherently global, yet gradient information is
inherently local. Few existing methods couple the two scales. To
close this gap, we propose Fiber Bundle Gating (FBG), an algebraic
framework that organizes sampled RLHF data as a fiber bundle,
with tokens as the total space and global contexts (trajectories,
partitioned into positive and negative drift channels) as the
base space. FBG decomposes ratio gating into several operations:
pushing token-level information to the base to form global
aggregates, applying a base-level gate to maintain a trust-region
budget at the context level, reflecting the gated signal back to
tokens via a Markov kernel, and gating fiber-level residuals to
preserve token-level gradients. We prove that this decomposition
preserves first-order agreement with the true RL objective near
on-policy whenever the atomic gates reduce to identity at the
reference point (Theorem~\ref{thm:fbg_first_order}). We note that
the fibration structure here is not an external import but is
inherent in the data: methods that aggregate per-token quantities
into context-level statistics implicitly invokes a pushforward
$\pi_{E*}$, and methods that distribute a context-level signal
back to individual tokens implicitly constructs a Markov
kernel~$K$. Making these operations explicit exposes the
reflecting condition $\pi_{E*} \circ K =
\mathrm{id}_{\mathbf{B}}$, a structural constraint ensuring that
global and local gates operate on orthogonal components without
double-counting (Appendix~\ref{app:fbg_transmit}). Any gating map
that avoids this double-counting while preserving first-order
agreement necessarily satisfies the reflecting condition, and
therefore implicitly re-derives the fibration decomposition.

Building on APC-Obj and FBG, we derive FiberPO by starting from a
$\delta$-relaxed APC-Obj formulation and decomposing its coupled
clipping constraint into two FBG components: a base-level
aggregate gate $g^{\rm agg}$ that allocates a global trust-region
budget $\delta$ across trajectory-level drift, and a fiber-level
$\operatorname{logclip}$ that bounds each token's residual
deviation by $\varepsilon$. This decomposition inherits the
trust-region mechanism of APC-Obj while providing independent
control at two scales, an explicit budget separation is rarely
explored in prior methods. The resulting Jacobian is
block-diagonal over trajectories, reduces to identity at
on-policy, and exhibits a restorative gradient in the rollback
regime that actively corrects trajectory drift. This restorative
property is absent in PPO, GRPO, and GSPO, which either zero the
gradient (PPO/GRPO clipping) or suppress it uniformly (GSPO
gating) when a trajectory drifts beyond the clipping boundary.

The hierarchy of domains, prompt groups, trajectories, and tokens
motivates a framework that scales beyond two levels. The fiber
bundle formalism naturally supports this extension. Because fibrations
compose algebraically, the same FBG gating mechanism extends to
deeper hierarchies by chaining fibrations into a Fibration Gating
Hierarchy (FGH). We demonstrate this by deriving FiberPO-Domain,
a four-level instantiation (domain, prompt group, trajectory,
token) that applies $g^{\rm agg}$ independently at each level of
the hierarchy, with the gated residual at each level capturing
only the deviation from the next-coarser aggregate. This
construction requires no new gating primitives: the same
algebraic decomposition that produces FiberPO at two levels
produces FiberPO-Domain at four, providing per-domain and
per-prompt-group trust-region control that is absent in all
existing methods. The compositionality of fibrations is what allows the framework
to scale to the hierarchical, multi-domain training regimes that
modern LLM systems increasingly demand.

In summary, this paper contributes:
\begin{itemize}
    \item Aggregational Policy Censoring Objective (APC-Obj), the
      first exact unconstrained reformulation of sample-based
      TV-TRPO, with a formal equivalence proof
      (Theorem~\ref{thm:apc_obj_trpo_equiv}). APC-Obj establishes
      clipping-based surrogate design and trust-region policy
      optimization as dual formulations of the same problem, and
      provides a reliable analytical anchor from which PPO, GRPO, and
      GSPO are each derived via identified relaxations.
    \item Fiber Bundle Gating (FBG) and its hierarchical
      generalization, Fibration Gating Hierarchy (FGH), a
      compositional algebraic framework that couples global
      (base) and local (fiber) stability control through
      density-based gating on a fiber bundle, with a first-order
      agreement guarantee near on-policy
      (Theorem~\ref{thm:fbg_first_order}). The algebraic
      compositionality of fibrations allows the same construction
      to extend to arbitrary hierarchical depth without
      introducing new gating primitives.
    \item FiberPO-Trajectory, a concrete FBG instantiation
      (trajectory, token) derived from APC-Obj that decomposes
      trust-region control into a base-level aggregate gate
      (budget $\delta$) and a fiber-level residual gate (budget
      $\varepsilon$), with a block-diagonal, restorative Jacobian
      structure.
    \item FiberPO-Domain, a four-level FGH instantiation (domain,
      prompt group, trajectory, token) that extends
      FiberPO-Trajectory to multi-domain training, providing
      independent trust-region budgets at each hierarchical
      level.
\end{itemize}

%% file: rlvr_trpo.tex
\section{TRPO and the Discount Factor Obstruction}
\label{subsec:trpo}

The monotonic-improvement framework of Kakade and Langford~\cite{kakade2002approximately}, refined by TRPO~\cite{schulman2015trust}, remains the primary theoretical foundation for trust-region policy optimization.
TRPO guarantees monotonic improvement by constraining each policy update to a trust region defined via TV or KL divergence, with a trust-region radius governed by the discount factor $\gamma$ through a surrogate-gap bound.
In this section, we show that this $\gamma$-dependence creates a fundamental obstruction for episodic LLM RL, where the effective discount factor is $\gamma = 1$.
The classical surrogate-gap bound (Appendix~\ref{app:tv_trpo}) shows that the gap between the true RL objective $J(\theta)$ and its linear surrogate $J^{(1)}(\theta\,|\,\theta_{\rm old})$ satisfies
\begin{equation}\label{eq:trpo_bound}
    J(\theta) - J^{(1)}(\theta\,|\,\theta_{\rm old})
    \;\geq\;
    -\,\frac{4\gamma\,\|A^{(\theta_{\rm old})}_\bullet\|_\infty}{(1-\gamma)^2}\;
    D_{\rm TV}^{\max}(\theta\,\|\,\theta_{\rm old})^2,
\end{equation}
which confines TRPO updates to a trust region of radius $\delta^{(\rm TRPO)} = (1{-}\gamma)/(8\gamma)$ (Lemma~\ref{lem:trpo_update_bounded}).
In LLM RL, later tokens must not be discounted, since the correctness of a response is often determined only at completion, so the training objective effectively requires $\gamma = 1$.
Inspecting the bound~\eqref{eq:trpo_bound}, the penalty coefficient $\gamma/(1-\gamma)^2$ diverges as $\gamma \to 1$, forcing $\delta^{(\rm TRPO)} \to 0$ to maintain the guarantee.
We formalize this observation as a vanishing theorem:

\begin{theorem}[TRPO vanishing theorem]\label{thm:trpo_vanishing}
When $\gamma = 1$, both the TV-based and KL-based TRPO trust
regions collapse to the reference policy (Definition~\ref{def:trust_region}):
\[
    \mathcal{B}^{\rm TV\text{-}TR}_{\delta^{(\rm TRPO)}}(\theta_{\rm old})
    \;=\;
    \mathcal{B}^{\rm KL\text{-}TR}_{\delta^{(\rm TRPO)}}(\theta_{\rm old})
    \;=\;
    \{\,\pi_{\theta_{\rm old}}\,\}.
\]
Consequently, the only policy update permitted by TV-based TRPO's
 trust region at $\gamma=1$ is the trivial one:
$\pi_{\theta_{new}}=\pi_{\theta_{old}}$.
\end{theorem}

\noindent Proof: see Appendix~\ref{app:vanishing_proof}.

Theorem~\ref{thm:trpo_vanishing} does not imply that
trust-region-style stabilization is impossible in the $\gamma = 1$
setting. It shows that the specific radius from TRPO's discounted
analysis cannot be used as-is. A natural workaround, annealing
$\gamma < 1$ toward $1$ during training~\cite{kim2022adaptive},
is impractical for LLMs, whose rewards are typically sparse and
all-or-nothing, making discounted reward estimation unstable (see
Appendix~\ref{app:gamma_annealing} for discussion). Importantly,
$\delta^{(\rm TRPO)} = 0$ at $\gamma = 1$ does not preclude the
existence of a useful TV bound $\delta > 0$ that controls update
stability in the $\gamma = 1$ case. In
Section~\ref{subsec:rgf_apc_obj}, we provide both theoretical and
empirical evidence for such a bound through the lens of
GSPO~\cite{zheng2025group}. The vanishing theorem thus motivates
relaxing $\delta$ into a tunable hyperparameter while retaining
the trust-region structure as a design guide for objective functions. This relaxed
formulation is the common starting point for
PPO~\cite{schulman2017proximal},
GRPO~\cite{shao2024deepseekmath}, GSPO~\cite{zheng2025group}, and
our FiberPO, as we shall show.

\noindent
To make the connection between
TRPO's trust-region mechanism and these practical methods
precise, we seek an unconstrained reformulation of sample-based
TV-TRPO (Appendix~\ref{app:tv_trpo}) that admits an explicit
Ratio Gating Formalism (RGF, Section~\ref{subsubsec:rgf}).
Because PPO, GRPO, and GSPO each also admit RGF representations,
casting TRPO into the same language will let us identify the
exact relaxation steps that transform a trust-region-enforced
source objective into each practical method.

%% file: rlvr_acpo.tex
\section{Ratio Gating Formalism and APC-Obj}
\label{subsec:rgf_apc_obj}

The vanishing theorem (Theorem~\ref{thm:trpo_vanishing}) shows
that TRPO's trust-region radius collapses at $\gamma=1$, but the
trust-region structure itself remains a valuable design guide. To
retain this structure, we first introduce the Ratio Gating
Formalism (RGF), a unified framework that captures a broad family
of proximal objectives by concentrating algorithmic choices into
a single ratio gating map. We then use RGF to reformulate sample-based TV-TRPO (Appendix~\ref{app:tv_trpo})
as an unconstrained clipping-based surrogate, Aggregational Censoring Policy
Optimization (APC-Obj), which enforces
the same policy update as the sample-based TV-TRPO objective (Theorem~\ref{thm:apc_obj_trpo_equiv}; see Appendix~\ref{app:apc_obj_equiv} for the full proof).
The motivation for casting this TRPO-equivalent objective
specifically in RGF form is that PPO, GRPO, and GSPO each also
admit explicit RGF representations
(Appendix~\ref{app:xxpo_rgf}). Expressing all methods in this
shared language makes the relationship between the
trust-region-enforced source (APC-Obj) and each practical method more transparent.
Starting from this
trust-region-enforced source objective, we show that PPO, GRPO,
and GSPO each arise through identifiable relaxations that
preserve the clipping structure and thereby inherit (partially) APC-Obj's
TV-penalty mechanism. We further establish that GSPO's clipping
implicitly maintains a TV-like trust region, providing empirical
evidence that a useful bound of TV distance exists even at
$\gamma=1$.

\subsection{The Ratio Gating Formalism}
\label{subsubsec:rgf}

\begin{definition}[Ratio Gating Formalism (RGF)]\label{def:rgf}
An RGF surrogate objective takes the form
\begin{equation}\label{eq:rgf}
    \hat{J}(\theta\,|\,\theta_{\rm old})
    \;=\;
    \sum_{(s,a,\mathcal{I})\,\in\,\mathcal{E}}
    \mu_{s,a,\mathcal{I}}\;
    \mathcal{G}\!\bigl(r_\bullet(\theta\,|\,\theta_{\rm old})\bigr)_{s,a,\mathcal{I}}\;
    \hat{A}^{\theta_{\rm old}}_{s,a},
\end{equation}
where:
\begin{itemize}
    \item $\mathcal{E}$ is the sampled index set of state--action pairs augmented with extra information $\mathcal{I}$ (e.g., trajectory membership, timestep);
    \item $\mu_{s,a,\mathcal{I}}$ are predetermined aggregation weights;
    \item $r_\bullet(\theta\,|\,\theta_{\rm old}) = \{r_{s,a}\}$ with $r_{s,a} := \pi_\theta(a\,|\,s)\,/\,\pi_{\theta_{\rm old}}(a\,|\,s)$ are the importance-sampling ratios;
    \item $\mathcal{G}: \mathbb{R}_{>0}^{\mathcal{E}} \to \mathbb{R}_{>0}^{\mathcal{E}}$ is the \textbf{ratio gating map}, the central design object that transforms ratio tuples into gated ratio tuples;
    \item $\hat{A}^{\theta_{\rm old}}_{s,a}$ is an estimate of the advantage function under $\pi_{\theta_{\rm old}}$.
\end{itemize}
\end{definition}

\noindent
RGF organizes a wide family of surrogate objectives, including
PPO~\cite{schulman2017proximal},
GRPO~\cite{shao2024deepseekmath}, GSPO~\cite{zheng2025group},
APC-Obj, and FiberPO.
Because all these methods share the RGF form, the differences
between them reduce entirely to their choice of
$(\mathcal{E},\,\mu,\,\mathcal{G})$. In particular, deriving a
trust-region-equivalent surrogate (APC-Obj) in RGF form lets us
identify, for each practical method, the precise relaxation steps
that transform the trust-region-enforced APC-Obj gating into the
method's own gating map. 
Appendix~\ref{app:xxpo_rgf} gives the explicit
$(\mathcal{E},\,\mu,\,\mathcal{G})$ specifications for each
method.

\subsection{Aggregational Policy Censoring Objective (APC-Obj)}
\label{subsubsec:apc_obj}

We now construct a concrete RGF instance that is provably equivalent to sample-based TV-TRPO (Appendix~\ref{app:tv_trpo}).
Aggregational Policy Censoring Objective (APC-Obj) is obtained by solving the sample-based surrogate-gap objective~\eqref{eq:trpo_bound} in closed form, yielding an unconstrained clipping-based surrogate whose clip bounds explicitly allocate the TV trust-region budget $T_s\delta^{(\rm APC\text{-}Obj)}$ in a cross-action-coupled form.

\begin{definition}[Aggregational Policy Censoring Objective, APC-Obj policy iteration\footnotemark]\label{def:apc_obj_main}
The APC-Obj update selects
$\theta_{\rm new} = \arg\max_\theta\,
  \hat{J}^{\rm APC\text{-}Obj}(\theta|\theta_{\rm old})$, where
\begin{equation}\label{eq:apc_obj_main_obj}
    \hat{J}^{\rm APC\text{-}Obj}(\theta|\theta_{\rm old})
    \;=\;
    \frac{1}{T}\sum_{(s,a,\tau,t) \in \bar{\mathcal{X}}}
    \left[
        \operatorname{clip}\!\Bigl(r_{s,a}-1,\;
        T_s\delta^{(\rm APC\text{-}Obj)}
        - \!\!\!\!\!\!
        \sum_{\substack{(s,a',\tau',t') \in \bar{\mathcal{X}}_s \\
          (a',\tau',t') \neq (a,\tau,t)}}
        \!\!\!\!\!\!|r_{s,a'}-1|\Bigr)
        \;\hat{A}^{\theta_{\rm old}}_{s,a}
        \;+\; \hat{A}^{\theta_{\rm old}}_{s,a}
    \right].
\end{equation}
\end{definition}
\footnotetext{The APC-Obj objective can equivalently be written in the RGF form; see Appendix~\ref{app:apc_obj_rgf}.}

\noindent
$\mathcal{X}$ denotes the set of all state--action pairs in the sampled trajectories, with $n_{s,a}$ the sample
multiplicity of pair $(s,a)$.  For each state
$s \in \mathcal{S}[\mathcal{X}]$, let
$\mathcal{X}_s := \{a \in \mathcal{A} \mid (s,a) \in \mathcal{X}\}$
be the set of actions co-occurring with $s$.  Define
$T := \sum_{(s,a) \in \mathcal{X}} n_{s,a}$ (total sample count) and
$T_s := \sum_{a \in \mathcal{X}_s} n_{s,a}$ (per-state count).
We also denote by $\bar{\mathcal{X}}$ the augmented space
$\{(s,a,\tau,t)\}$ of all sampled state--action pairs augmented with
trajectory membership $\tau$ and time step $t$. Each element of
$\bar{\mathcal{X}}$ is a distinct sampled token, so
$|\bar{\mathcal{X}}| = T$ and $|\bar{\mathcal{X}}_s| = T_s$.
$\bar{\mathcal{X}}$ is an example of an augmented index set $\mathcal{E}$ used in
Definition~\ref{def:rgf}. $\delta^{(\rm APC\text{-}Obj)} :=
  \frac{(1-\gamma)^2}{8\gamma\,\|\hat{A}^{\theta_{\rm old}}_\bullet\|_\infty}
  \,\frac{M(\hat{A}^{\theta_{\rm old}}_\bullet)}{T}$
(the argument $\hat{A}^{\theta_{\rm old}}_\bullet$ makes the dependence on the advantage explicit),
the per-entry clip bound is
$B_{s,a,\tau,t} := T_s\delta^{(\rm APC\text{-}Obj)} - \sum_{\substack{(s,a',\tau',t') \in \bar{\mathcal{X}}_s,\, (a',\tau',t') \neq (a,\tau,t)}} |r_{s,a'}-1|$,
and $\operatorname{clip}(a, B) := \operatorname{clip}(a, -B^+, B^+)$
with $B^+ := \max(B, 0)$. $M$ depends on $\hat{A}^{\theta_{\rm old}}$ and is bounded by $\|\hat{A}^{\theta_{\rm old}}_\bullet\|_\infty$, see Lemma~\ref{lem:M_bound}. See Appendix~\ref{app:apc_obj_rgf}  and Appendix~\ref{app:apc_obj_equiv} for more details.

\noindent

\noindent
The per-entry clip bound $B_{s,a,\tau,t}$ can be interpreted as:
the per-state trust-region budget $T_s\delta^{(\rm APC\text{-}Obj)}$ minus the
TV contribution already consumed by all other entries at the same
state.  This cross-entry coupling is what enforces a per-state TV
constraint in aggregate, rather than merely bounding each token
independently.
When the remaining budget is exhausted (i.e.\ $B_{s,a,\tau,t} \leq 0$),
the deviation $r_{s,a}-1$ is
clipped to zero, censoring that entry's contribution to the
surrogate.

Despite their superficially different structures (APC-Obj applies
per-token clipping with cross-action coupling, while sample-based TV-TRPO
maximizes a globally penalized objective), the two algorithms
produce the same policy update.

\begin{theorem}[Sample-based TV-TRPO and APC-Obj equivalence; restatement of Theorem~\ref{thm:apc_obj_trpo_equiv}]\label{thm:apc_obj_trpo_equiv_main}
  Suppose $\theta$ satisfies standard function approximation
  assumptions, APC-Obj and sample-based TV-TRPO
  (Definition~\ref{def:sample_tv_trpo}) produce the same policy
  update:
\[
    \pi_{\theta_{\rm new}^{(\rm APC\text{-}Obj)}} \;=\; \pi_{\theta_{\rm new}^{(\rm TV\text{-}TRPO)}}.
\]
\end{theorem}

The proof (Appendix~\ref{app:apc_obj_equiv}) proceeds in three
stages: (i)~we derive the maximizer of sample-based TV-TRPO
(Theorem~\ref{thm:tv_trpo_max}), showing that the optimal ratio
deviation is a scalar multiple of the per-state unit-TV maximizer,
(ii)~we establish a retraction property
(Lemma~\ref{lem:retraction}), showing that APC-Obj's clipping
mechanism projects any trust-region violation strictly back
inside the feasible set, and (iii)~we combine these to show that
the APC-Obj maximizer must lie inside the trust region, at which
point its objective reduces to the same linear program solved by
sample-based TV-TRPO. \qed

\begin{remark}[APC-Obj as the structural bridge for $\delta$-relaxation]\label{rem:trust_region_design}
The vanishing theorem (Theorem~\ref{thm:trpo_vanishing}) establishes that $\delta^{(\rm TRPO)} = \delta^{(\rm APC\text{-}Obj)} = 0$ at $\gamma=1$, rendering both TRPO and APC-Obj trivial as deployed algorithms.
However, APC-Obj's significance is structural rather than algorithmic: its RGF form cleanly separates the clipping mechanism (cross-action-coupled clipping, which remains well-defined for any $\delta > 0$) from the specific trust-region enforcing \emph{radius} $\delta^{(\rm APC\text{-}Obj)}$ inherited from the TRPO's strict stability requirement.
This separation is what makes $\delta$-relaxation---replacing the vanishing $\delta^{(\rm APC\text{-}Obj)}$ with a positive tunable hyperparameter---a precise, traceable operation rather than an ad hoc modification.
The GSPO analysis below (Claim~\ref{claim:gspo_tv}, Remark~\ref{rem:gspo_delta}) then provides the evidence that a specific positive $\delta$ controls policy drift at $\gamma=1$, supporting the motivation for relaxation.
APC-Obj therefore serves as a \emph{design guide} for surrogate construction: PPO~\cite{schulman2017proximal}, GRPO~\cite{shao2024deepseekmath}, and GSPO~\cite{zheng2025group} can each be recovered from APC-Obj through identifiable relaxation steps (Appendix~\ref{app:xxpo_derivation}).
\end{remark}

\noindent
Among these relaxations, GSPO is of particular interest because it directly addresses the open question raised by the vanishing theorem: \emph{does a positive  TV trust-region radius $\delta > 0$ effectively maintain stability in training at $\gamma = 1$?}
GSPO replaces $\delta^{\rm (APC-Obj)}$ with a positive tunable hyperparameter $\epsilon^{\rm (GSPO)}$ and introduces a direct connection between its clipping threshold and the TV trust-region radius, providing both theoretical and empirical evidence for an affirmative answer.

\begin{claim}\label{claim:gspo_tv}
GSPO (without minimization function) maintains approximately an average total-variation distance bound at the trajectory level with $\bar{D}_{\rm TV}^{\rm (tj)} \leq \delta \sim \sqrt{\epsilon^{\rm (GSPO)}/2}$. For definition of $\bar{D}_{\rm TV}^{\rm (tj)}$ see Appendix~\ref{app:divergences}.
\end{claim}

\noindent Proof: see Appendix~\ref{app:gspo_tv_proof}. \qed

\begin{remark}\label{rem:gspo_delta}
In practice, GSPO uses $\epsilon^{\rm (GSPO)} \approx 4 \times 10^{-4}$~\cite{zheng2025group}, which translates to $\delta \approx \sqrt{2 \times 10^{-4}} \approx 0.014$.
GSPO runs stably at this scale, providing empirical evidence that a TV trust region with $\delta \sim 10^{-2}$ remains effective for maintaining stability in LLM RL at $\gamma = 1$.
This confirms that $\mathcal{B}^{\rm TV\text{-}TR}_\delta$ is a meaningful design criterion even when TRPO's own radius has vanished: the vanishing theorem (Theorem~\ref{thm:trpo_vanishing}) shows that relaxation is \emph{necessary} to have nontrivial policy 
improvement, while in practice the relaxation is totally reasonable.
\end{remark}

\noindent
However, all methods derived from APC-Obj via these relaxations operate at a single scale: token-wise objectives (PPO, GRPO) gate each ratio independently without directly bounding trajectory-level drift, while sequence-wise objectives (GSPO) collapse each trajectory to a single aggregate and suppress within-trajectory variation. The relevant stability criterion, the trajectory-level average TV divergence $\bar{D}_{\rm TV}^{\rm (tj)}$ (Appendix~\ref{app:divergences}), is defined by aggregating per-token ratios over a single trajectory, so neither independent token-wise clipping nor trajectory-level gating alone provides a principled coupling between the two scales.
Because APC-Obj supplies a trust-region-enforced source objective, it can be composed with more structured gating mechanisms that operate across scales while retaining the underlying TV-based stability guarantee. In the next section, we introduce Fiber Bundle Gating (FBG) to address this gap.

%% file: rlvr_fiberpo.tex
\section{Fiber Bundle Gating Framework}
\label{subsec:fbg}

The APC-Obj derivations in Section~\ref{subsec:rgf_apc_obj} provide a trust-region-enforced source objective and reveal that, token-wise proximal objectives (PPO, GRPO) gate
each token's importance ratio independently without directly bounding trajectory-level
drift, while trajectory-wise objectives (GSPO, Claim~\ref{claim:gspo_tv}) collapse each trajectory to a single aggregate
ratio, reducing trajectory-level variation.
Neither paradigm offers a direct coupling between these two scales. 
In this section, we introduce an algebraic framework that addresses this limitation by exploiting
the natural two-level structure of sampled RLHF data.
We first organize sampled token data as a fiber bundle (Section~\ref{subsubsec:fiber_bundle}),
and then define Fiber Bundle Gating (FBG), a density-based gating operator that decomposes
stability control into explicit global and local components while preserving first-order
agreement with the true RL objective near on-policy. (Section~\ref{subsubsec:fbg}).

\subsection{Fiber Bundle Model}
\label{subsubsec:fiber_bundle}

Sampled RLHF data has a natural two-level structure: every token belongs to a
trajectory or global domain context. Global information arises by aggregating local quantities over a chosen domain or context; for example, the trajectory-level average total variation distance $\bar{D}_{\rm TV}^{\rm (tj)}$ (Appendix~\ref{app:divergences}) aggregates per-token ratios within each trajectory.
The notions of ``local information'' (per-token ratios) and ``global information'' (trajectory-level aggregates) are intuitive but imprecise; to formalize them and, crucially, to describe how they interact under gating, we represent both as \emph{densities on a manifold}.
Fiber bundle theory is an expressive mathematical language for this type of hierarchical
decomposition.
A fiber bundle $\pi_E: E \to B$ separates a total space $E$ of local data points from a base
space $B$ of global contexts; for each context $b \in B$, the fiber $\pi_E^{-1}[\{b\}]$
collects exactly the local data belonging to $b$. 
We organize sampled RLHF data as a fiber bundle of this form:
policy ratios are first converted to densities on the total space
$E$ via a fibration decomposition map $\mathcal{F}$; all gating
operations (global and local) are performed directly on these
densities; and a recovery map $\mathcal{R}$ converts the gated
densities back to policy ratios for use in the surrogate objective.  (see Appendix~\ref{app:fiber_bundle}, \ref{app:density}, and \ref{app:fbg} for a formal
definition).

To this end, we first define the augmented sampled space (Appendix~\ref{app:estimation}):
\[
    \bar{\cal X}:=\{ (s_t(\tau),a_t(\tau),\tau,t) \}\subseteq \mathcal X\times \mathrm{Tj}^{\theta_{\rm old}}\times \mathbb N.
\]
Each element of $\bar{\mathcal X}$ is a sampled state-action pair together with its trajectory membership and timestep. The space $\bar{\mathcal X}$ serves as the index set $\mathcal E$ in the RGF form (Definition~\ref{def:rgf}).

\begin{remark}
For mathematical completeness, even if a trajectory terminates at a finite step, we include a null state-action pair $(\varnothing,\varnothing,\tau,t)$ after termination.
\end{remark}

\noindent
The fiber bundle in our work consists of the following objects:

\paragraph{Base space ($B$).}
The base space collects all classes or contexts.
In this work, we define
\[
    B := \mathrm{Tj}^{\theta_{\rm old}} \times \{-1, +1\},
\]
so a typical element of the base is a tuple $(\tau, +1)$ or $(\tau, -1)$.

\paragraph{Total space ($E$).}
The total (bundle) space consists of all individual token-level data points together with a sign label.
We define
\[
    E := \bar{\mathcal X}\times \{-1,+1\}.
\]
The sign channel $\{-1,+1\}$ means positive and negative log-ratio contributions.

\paragraph{Bundle projection ($\pi_E$).}
The projection $\pi_E: E \to B$ maps each augmented token to its corresponding context:

$$
\pi_E:\bar{\mathcal X}\times \{-1,+1\}\to {\rm Tj}^{{\theta_{\rm old}}}\times\{-1,+1\}
\qquad
\pi_E(s,a,\tau,t;\,l):=(\tau,l)
$$

sending an augmented state-action pair to its trajectory while preserving the sign channel, which means a point in the base $B$ (e.g., $(\tau, 1)$) will either correspond to the tokens sampled in $\tau$ with positive log-ratio contribution or the tokens sampled in $\tau$ with negative log-ratio contribution, depending on the sign label $l$ ($l$ is either $+1$ or $-1$).

\paragraph{Fiber.}

For a given (trajectory, sign) tuple $b \in B$, the fiber $E_b := \pi_E^{-1}[\{b\}]$ contains all the local points belonging to the class/context $b$.
In our instantiation, $\pi_E^{-1}[\{(\tau,l)\}]$ collects \textbf{all} sampled tokens from the trajectory $\tau$ with sign label $l$.

\medskip
\noindent

\begin{remark}[Choice of base space and extensibility]\label{rem:base_extensibility}
The choice of base space $B$ determines the granularity of global control.
Here we instantiate FBG on a topologically trivial (product) fiber bundle with $B = \mathrm{Tj}^{\theta_{\rm old}} \times \{-1,+1\}$, where the base gate controls trajectory-level drift.
The framework extends to richer global structure by composing fibrations into a hierarchy: enriching $B$ with domain and prompt group indices produces per-domain and per-group aggregates, as realized concretely in FiberPO-Domain (Section~\ref{subsubsec:domain_fiberpo}).
In each case, the fiber $\pi_E^{-1}[\{b\}]$ collects exactly those tokens belonging to class $b$, and the local-global decomposition separates class-level drift from within-class variation.
See Remark~\ref{rem:base_choice_app} for further discussion.
\end{remark}

\begin{figure}[htbp]
\centering
\includegraphics[width=\textwidth]{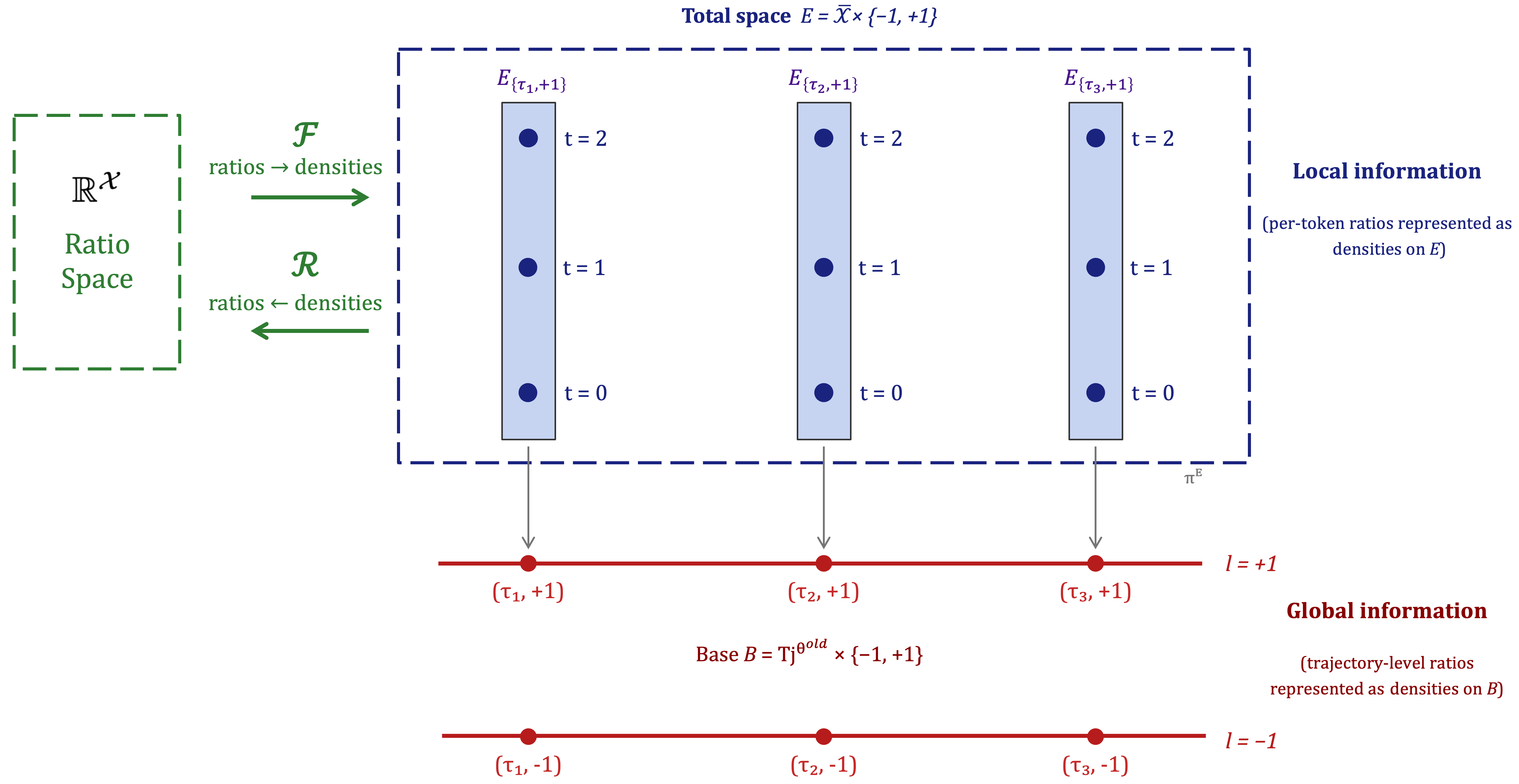}
\caption{Fiber bundle model for sampled RLHF data.
The base space $B=\mathrm{Tj}^{\theta_{\rm old}}\!\times\!\{-1,+1\}$ consists of two separate lines indexed by the sign label $l\in\{-1,+1\}$, encoding trajectory-level global information as densities on $B$.
Each fiber $\pi_E^{-1}[\{(\tau,l)\}]$ collects the per-token data points belonging to trajectory $\tau$ with sign $l$, encoding local information as densities on the total space $E$.
The maps $\mathcal{F}$ and $\mathcal{R}$ convert between policy ratios and densities.}
\label{fig:fiber_bundle}
\end{figure}

\noindent
Figure~\ref{fig:fiber_bundle} illustrates this fiber bundle structure: the projection $\pi_E$ maps each token to its trajectory context. The densities on the base space $B$ and bundle space $E$ will be introduced in the next section, where we define the FBG operator that performs stability control by gating these densities.

\subsection{Fiber Bundle Gating}
\label{subsubsec:fbg}

Building on the fiber bundle model (Section~\ref{subsubsec:fiber_bundle}) and the RGF language
(Section~\ref{subsubsec:rgf}), we now introduce Fiber Bundle Gating (FBG), a framework
that formally couples global and local information by operating
on densities over the bundle and base space.
The FBG operator acts in four stages: it pushes the token-level density forward to the base to
form the trajectory-level base density; applies a base-level atomic gate to enforce a trust-region
budget at the trajectory level; reflects the gated global signal back to the total space via a Markov kernel $K$; and gates the residual density on each fiber, capturing each token's
deviation from the trajectory aggregate, to prevent individual token spikes.
The reflecting condition on $K$ ensures the two components decouple cleanly: the residual
carries no base-level information, so global and local gating operate independently and
compositionally.
Once an FBG instance is specified, it induces a surrogate objective in RGF form
(Definition~\ref{def:rgf}), with the ratio gating map $\mathcal{G} = \mathcal{R} \circ G \circ
\mathcal{F}$ determined by the bundle and gating functions (Definition~\ref{def:fbg_rgf}).
All four constituent objects are defined relative to the chosen base space $B$ and therefore
extend naturally to richer global structures such as domain or temporal groupings (see
Appendix~\ref{app:fbg} for detailed exposition).

\begin{definition}[Fiber Bundle Gating]\label{def:fbg}
A Fiber Bundle Gating instance is specified by the tuple $(E,B,\pi_E,K,g_{\rm Base},g_{\rm Fiber,-};\,\mathcal F,\mathcal R)$, where $\pi_E:E\to B$ is a fiber bundle.
Let $\mathbf{E}$ and $\mathbf{B}$ denote the spaces of densities on $E$ and $B$, respectively.
The FBG gating function $G_{(\pi_E,K,g_{\rm Base},g_{\rm Fiber,-})}:\mathbf{E}\to\mathbf{E}$ is defined as:
\begin{equation}\label{eq:fbg}
    G(\sigma):=K\!\left(\Bigl(\bigoplus_{b \in B} g_{{\rm Base},b}\Bigr)(\pi_{E*}\sigma)\right)+\left(\Bigl(\bigoplus_{i \in E} g_{{\rm Fiber},i,\pi_{E*}(\sigma)}\Bigr)\bigl(\sigma - K(\pi_{E*}\sigma)\bigr) \right),
\end{equation}
where the constituent objects are:
\begin{enumerate}
    \item \textbf{Fiber bundle and density spaces.}
    $\pi_E: E\to B$ is a fiber bundle.
    Associated to $E$ and $B$ are their density spaces $\mathbf{E}$ and $\mathbf{B}$, consisting of all densities on the total space $E$ and the base space $B$, respectively
    (see Appendix~\ref{app:density} for the formal definition of densities on manifolds).

    \item \textbf{Fibration decomposition and recovery maps.}
    The fibration decomposition map $\mathcal F:\mathbb R_{>0}^{\mathcal E}\to \mathbf{E}$ converts policy ratios to densities on the total space, and the recovery map $\mathcal R:\mathbf{E}\to \mathbb R_{>0}^{\mathcal E}$ is a left-inverse of $\mathcal F$ that converts densities back to policy ratios.

    \item \textbf{Reflecting Markov kernel.}
    $K:\mathbf{B}\to \mathbf{E}$ is a  Markov kernel satisfying the reflecting condition $\pi_{E*}\circ K=\mathrm{id}_{\mathbf{B}}$. The subtraction $\sigma - K(\pi_{E*}\sigma)$ decouples the global information encoded in the base density $\pi_{E*}\sigma$ from $\sigma$, through the Markov kernel $K$. We need to ensure that, after decoupling, the residual density should no longer contribute to any global information, so we must have  $\pi_{E*}(\sigma - K(\pi_{E*}\sigma))=0$, which directly implies our requirement for $K$: $\pi_{E*}\circ K={\rm id}_{\mathbf B}$.

    \item \textbf{Atomic gating functions.}
    $g_{{\rm Base},b}:\mathcal D^1_bB\to\mathcal D^1_bB$ and $g_{{\rm Fiber},i,p_B}:\mathcal D^1_iE\to\mathcal D^1_i E$
    are pointwise gating functions on the localized base and fiber densities, respectively, where $b \in B$, $i \in E$, and $p_B \in \mathbf{B}$.
    The fiber gating function $g_{{\rm Fiber},i,p_B}$ depends on the base density $p_B$, reflecting the fact that local gating may be conditioned on global context.
\end{enumerate}
\end{definition}

\noindent
The FBG gating function~\eqref{eq:fbg} operates in two steps.
First, the base component: the density $\sigma$ is pushed forward to the base via $\pi_{E*}$ to form a base density, gated by $g_{\rm Base}$, and reflected back to the total space via the Markov kernel $K$.
Second, the fiber component: the residual $\sigma - K(\pi_{E*}\sigma)$, which captures local information after removing the global contribution, is gated by $g_{\rm Fiber}$.
The reflecting condition on $K$ ensures that these two components are orthogonal in the sense that the residual carries no base-level information.

\begin{definition}[FBG gating form]\label{def:fbg_rgf}
Given an FBG instance $(E,B,\pi_E,K,g_{\rm Base},g_{\rm Fiber,-};\,\mathcal F,\mathcal R)$, the associated surrogate objective takes the RGF form:
\begin{equation}\label{eq:fbg_rgf}
    \hat J_{(G;\mathcal F,\mathcal R)}(\theta|\theta_{\rm old}):=\sum_{\tau\in \mathrm{Tj}^{\pi_{\theta_{\rm old}}}}\left[ \sum_{t=0}^{T_\tau-1}  {\frac{1}{T_\tau}\, \mathcal R\circ G\circ\mathcal F(r_{\bullet})_{s_t,a_t}\; \hat A_{s_t,a_t} }  \right],
\end{equation}
where the induced ratio gating map is the composition $\mathcal{G} = \mathcal R\circ G\circ \mathcal F: \mathbb R_{>0}^{\bar{\mathcal X}} \to \mathbb R_{>0}^{\bar{\mathcal X}}$.
\end{definition}

\noindent

Fiber Bundle Gating guarantees the first order approximation to the true RL objective near on-policy:

\begin{theorem}[First-order agreement of FBG]\label{thm:fbg_first_order}
Suppose the atomic gating functions satisfy the identity conditions at the on-policy density, i.e., for all $i\in E$ and $b\in B$:
\[
    g_{{\rm Fiber},i,p_B}(\mathcal F(\mathbf{1})_i  -  K(\pi_{E*}\mathcal F(\mathbf{1}))_i)=\mathcal F(\mathbf{1})_i - K(\pi_{E*}\mathcal F(\mathbf{1}))_i
\]
\[
    g'_{{\rm Fiber},i,p_B}(\mathcal F(\mathbf{1})_i-  K(\pi_{E*}\mathcal F(\mathbf{1}))_i)=1, \quad
    g_{{\rm Base},b}(\pi_{E*}\mathcal F(\mathbf{1}))=\pi_{E*}\mathcal F(\mathbf{1}), \quad
    g_{{\rm Base},b}'(\pi_{E*}\mathcal F(\mathbf{1}))=1.
\]
Then the associated FBG objective agrees with the true RL objective to first order at $\theta = \theta_{\rm old}$:
\[
    \hat J_{(G;\mathcal F,\mathcal R)}(\theta_{\rm old}|\theta_{\rm old})= J(\theta_{\rm old}),
    \qquad
    \nabla_{\theta}\hat J_{(G;\mathcal F,\mathcal R)}(\theta|\theta_{\rm old})\big|_{\theta = \theta_{\rm old}}= \nabla_{\theta}J(\theta)\big|_{\theta = \theta_{\rm old}}.
\]
\end{theorem}

\noindent
\emph{Proof} See Appendix~\ref{app:fbg_proof}. \qed

\begin{remark}[Global-to-local information exchange]\label{rem:global_local_transmit}
Beyond first-order agreement, FBG's decoupling mechanism is essential for \emph{exchanging} global gating effects with individual token ratios control.
General RGF gating maps only operate on ratio tuples $r_\bullet$, without the structured decomposition, naively combining global and local gating. This may produce undesirable gradients, or entangle the two scales so that neither gate operates at its intended granularity.
FBG resolves this: $g_{\rm Base}$ gates the trajectory-level base density in isolation, while $g_{\rm Fiber}$ gates local residual $\sigma - K(\pi_{E*}\sigma)$---the per-token variation from which all global influence has been removed.
The reflecting condition guarantees that these two operations are orthogonal and non-interfering (see Appendix~\ref{app:fbg_transmit} for a detailed discussion).
\end{remark}

\begin{remark}[Extensibility via the base space]\label{rem:fbg_extensibility}
All four constituent objects of FBG are defined relative to the chosen base space $B$.
Extending $B$ to include additional classifying factors (e.g.,
domain, temporal grouping) naturally extends the FBG instance:
the pushforward $\pi_{E*}$ produces finer aggregates, the Markov
kernel $K$ reflects finer base densities back to the bundle, and
the atomic gates $g_{\rm Base}$, $g_{\rm Fiber}$ operate per class and locally, respectively.
Theorem~\ref{thm:fbg_first_order} guarantees first-order agreement for any such extension, provided the atomic gates satisfy the identity conditions at the on-policy point.
The FiberPO-Trajectory objective derived in Section~\ref{subsec:fiberpo} is the simplest nontrivial instantiation of this framework. In Section~\ref{subsec:fibration_hierarchy}, we generalize FBG to the Fibration Gating Hierarchy (FGH) and derive FiberPO-Domain, a four-level instantiation that incorporates domain and prompt group structure (see also Remark~\ref{rem:base_choice_app}).
\end{remark}

%% file: rlvr_fiberpo_sec3.tex
\section{FiberPO}
\label{subsec:fiberpo}

Starting from the APC-Obj objective (Section~\ref{subsubsec:apc_obj}), we derive the FiberPO surrogate by applying a sequence of controlled transformations yielding a Fiber Bundle Gating form (Definition~\ref{def:fbg}).
The derivation proceeds through four derivation steps (detailed in Appendix~\ref{app:fiberpo_derivation}):
\begin{enumerate}
    \item \textbf{$\delta$-relaxation:} replace the vanishing TRPO/APC-Obj threshold $\delta^{(\rm APC\text{-}Obj)}$ with a positive tunable hyperparameter $\delta$.
      This step is necessitated by the vanishing theorem (Theorem~\ref{thm:trpo_vanishing}), which shows the classical radius is zero at $\gamma{=}1$, and justified by the GSPO analysis (Claim~\ref{claim:gspo_tv}), which demonstrates that a positive $\delta \sim 10^{-2}$ effectively controls policy drift in the undiscounted setting.
      The clipping mechanism in APC-Obj's RGF form remains well-defined at the relaxed radius (Remark~\ref{rem:trust_region_design}).
    \item \textbf{Logarithmic approximation:} move to log-ratio coordinates via $\log r \approx r-1$ and $1+\operatorname{clip}(\cdots)\approx\exp(\operatorname{clip}(\cdots))$.
      Working in $\log r \in \mathbb R$ and exponentiating the result guarantees $r > 0$, so all gating operations (clipping, piecewise-linear gates) can be designed with unconstrained linear algebra on $\mathbb R$ rather than constrained to preserve positivity.
    \item \textbf{Sequence-level aggregation:} restrict the
      cross-token coupling to tokens within the same trajectory
      (replacing the all-token sum with a per-trajectory sum)
      and replace the per-state budget $T_s\delta$ with a
      per-trajectory budget $T_\tau\delta$.
    \item \textbf{Clipping decomposition:} decompose the coupled clipping term into an explicit base-level gate $g^{\rm agg}$ on the trajectory aggregate ratios $s_\tau^+, s_\tau^-$ and a fiber-level $\operatorname{logclip}$ on per-token residuals, with per-channel trust-region budgets $C^+, C^-$ ($C^++C^-=\delta$) that accommodate the intrinsic asymmetry $\log s^-_\tau \geq \log s^+_\tau$.
\end{enumerate}

\noindent
Derivation~4 is the central and distinctive step: it transforms a single coupled clipping bound into two explicit, complementary gating components that map directly onto the FBG architecture (base gate + fiber gate).
The resulting objective takes the following form.

\begin{definition}[Fiber-Aware Clipping Policy Optimization (FiberPO)]\label{def:fiberpo_main}
The FiberPO objective in RGF form is:
\begin{equation}\label{eq:fiberpo_main}
    {\hat J^{{\text{FiberPO}}}}(\theta |\theta_{\rm old}) = \sum_{(s,a,\tau,t)\in\bar{\mathcal X}} {\frac{1}{{|\mathrm{Tj}^{\theta_{\rm old}}|}}{\frac{1}{{T_\tau}}}  \cdot\mathcal G(r_\bullet)_{s,a,\tau,t}\;\cdot \hat A_{s,a} },
\end{equation}
where $\mathcal E=\bar{\mathcal X} :=\{ (s_t(\tau),a_t(\tau),\tau,t) \}\subseteq \mathcal X\times \mathrm{Tj}^{\pi_{\theta_{\rm old}}}\times \mathbb N $ is the augmented state-action
pair space\footnote{In the RGF (Definition~\ref{def:rgf}), each index $i \in \mathcal{E}$ carries extra information $\mathcal{I}$; here the extra information is the trajectory--timestep pair $(\tau, t)$, so each element of $\mathcal{E}$ is an augmented state-action pair $(s,a,\tau,t)$. See Appendix~\ref{app:estimation} for details.}, $T_\tau$ is the
length of trajectory $\tau$, and the ratio gating map $\mathcal G
= \mathcal R \circ G \circ \mathcal F$
(Definition~\ref{def:fbg_rgf}) acts on the full tuple of
importance ratios $r_\bullet$ and produces a gated ratio for each
augmented token $(s,a,\tau,t)\in\bar{\mathcal X}$. The fully expanded objective and its equivalent expression in LLM notation $(g,j,i)$ are given in Definition~\ref{def:fiberpo} and Appendix~\ref{app:fiberpo_llm_notation}, respectively.

\medskip\noindent
The FBG gating map $\mathcal G$ for each token $i \equiv (s,a,\tau,t)$ is a multiplicative decomposition into a trajectory-level \emph{base weight} and a token-level \emph{gated residual}:
\begin{equation}\label{eq:fiberpo_gating_decomp}
    \mathcal G(r_\bullet)_i \;=\; \underbrace{\frac{\exp\circ\;g^{\rm agg}(\log s_\tau^+,\;C^+,\;T_\tau)}{\exp\circ\;g^{\rm agg}(\log s_\tau^-,\;C^-,\;T_\tau)}}_{w^{\rm base}_\tau\;:\;\text{base weight (after base gating)}} \;\cdot\; \underbrace{\frac{{\rm logclip}\!\left((s_\tau^{(l_i)})^{-l_i}\,r_i,\;\epsilon\right)}{{\rm logclip}\!\left((s_\tau^{(-l_i)})^{-l_i},\;\epsilon\right)}}_{\tilde r_i^{\rm fiber}\;:\;\text{gated residual (after fiber gating)}},
\end{equation}
Each term will be explained carefully in the subsections below.
Overall, we call $w^{\rm base}_\tau$ the \emph{base weight} and
$\tilde r_i^{\rm fiber}$ the \emph{gated residual}. This
decomposition directly reflects the FBG architecture
(Definition~\ref{def:fbg}): the base gate $g_{\rm Base}$ operates
on pushed-forward densities on $B$, and the fiber gate $g_{\rm
  Fiber}$ operates on the residual after reflecting back to $E$
via the Markov kernel $K$. The base weight $w^{\rm base}_\tau$
depends only on trajectory-level aggregates and is shared by all
tokens in trajectory $\tau$; it controls how much gradient signal
the trajectory as a whole is permitted to contribute. The gated
residual $\tilde r_i^{\rm fiber}$ captures each token's deviation
from the trajectory aggregate on its respective sign channel,
gated by $\operatorname{logclip}$ to prevent individual token
spikes. The constituents of each component are defined in the
following two subsections.

\end{definition}

\subsection{Base Weight (Base Gate)}
\label{subsubsec:base_weight}

The base weight is
\begin{equation}\label{eq:base_weight}
    w^{\rm base}_\tau \;:=\; \frac{\exp\circ\;g^{\rm agg}(\log s_\tau^+,\;C^+,\;T_\tau)}{\exp\circ\;g^{\rm agg}(\log s_\tau^-,\;C^-,\;T_\tau)},
\end{equation}
where the constituent quantities are:

\paragraph{Trajectory aggregate ratios.}
The positive and negative aggregate ratios decompose the trajectory-level drift by sign:
\begin{equation}\label{eq:aggregate_ratios}
    \log s^+_\tau :=\frac{1}{T_\tau}\sum_{t=0}^{T_\tau-1} \max(\log r_{s_t(\tau),a_t(\tau)},\,0), \qquad
    \log s^-_\tau :=\frac{1}{T_\tau}\sum_{t=0}^{T_\tau-1} \max(-\log r_{s_t(\tau),a_t(\tau)},\,0).
\end{equation}
Here $\log s^+_\tau$ averages the positive log-ratio
contributions (tokens whose likelihood increased under
$\pi_\theta$) and $\log s^-_\tau$ averages the negative ones
(tokens whose likelihood decreased). We refer to $\log
s^+_\tau$ and $\log s^-_\tau$ as the positive and negative
channels' drift, or the signed-trajectory mean, respectively. \\[0.5em]
One of the merits of splitting by sign (for more theoretical reasoning, see Appendix~\ref{app:fiberpo_derivation4})  rather than averaging all log-ratios is that
the total trajectory drift $\overline{\log r}_\tau = \log s^+ -
\log s^-$ can be small even when both $\log s^+$ and $\log s^-$
are individually large. In this case, the trajectory contains many token-ratios that have shifted substantially in both positive and negative directions.\\[0.5em]
(i) If one simply performs $\overline{\log r}_\tau$ mean-subtraction, the mean-subtracted individual token will be large and easily clipped by token-level $\operatorname{logclip}$ (for details, see ratio decomposition in \ref{subsubsec:ratio_decoupling}), and\\[0.5em]
(ii) in this case, the trajectory TV distance $\hat{\bar D}_{\rm TV}^{(\tau)} \approx (1/T_\tau)\sum_t |\log r_{s_t, a_t}|$ is large and may need control, however keeping only $\overline{\log r}_\tau$ would mask this need for control.\\[0.7em]
By tracking
each sign channel independently, $g^{\rm agg}$ detects high total
variation (in ii) in the importance weights even when the signed average
nearly cancels, can apply rollback on the offending channel's aggregate ratio
without suppressing the well-behaved aggregate ratio, while also retaining many well-behaved local tokens that are located near the possibly large $\log s^{\pm}$ unclipped (from i), providing more nuanced control over the optimization process.

\paragraph{Aggregate gating function $g^{\rm agg}$.}
The function $g^{\rm agg}$ is a piecewise-linear gate that
transforms the aggregate log-ratio on each sign channel:
\begin{equation}\label{eq:gagg_main}
    g^{\rm agg}(x,C,T_\tau):= \left\{\begin{array}{ll}   x&\text{if }|x|\leq C\\[0.5em]  \operatorname{sign}(x)(T_\tau+1)C-T_\tau x & \text{if }C<|x|<(1+T_\tau^{-1})C  \\[0.5em]  0 &\text{otherwise}   \end{array}\right.
\end{equation}
where $T_\tau$ is the sequence length and $C \in \{C^+, C^-\}$ is the
per-channel trust-region budget. The three regimes correspond
directly to the three APC-Obj clipping regimes~(P), (R), (Z)
(Appendix~\ref{app:apc_obj_zones}), lifted to the aggregate level via the
clipping decomposition (Appendix~\ref{app:fiberpo_derivation},
Eq.~\ref{eq:fiberpo_decomposed}):

\begin{itemize}

\item \textbf{Pass-through~(P)} ($|x| \leq C$): the aggregate drift
  is within the trust-region budget. The gate outputs $x$
  unchanged, preserving the full importance-sampling correction
  on this channel. This corresponds to the aggregate-level counterpart of APC-Obj regime~(P) where $\hat{D}_{\rm TV}(s) \leq \delta^{(\rm APC\text{-}Obj)}$ (Eq.~\ref{eq:zone_P}): the total budget consumption (C1)~+~(C2) does not exceed the budget~(B), and the clipping is inactive.
\item \textbf{Rollback~(R)} ($C < |x| < C^* := (1+T_\tau^{-1})C$):
  the aggregate drift has exceeded the budget. The gate reverses
  slope to $-T_\tau$, producing a restorative gradient that
  actively opposes the drift direction and pushes the aggregate
  back toward on-policy.
  The term ``rollback'' follows the terminology of~\cite{wang2020truly}, who introduced a similar slope-reversing mechanism at the per-token level.
  This is the aggregate-level counterpart of APC-Obj regime~(R) (Eq.~\ref{eq:zone_R}), where $\hat{D}_{\rm TV}(s) > \delta^{(\rm APC\text{-}Obj)}$ but the current token's deviation (C2) has not individually exhausted the budget overshoot, so the clip suppresses $|r_{s,a}-1|$ without zeroing it\footnote{In the per-state APC-Obj objective (Eq.~\ref{eq:apc_obj_per_token}), regime~(R) requires $|r_{s,a}-1|/T_s > \hat{D}_{\rm TV}(s) - \delta$, equivalently $0 < |\operatorname{clip}(r_{s,a}-1, \text{(B)}-\text{(C1)})| < |r_{s,a}-1|$. After trajectory-level aggregation (Appendix~\ref{app:fiberpo_derivation3}), logarithmic approximation (Appendix~\ref{app:fiberpo_derivation2}), and clipping decomposition (Eq.~\ref{eq:fiberpo_decomposed}), this per-token condition on $|r_{s,a}-1|$ lifts to a condition on the aggregate log-ratio $\log s_\tau^{(l)}$ that still depends on individual $\log r_i$ (i.e., $|\log r_i|/T_\tau > \log s^{(l_i)}_\tau - \delta$). Since $T_\tau$ is typically large in LLM settings, the per-token dependence becomes negligible and the condition reduces to the uniform cross-token condition $C^{(l)} < \log s_\tau^{(l)} < C^{*(l)}$ on the aggregate ratio alone.}.
  The remaining budget induces a linear penalty whose gradient opposes further
  drift. The rollback regime is the unique continuous
  piecewise-linear interpolation between pass-through and
  zeroing, with slope fixed by the APC-Obj decomposition structure.
\item \textbf{Zeroed~(Z)} ($|x| \geq C^*$): the budget is fully
  consumed. The gate outputs $0$, completely blocking gradient
  signal for this channel. This is the aggregate-level counterpart of APC-Obj regime~(Z) (Eq.~\ref{eq:zone_Z}), where the cross-token consumption~(C1) alone exceeds the total budget~(B), making the residual clip bound nonpositive and forcing the clip output to zero.
\end{itemize}

\noindent
The rollback regime has width $C/T_\tau$ and shrinks for longer
trajectories, so $g^{\rm agg}$ approaches a hard clip at $\pm C$
as $T_\tau \to \infty$.

\paragraph{Per-channel budgets.}
$C^+$ and $C^-$ are the per-channel trust-region budgets
satisfying $C^+ + C^- = \delta$, with the recommendation $C^- <
C^+$ to compensate for the intrinsic KL bias $\log s^-_\tau \geq \log
s^+_\tau$ (Appendix~\ref{app:fiberpo_derivation}). The symmetric
case $C^+=C^-=\delta/2$ is a special case.

\begin{figure}[t]
    \centering
    \includegraphics[width=\textwidth]{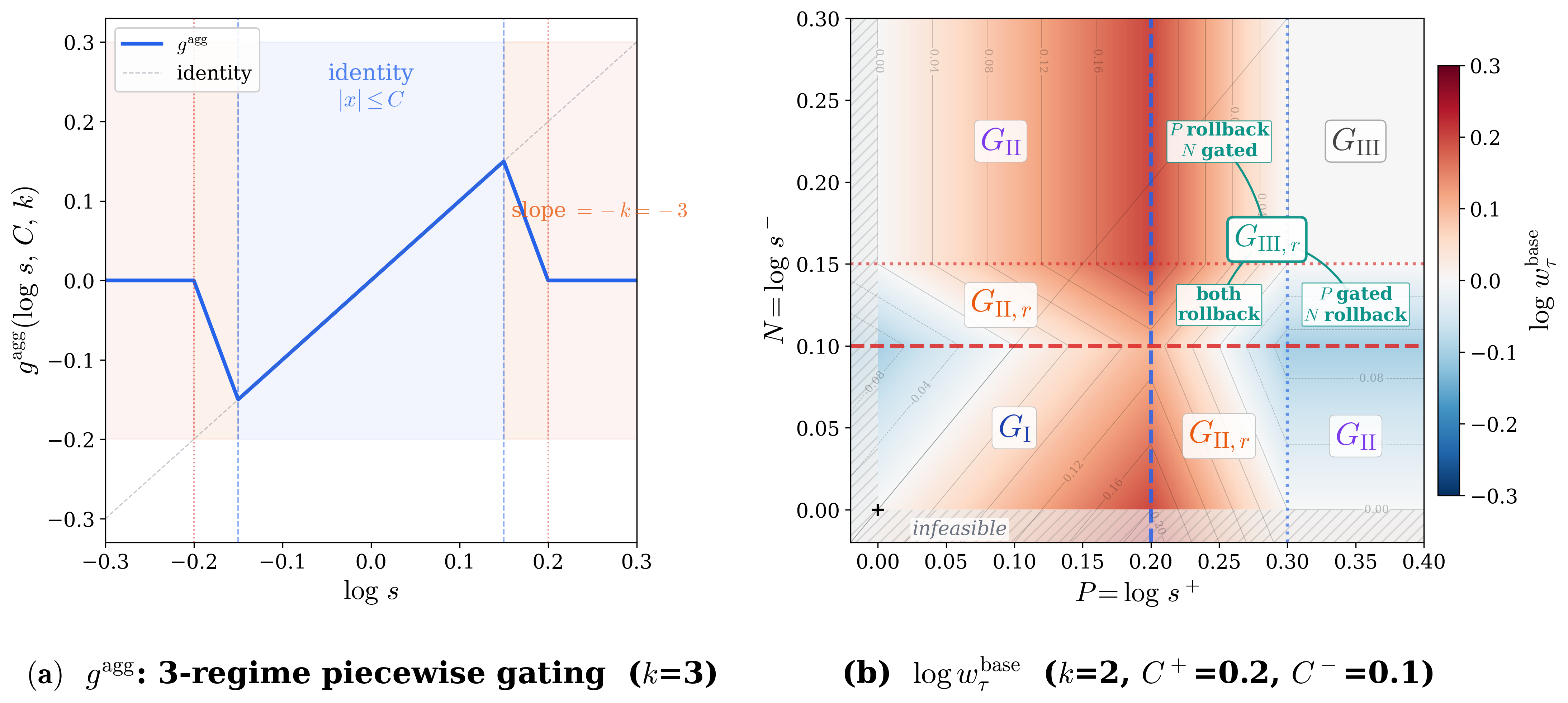}
    \caption{Base weight visualization.
      Panel~(a) shows the aggregate gate $g^{\rm agg}$
      (Eq.~\ref{eq:gagg_main}) with three regimes: pass-through
      ($|x| \leq C$, slope $1$), rollback ($C < |x| < C^*$, slope
      $-k$), and zeroed ($|x| \geq C^*$, output $0$). As $k =
      T_\tau$ increases, the rollback regime narrows (width
      $C/T_\tau$) and $g^{\rm agg}$ approaches a hard clip at
      $\pm C$.
      Panel~(b) shows the base weight $\log w_\tau^{\rm base}$
      (Eq.~\ref{eq:base_weight}) in $(\log s^+, \log s^-)$-space
      with asymmetric thresholds $C^+ = 0.20$, $C^- = 0.1$ ($k =
      2$). Dashed lines mark the budget boundaries $C^\pm$ (onset
      of rollback); dotted lines mark the full-gating thresholds
      $C^{*\pm}$ (onset of zeroing). The five global regimes (G-I
      through G-III)
      follow a non-monotonic pattern: $|\log w|$ rises through the
      rollback onset (G-II,r), peaks when one channel is fully
      gated (G-II), declines under mutual rollback (G-III,r), and
      collapses to zero when both channels are fully gated
      (G-III, $w^{\rm base}_\tau = 1$).}
    \label{fig:fiber_weight}
\end{figure}

\paragraph{Regime analysis of the base weight.}
\label{par:regime_base_weight}
Figure~\ref{fig:fiber_weight} visualizes $g^{\rm agg}$ and the
resulting base weight. Since $w^{\rm base}_\tau$ is a ratio of
two independently gated channels (Eq.~\ref{eq:base_weight}), the
combined behavior depends on which regime each channel occupies.
The key structural feature is that $w^{\rm base}_\tau$ depends on
$\log s_\tau^+$ and $\log s_\tau^-$ \emph{independently}: two
trajectories with identical mean log-ratio
$\frac1{T_\tau}\sum{\log r_t}=\log s^+ - \log s^-$ but different $\hat{\bar{D}}_{\rm TV}^{(\tau)}$ may produce same base weights, depending on
whether $(\log s^+ , \log s^-)$ lies in the G-I region or not.
A trajectory with large positive \emph{and} negative shifts falls
outside G-I, triggering rollback on both channels even when the
shifts nearly cancel; this produces a gated base weight. In contrast, a trajectory in G-I has base weight exactly
$\exp(\log s^+ - \log s^-)$, since $g^\text{agg}$ acts as the identity.

The five global regimes are:
\begin{itemize}
    \item \textbf{G-I} (both pass-through): $w^{\rm base}_\tau = \exp(\log s^+ - \log s^-)$, the unmodified importance-sampling ratio. The base gate is transparent.
    \item \textbf{G-II,r} (one channel rolling back): the active rollback produces a restorative gradient that opposes the drifting channel; $|\log w^{\rm base}_\tau|$ is rising.
    \item \textbf{G-II} (one channel fully gated): $|\log w^{\rm base}_\tau|$ peaks at approximately $C^{(l)}$ (the surviving channel $l$'s budget), delivering maximum one-sided correction.
    \item \textbf{G-III,r} (mutual rollback): both channels are past their thresholds; $|\log w^{\rm base}_\tau|$ is declining.
    \item \textbf{G-III} (both fully gated): $w^{\rm base}_\tau = 1$ and the trajectory-level gradient contribution vanishes entirely.
\end{itemize}

\subsection{Gated Residual (Fiber Gate)}
\label{subsubsec:gated_residual}

The gated residual captures each token's deviation from the trajectory aggregate, gated to prevent individual token spikes:
\begin{equation}\label{eq:gated_residual}
    \tilde r_i^{\rm fiber} \;:=\; \frac{{\rm logclip}\!\left((s_\tau^{(l_i)})^{-l_i}\,r_i,\;\epsilon\right)}{{\rm logclip}\!\left((s_\tau^{(-l_i)})^{-l_i},\;\epsilon\right)}.
\end{equation}
We now define its constituent elements.

\paragraph{Log-clipping function.}
The log-clipping function is:
\begin{equation}\label{eq:logclip}
    \operatorname{logclip}(x,\epsilon) \;:=\; \exp\!\bigl(\operatorname{clip}(\log x,\;\pm\epsilon)\bigr).
\end{equation}

We explain the specific choice of log space and thus log clip in
Appendix~\ref{app:fiberpo_derivation2}. Intuitively, in non-log
space, this is equivalently
$\operatorname{clip}(x,\;e^{-\epsilon},\;e^{+\epsilon})$: the
argument is clamped to the interval $[e^{-\epsilon},
e^{+\epsilon}]$. The asymmetry is inherent: the upper bound
$e^{+\epsilon} \approx 1 + \epsilon + \epsilon^2/2$ is further
from $1$ than the lower bound $e^{-\epsilon} \approx 1 - \epsilon
+ \epsilon^2/2$, so a symmetric $\epsilon$ in log-space naturally
produces an asymmetric clip in ratio space. This subsumes the
effect of asymmetric clip bounds $\epsilon^+/\epsilon^-$ used in
PPO/GRPO without introducing an additional hyperparameter.

\paragraph{Sign label and channel decomposition.}
For each token $i$, the sign assignment
\begin{equation}\label{eq:parity}
    l_i \;:=\; \operatorname{sign}(\log r_i) \;\in\; \{+1,\,-1\}
\end{equation}
indicates whether the new policy has increased ($l_i = +1$) or
decreased ($l_i = -1$) the likelihood of token~$i$ relative to
the reference policy. A parenthesized sign-label superscript
selects the sign channel: $s_\tau^{(l_i)} = s_\tau^+$ when
$l_i=+1$ and $s_\tau^-$ when $l_i=-1$; the negated form
$s_\tau^{(-l_i)}$ selects the opposite channel (see
Appendix~\ref{app:ratio_space} for notation conventions).

This sign partitions the tokens within each trajectory into two
channels (positive-ratio and negative-ratio, reflecting the FBG
fiber bundle structure of Section~\ref{subsubsec:fiber_bundle},
see Appendix~\ref{app:fiberpo_derivation4} for the instantiation
with signed channels). The gated
residual~\eqref{eq:gated_residual} decomposes into two
sign-channel cases:
\[
    \tilde r_i^{\rm fiber} = \begin{cases}
        \dfrac{\operatorname{logclip}\!\bigl((s_\tau^+)^{-1}\,r_i,\;\epsilon\bigr)}{\operatorname{logclip}\!\bigl((s_\tau^-)^{-1},\;\epsilon\bigr)} & \text{if } l_i = +1, \\[1.2em]
        \dfrac{\operatorname{logclip}\!\bigl((s_\tau^-)^{+1}\,r_i,\;\epsilon\bigr)}{\operatorname{logclip}\!\bigl((s_\tau^+)^{+1},\;\epsilon\bigr)} & \text{if } l_i = -1.
    \end{cases}
\]
In each case, the numerator involves only the same-sign aggregate
$s_\tau^{(l_i)}$, while the denominator involves the
opposite-sign aggregate $s_\tau^{(-l_i)}$.

\paragraph{Fiber residual and the gated residual in log-space.}
Define the \emph{fiber residual}
\begin{equation}\label{eq:fiber_residual}
    u_i \;:=\; l_i\log r_i - \log s_\tau^{(l_i)}
\end{equation}
and the \emph{opposite-sign aggregate}
\begin{equation}\label{eq:opposite_sign_agg}
    v_i \;:=\; -\log s_\tau^{(-l_i)}.
\end{equation}
The gated residual~\eqref{eq:gated_residual} can be written equivalently in terms of $u_i$ and $v_i$ as:
\begin{equation}\label{eq:gated_residual_uv}
    \tilde{r}_i^{\rm fiber} \;=\; \frac{\operatorname{logclip}\!\bigl(e^{l_i\,u_i},\;\epsilon\bigr)}{\operatorname{logclip}\!\bigl(e^{l_i\,v_i},\;\epsilon\bigr)} =  \exp(\operatorname{clip}\!\bigl({l_i\,u_i},\;\epsilon\bigr) - \operatorname{clip}\!\bigl({l_i\,v_i},\;\epsilon\bigr)).
\end{equation}
The fiber residual $u_i$~\eqref{eq:fiber_residual} subtracts the same-sign trajectory mean ($\log s_\tau^{(l_i)}$)
from each token's log-ratio. The numerator's
$\operatorname{logclip}$ acts on $e^{l_i u_i}$, which involves
only the same-sign channel aggregate $s_\tau^{(l_i)}$; this
prevents the opposite channel's aggregate ratio from directly
entering the numerator clip, avoiding opposite-channel
contamination. The denominator's $\operatorname{logclip}$ acts on
$e^{l_i v_i}$, incorporating the opposite-sign aggregate
$s_\tau^{(-l_i)}$ to complete the subtraction by trajectory-mean log-ratio ($\overline{\log r}_\tau := (1/T_\tau)\sum_{t}\log
r_{s_t,a_t}$).

\paragraph{Ratio decoupling.}
\label{subsubsec:ratio_decoupling}

When neither the numerator nor denominator $\operatorname{logclip}$ is saturated (i.e., the internal $\operatorname{clip}$ is inactive: $|u_i|
\leq \epsilon$ and $|v_i| \leq \epsilon$), the $\exp$ and $\log$
cancel and we obtain:
\[
    \log\tilde r_i^{\rm fiber} = l_i(u_i - v_i) = \log r_i - \overline{\log r}_\tau,
\]
where, recalling the log-aggregate ratio $\log s_\tau := \frac{1}{T_\tau}\sum_{t=0}^{T_\tau-1}\log r_{s_t(\tau),a_t(\tau)}$ (see also Appendix~\ref{app:gspo_tv_proof} for definition),
\[
    \overline{\log r}_\tau \;:=\; \log s_\tau \;=\; \frac{1}{T_\tau}\sum_{t=0}^{T_\tau-1}\log r_{s_t(\tau),a_t(\tau)}
\]
is the trajectory mean log-ratio. Thus the gated
residual is a trajectory-mean-centered quantity: it measures each
token's deviation from the trajectory average, with the
same-sign/opposite-sign channel split ensuring that the
$\epsilon$-clipping acts within each sign channel independently.

First, the total mean-centering $\overline{\log r}_\tau$ has the advantage of recovering the true linear surrogate and thus true RL objective near on-policy, when combined with the base weight's pass-through behavior in G-I.

Second, the signed-trajectory-mean centering quantity $l_iu_i = \log r_i - l_i\log s^{(l_i)}_\tau$ (from the sign channel decomposition also \eqref{eq:parity}) that passed into the numerator (equivalently may be expressed as $\exp\circ\operatorname{clip}(l_iu_i, \epsilon)$) provides the following additional advantages:

\textit{Fiber residual clip isolates signed trajectory-level drift}: even when the trajectory aggregate
$\log s^{(l_i)}_\tau$ is large (indicating substantial overall
policy shift), individual tokens are regulated only by how much
they deviate from the trajectory norm. Tokens that move
relatively closely with the same-sign trajectory-mean (trajectory-level drift)
always pass through the FiberPO gating function ($\mathcal G$ of FiberPO) unattenuated, thus preserving well-behaved tokens their full
gradient signal, allowing them to \textit{provide finer update direction}, leading to higher \textit{token-efficiency}, even when
(signed-)trajectory-level drift are large ($\log s^\pm_\tau > \epsilon$). And in this case,  methods like PPO/GRPO etc. that clip $\log r_i$
directly would have more than half tokens to exceed the $\operatorname{clip}$ bound simultaneously and destroy
token-level discriminated signal.

\textit{Fiber residual isolates inter-trajectory and
  intra-trajectory drift}: To demonstrate this point, consider
two trajectories answering ``Name a famous landmark'':
\begin{center}
\emph{``I love Paris and the Eiffel Tower''} \quad vs.\quad
\emph{``I love Rome and the Colosseum.''}
\end{center}
Globally, the policy may strongly prefer the Paris response, perhaps
it scores higher overall, so the aggregate ratio $s_\tau^+$ for that
trajectory is large. Without decoupling, this global preference bias
leaks into every token's gradient: the token \emph{Colosseum} in the
Rome trajectory receives a weaker learning signal not because the
token-level association ``Rome $\to$ Colosseum'' is poor, but simply
because its trajectory is globally less preferred. The residual
decomposition in Eq.~\ref{eq:fiberpo_gating_decomp} prevents this
contamination. By subtracting the trajectory aggregate from each
token's log-ratio, the fiber gate $\tilde r_i^{\rm fiber}$ isolates
the pure local association, how much ``Colosseum'' co-varies with
``Rome'' relative to what the trajectory drift alone would predict,
and gates it independently via $\operatorname{logclip}$. Within each
trajectory, token-level learning thus operates at a uniform, unbiased
scale: $P(\text{Colosseum}\mid\text{Rome})$ and
$P(\text{Eiffel Tower}\mid\text{Paris})$ are each refined on their own
statistical merits, free from the global preference
$P(\text{Paris trajectory}) \gg P(\text{Rome trajectory})$. The base
weight $w^{\rm base}_\tau$ then re-couples the trajectory-level
preference when the two scales are composed, so that global
significance is preserved without polluting local precision. This is
the orthogonal, non-interfering decomposition guaranteed by the
reflecting condition $\pi_{E*} \circ K = \mathrm{id}_{\mathbf{B}}$
(Appendix~\ref{app:fbg_transmit}).

\paragraph{Regime analysis of the gated residual.}
\label{par:regime_gated_residual}
To make the behavior of the $\epsilon$-clip explicit, we partition trajectories into three \emph{local} regimes according to how many tokens' fiber residuals exceed the clip threshold:
\begin{itemize}
    \item \textbf{L-I} (unclipped): $\max_i |u_i| < \epsilon$. Every token retains its full per-token self-gating; the fiber gate(gated residual $\tilde r^{\rm fiber}_i$) imposes no restriction and faithfully tracks each token's deviation from the trajectory mean log-ratio.
    \item \textbf{L-II} (selective clipping): $0 < n_\epsilon < T_\tau$, where $n_\epsilon$ counts tokens with $|u_i| \geq \epsilon$. Outlier tokens lose their direct per-token gradient while well-behaved tokens retain full signal---the regime where most of the meaningful per-token signals are retained.
    \item \textbf{L-III} (all $\operatorname{logclip}$'s saturated): $\min_i |u_i| \geq \epsilon$. Individual token identity is lost; the Jacobian \eqref{eq:fiberpo_jacobian} reduces to $\frac{\mathcal G(r)_j}{r_j}\tfrac{1}{T_\tau}\gamma_\tau^{(l_j)}$, the gradient is governed entirely by the trajectory-level controlled base weight $w_\tau^{\rm base}$.
\end{itemize}
Figure~\ref{fig:fiberpo_regime_map} visualizes the joint local--global regime map on the probability simplex. It also shows that many parts of L-I/II/III regime cascade fits inside G-I. This is the outcome of encouraged parameter choice:
That local token-level regulation engages before global trajectory-level regulation, allowing the fiber gate to preserve the most meaningful(non-outlier) per-token gradient signal for as long as possible, while the base gate remains transparent until the trajectory as a whole has drifted substantially where regulation on global quantities becomes necessary.

\begin{figure}[t]
    \centering
    \includegraphics[width=\textwidth]{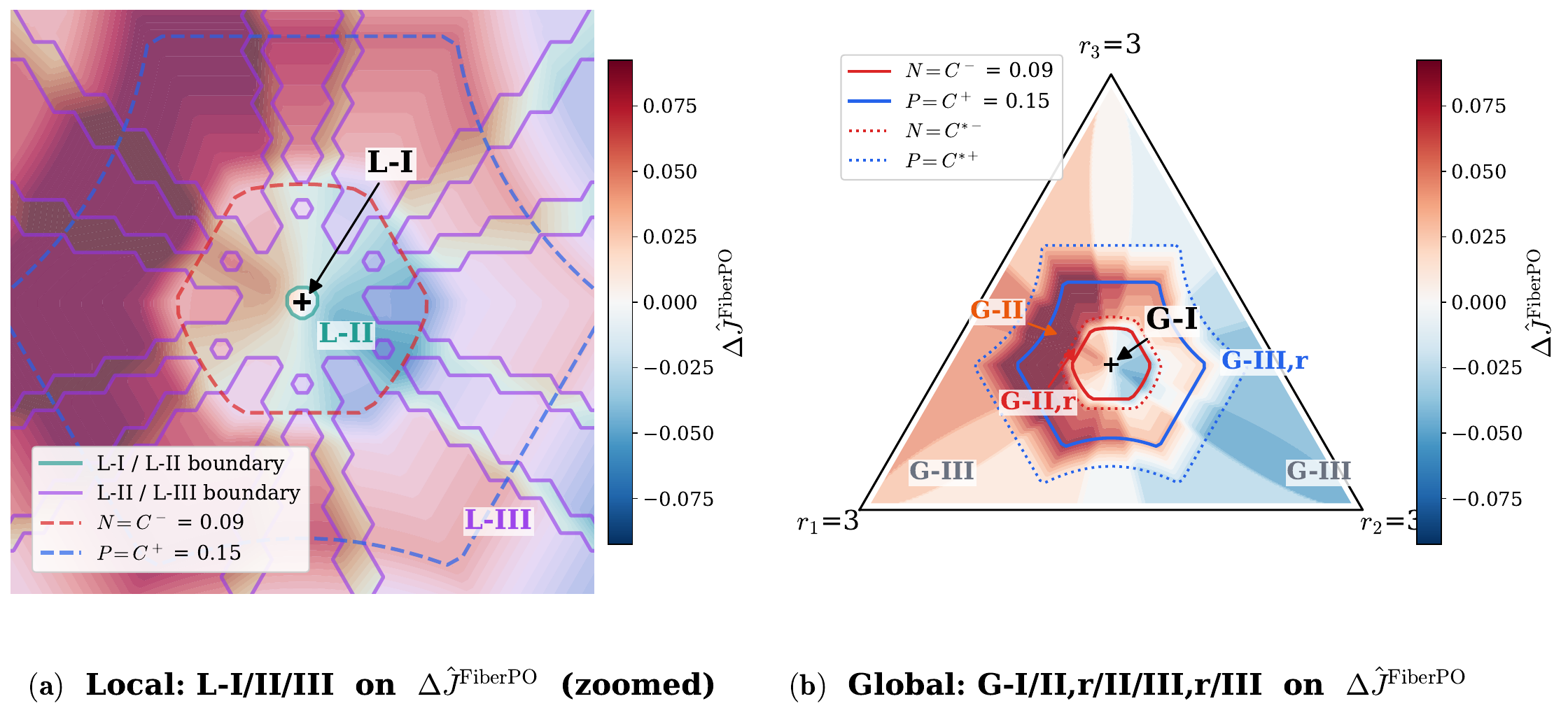}
    \caption{Regime map on the probability simplex given by all
      possible configurations of a trajectory's policy ratios
      (subject to probabilistic constraint ($\frac 1{T_{\rm
          trajectory}}\sum_{t=0}^{T_{\rm trajectory}-1} r_t = 1$)
      where $T{=}3$, $\varepsilon{=}0.025$, $C^+{=}0.15$,
      $C^-{=}0.09$, $K{=}3$,
      $C^{*\pm}{=}(1{+}1/K)\,C^{\pm}$). \textbf{(a)}~Local
      branch (zoomed $\pm 0.18$ around centroid): L-I (no clip,
      central polytope inside teal boundary), L-II (partial clip,
      between teal solid and purple solid boundaries), and L-III
      (full clip, outside purple boundary). Red and blue dashed
      contours show the rollback thresholds $N{=}C^-$ and
      $P{=}C^+$. \textbf{(b)}~Global branch (full simplex): five
      concentric regimes---G-I (inside red solid $N{=}C^-$),
      G-II,r (rollback transition, between red solid $N{=}C^-$
      and red dotted $N{=}C^{*-}$), G-II (between red dotted
      $N{=}C^{*-}$ and blue solid $P{=}C^+$), G-III,r (rollback
      transition, between blue solid $P{=}C^+$ and blue dotted
      $P{=}C^{*+}$), and G-III (outside blue dotted
      $P{=}C^{*+}$).}
    \label{fig:fiberpo_regime_map}
\end{figure}

\subsection{Per-Token and Total Objective}
\label{subsubsec:total_objective}

Substituting the gated ratio decomposition~\eqref{eq:fiberpo_gating_decomp} into the FiberPO objective~\eqref{eq:fiberpo_main}, the per-token objective contribution is:
\begin{equation}\label{eq:per_token_obj}
    \frac{1}{|\mathrm{Tj}^{\theta_{\rm old}}|}\,\frac{1}{T_\tau}\;w^{\rm base}_\tau \cdot \tilde r_i^{\rm fiber} \cdot \hat A_i,
\end{equation}
and the total objective is
\begin{equation}\label{eq:total_obj}
    \hat J^{\rm FiberPO}(\theta|\theta_{\rm old}) = \sum_{(s,a,\tau,t)\in\bar{\mathcal X}} \frac{1}{|\mathrm{Tj}^{\theta_{\rm old}}|}\,\frac{1}{T_\tau}\;w^{\rm base}_\tau \cdot \tilde r_i^{\rm fiber} \cdot \hat A_{s,a}.
\end{equation}
In terms of the fiber residual $u_i$, opposite-sign aggregate $v_i$, and aggregate ratios $\log s^\pm_\tau$, this can be expressed as:
\[
    \hat J^{\rm FiberPO} = \sum_{i\in\bar{\mathcal X}} \frac{1}{|\mathrm{Tj}^{\theta_{\rm old}}|}\,\frac{1}{T_\tau}\;\frac{\exp\circ\,g^{\rm agg}(\log s^+_\tau, C^+, T_\tau)}{\exp\circ\,g^{\rm agg}(\log s^-_\tau, C^-, T_\tau)} \cdot \frac{\operatorname{logclip}(e^{l_i u_i}, \epsilon)}{\operatorname{logclip}(e^{l_i v_i}, \epsilon)} \cdot \hat A_i.
\]

\begin{figure}[t]
    \centering
    \includegraphics[width=\textwidth]{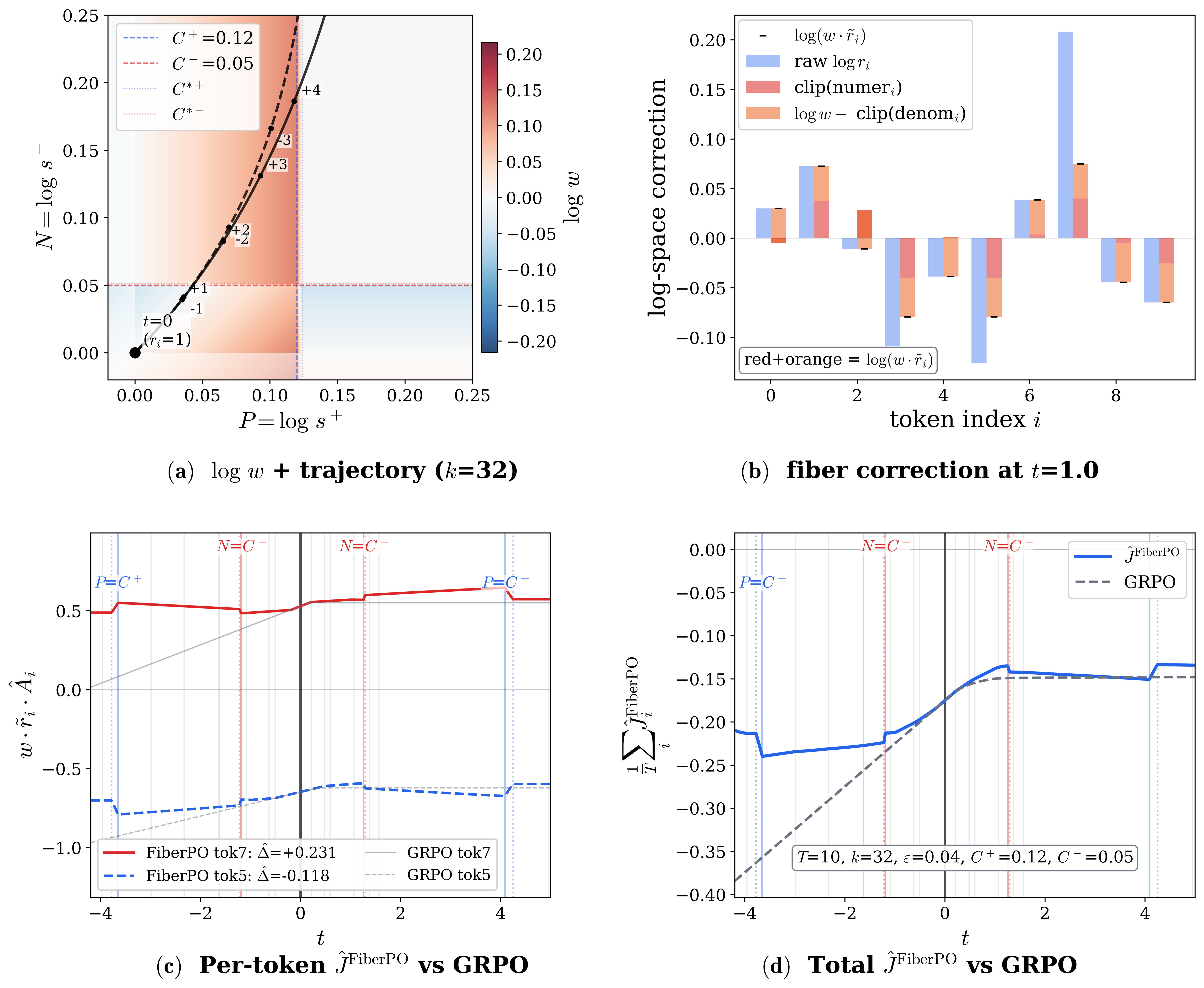}
    \caption{Per-token and total FiberPO objective under parameterized drift $r_i = 1 + t\hat\Delta_i$ ($T{=}10$, $k{=}32$, $\epsilon{=}0.04$, $C^+{=}0.12$, $C^-{=}0.05$), with $\operatorname{sign}(\hat\Delta_i) = \operatorname{sign}(\hat A_i)$, reflecting the typical policy update in which tokens with positive advantage increase their probability ($\hat\Delta_i > 0$) and tokens with negative advantage decrease their probability ($\hat\Delta_i < 0$). GRPO in the graph uses the same $\epsilon$-clip as FiberPO.\\[0.5em]
    \textbf{(a)}~Trajectory in $(\log s^+, \log s^-)$-space overlaid on the $\log w$ heatmap; as drift increases, the trajectory moves through the global regimes of Figure~\ref{fig:fiber_weight}(b).
    \textbf{(b)}~Per-token fiber correction at $t{=}1$: fiber residual (red) and base-weight contribution (orange) stack to form the full gated log-ratio.
    \textbf{(c)}~Per-token objective for two representative tokens.
    \textbf{(d)}~Total objective: FiberPO's $g^{\rm agg}$ and $\operatorname{logclip}$ saturate, causing the objective to return to baseline.
    Vertical markers: $\epsilon$-clip transitions (gray, L-I$\to$L-II), $C^-$ rollback (red, G-II,r onset), $C^+$ rollback (blue, G-II,r onset).}
    \label{fig:fiberpo_objective}
\end{figure}

Figure~\ref{fig:fiberpo_objective} traces the per-token and total objective under varying policy drift.
The key features are:
\begin{itemize}
\item \textbf{Local regulation engages first} (gray vertical
  lines, panels c,\,d): individual tokens hit the $\epsilon$-clip
  boundary ($|u_i| = \epsilon$, L-I$\to$L-II transition) well before the
  trajectory-level aggregates reach the budget threshold. At this
  stage $w^{\rm base}$ is in G-I, so only per-token
  spikes are capped while the trajectory-level
  importance-sampling correction remains intact.
\item \textbf{Rollback activates on aggregate drift} (between solid and dotted red/blue
  vertical lines): once $\log s^-_\tau$ exceeds $C^-$ (or $\log
  s^+_\tau$ exceeds $C^+$), $g^{\rm agg}$ enters rollback regime. The
  base weight $w^{\rm base}_\tau$ begins to oppose the drift
  direction, producing a restorative gradient (G-II,r).
\item \textbf{Non-monotonic pattern} (panel d): the total objective
  rises, peaks as $|\log w^{\rm base}_\tau|$ reaches its maximum
  (G-II,r), then declines through mutual rollback (G-III,r) and
  collapses to the no-reweighting baseline when both channels are
  fully gated (G-III, $w^{\rm base}_\tau = 1$).
\end{itemize}

\subsection{Theoretical Analysis}
\label{subsec:fiberpo_analysis}

Having defined the FiberPO objective and its two-scale gating structure, we now analyze the gradient-level properties that distinguish it from existing proximal methods.

\paragraph{Gradient Decomposition via the Ratio Gating Jacobian.}
\label{subsubsec:jacobian}

The gradient of a generic RGF objective (Definition~\ref{def:rgf}), and thus of FiberPO, with respect to the model parameters $\theta$ admits the decomposition:
\begin{equation}\label{eq:grad_decomp}
    \frac{\partial \hat J(\theta|\theta_{\rm old})}{\partial \theta}=\sum_{k\in {\cal E}}\frac{\partial z_k}{\partial \theta}\frac{\partial \hat J}{\partial z_k}=
    \sum_{i,j,k \in {\cal E}}{ \underbrace{ \frac{\partial z_k}{\partial \theta}\,\mu_{i}}_{\substack{\text{parametrization-related}\\\text{prior weights}}}
    \underbrace{\underbrace{\frac{\partial \mathcal G( r_\bullet)_{i}}{\partial r_{j}}}_{\mathfrak J(\mathcal G)_{ij}}\;\underbrace{\left( r_j(\delta_{jk}-(\pi_\theta)_k) \right)}_{\mathbf\Pi_{jk}\,=\,\frac{\partial r_j}{\partial z_k}}}_{\text{ratio gating factor}}
    \underbrace{\hat A_{\psi( i)}}_{\text{advantage force}}
       },
\end{equation}
where $z_i$ is the logit parameter of the policy, i.e., $\pi_\theta(\cdot |s)=\operatorname{softmax}(\mathbf z(\theta,s))$ (see Appendix~\ref{app:ratio_space}), $\delta_{jk}$ is the Kronecker delta, $(\pi_\theta)_k \equiv \pi_\theta(a_k|s_k)$ is the policy probability at the state-action pair indexed by $k \in \mathcal E$, and $\psi(i)$ maps an augmented index $i \in \mathcal E$ to its underlying state-action pair $(s,a)$.

The gradient decomposes into three factors: the parametrization-related prior weights $\frac{\partial z_k}{\partial \theta}\mu_{i}$, which are fixed a priori; the advantage force $\hat A_{\psi(i)}$, which comes from the reward model and is also given; and the ratio gating factor, which is the only component determined by the choice of surrogate objective.
Within the ratio gating factor, the matrix $\mathbf \Pi\equiv {\bigl[ r_j(\delta_{jk}-(\pi_{\theta})_k) \bigr]}_{j,k\in {\cal E}}$ is simply the transformation matrix from logit space to ratio space $[\frac{\partial r_j}{\partial z_k}]$.
Therefore, the Jacobian of the ratio gating map, $\mathfrak J(\mathcal G)_{ij} \equiv \frac{\partial \mathcal G( r_\bullet)_{i}}{\partial r_{j}}$, is the central object that determines how different surrogate objectives shape the policy update.

\paragraph{The FiberPO Ratio Gating Jacobian.}
\label{subsubsec:fiberpo_jacobian}

To characterize how this two-scale gating shapes gradient flow, we compute the Jacobian of the ratio transform $\mathcal G$.
Since $\mathcal G_i = \exp(\log \mathcal G_i)$, the Jacobian $\mathfrak J(\mathcal G)_{ij} := \frac{\partial \mathcal G(r_\bullet)_i}{\partial r_j}$ factors as $\frac{\mathcal G_i}{r_j}\cdot\frac{\partial \log \mathcal G_i}{\partial \log r_j}$.
Using the gating decomposition~\eqref{eq:fiberpo_gating_decomp} and differentiating through $g^{\rm agg}$, $\operatorname{logclip}$, and the aggregate ratios $s_\tau^\pm$, one obtains:

\begin{proposition}[Jacobian of FiberPO ratio transform]\label{prop:fiberpo_jacobian}
The Jacobian of the FiberPO ratio gating map is:
\begin{equation}\label{eq:fiberpo_jacobian}
    \mathfrak J(\mathcal G)_{ij} = \frac{\mathcal G(r_\bullet)_i}{r_j}\;\mathbb I_{\operatorname{tj}_i=\operatorname{tj}_j} \left[\;\underbrace{\mathbb I_{|u_i|\leq\epsilon}\;\mathbb I_{i=j}}_{\text{(a) local self-gating}} \;+\; \frac{1}{T_{\operatorname{tj}_i}}\underbrace{\left(\gamma_{\operatorname{tj}_i}^{(l_j)} - \mathbb I_{|u_i|\leq\epsilon}\;\mathbb I_{l_i=l_j} - \mathbb I_{|v_i|\leq\epsilon}\,(1-\mathbb I_{l_i=l_j})\right)}_{\text{(b) trajectory-mediated coupling}}\;\right],
\end{equation}
where:
\begin{itemize}
    \item $\mathrm{tj}_i$ represents the trajectory which
      state-action pair in $i$ belongs to,
    \item $u_i := l_i\log r_i - \log s_{\operatorname{tj}_i}^{(l_i)}$ is the fiber residual of token $i$ (Eq.~\ref{eq:fiber_residual}),
    \item $v_i := -\log s_{\operatorname{tj}_i}^{(-l_i)}$ is the opposite-sign aggregate (Eq.~\ref{eq:opposite_sign_agg}),
    \item $\gamma_\tau^{(l)} := (g^{\rm agg})'(\log
      s_\tau^{(l)},\;C^{(l)},\;T_\tau)$ is the derivative of the
      base gate (Eq.~\ref{eq:gagg_main}), taking the value $1$ in the pass-through regime
      ($|\log s_\tau^{(l)}|\leq C^{(l)}$), $-T_\tau$ in the rollback
      regime ($C^{(l)} < |\log s_\tau^{(l)}| <
      (1+T_\tau^{-1})C^{(l)}$), and $0$ in the zeroed regime ($|\log s_\tau^{(l)}| \geq (1+T_\tau^{-1})C^{(l)}$).
\end{itemize}
\end{proposition}

\noindent
Three structural properties follow from this expression.

\begin{property}[Trajectory independence]\label{prop:traj_independence}
The factor $\mathbb I_{\operatorname{tj}_i=\operatorname{tj}_j}$ ensures that tokens in different trajectories are fully decoupled: $\mathfrak J(\mathcal G)_{ij} = 0$ whenever $\operatorname{tj}_i \neq \operatorname{tj}_j$.
Equivalently, the Jacobian is block-diagonal over trajectories, so each trajectory's gradient is self-contained.\footnote{It is still worth mention that there might still be cross-trajectory coupling via the parametrization-related prior weights $\frac{\partial z_k}{\partial \theta}$ in the full gradient decomposition (Eq.~\ref{eq:grad_decomp}), however, (i) a canonical choice of parametrization should have different parameters controlling for unrelated trajectories, thus typically leading to sparse $\frac{\partial z_k}{\partial \theta}$ and thus minimal unrelated cross-trajectory coupling. In contrast, the related trajectory-level coupling of parametrization is in fact desirable, as it allows more stable and robust gradient direction provided by correlated gradients of related trajectories. (ii) The parametrization-related prior weights $\frac{\partial z_k}{\partial \theta}$ are independent of the choice of surrogate objective and thus is not possible to further control with surrogate formulations. The block-diagonal Jacobian already provides maximal decoupling that can be achieved at the surrogate level.}
\end{property}

\begin{property}[First-order agreement]\label{prop:first_order}
At the on-policy point ($r_\bullet = \mathbf 1$), the Jacobian reduces to $\mathfrak J(\mathcal G)_{ij} = \mathbb I_{i=j}/r_j$, the identity gating.
The FiberPO surrogate therefore recovers the linear surrogate $J^{(1)}(\theta|\theta_{\rm old})$ and the true RL objective to first order near on-policy (Theorem~\ref{thm:fbg_first_order}).
\end{property}

\begin{property}[Scale separation]\label{prop:scale_separation}
The local self-gating term~(a) in Eq.~\ref{eq:fiberpo_jacobian}
has $O(1)$ magnitude. The coupling term~(b) is weighted by
$1/T_\tau$, so trajectory-mediated effects scale inversely with
sequence length. In the near-on-policy regime, the local gradient
dominates; trajectory-level corrections become significant only
as aggregate drift grows.
\end{property}

%% file: rlvr_fiberpo_sec4.tex
\section{Fibration Hierarchy and FiberPO-Domain}
\label{subsec:fibration_hierarchy}

The fiber bundle model of Section~\ref{subsubsec:fiber_bundle}
captures a single level of abstraction: tokens are grouped into
trajectories, and Fiber Bundle Gating decomposes stability
control into a base gate (trajectory level) and a fiber gate
(token level). In practice, however, dependency and
classification in LLM RL training are inherently hierarchical.
Tokens belong to trajectories, trajectories belong to prompt
groups, and prompt groups belong to domains such as mathematics,
code generation, and instruction following. There is no reason to
restrict the fibration to a single level, and the algebraic
structure of FBG extends naturally to a chain of fibrations.

This generalization is not merely a matter of symbolic
generality. Different levels of the hierarchy demand different
trust-region radii for substantive reasons. First, different
domains may be trained at different stages and therefore have
different proximity to the optimum, requiring appropriately sized
per-domain trust regions (a radius that is too large causes
instability, while one that is too small leads to premature
convergence to a local optimum). Second, parameters associated
with each hierarchical level carry different degrees of
importance in the network architecture and should accordingly be
given different amounts of flexibility. Third, even the positive
and negative drift channels within each level demand different
budgets, to avoid overly optimistic updates from accidental
rewards or excessive punishment of exploratory behavior. Existing
methods provide only coarse control over these demands: GRPO
applies a uniform per-token clip, GSPO operates at a single
trajectory level, and APC-Obj and TRPO enforce a global aggregate
trust region that does not discriminate across finer structures.
FiberPO-Trajectory (Section~\ref{subsec:fiberpo}) improves on
these by coupling trajectory-level and token-level control, but
it operates at only two levels of the hierarchy. In this section,
we show that the fibration formalism generalizes naturally to
arbitrarily many hierarchical levels, providing a principled
compositional framework for multi-scale trust-region control. We
formalize this generalization as the Uniform Fibration Hierarchy
and its associated Fibration Gating Hierarchy, and demonstrate
its utility by deriving FiberPO-Domain, a concrete four-level
instantiation (domain, prompt group, trajectory, token). The same
construction extends readily to any level of hierarchical
complexity that a given training setup requires.

\subsection{Uniform Fibration Hierarchy and Fibration Gating Hierarchy}
\label{subsubsec:fgh}

We begin by extending the fiber bundle model to a chain of
fibrations, each representing one level of abstraction.

\begin{definition}[Uniform Fibration Hierarchy]\label{def:uniform_fibration}
  A \emph{uniform fibration hierarchy} of depth $n$ is a chain of
  spaces $\{B_k\}_{k\in\{0,\ldots,n\}}$ (called \emph{strata})
  together with fibrations $\{\pi_k\}_{k\in\{1,\ldots,n\}}$,
  where each $\pi_k: B_k \to B_{k-1}$ is a fibration
  projection:
\begin{equation}\label{eq:fibration_chain}
B_0 \overset{\pi_1}{\longleftarrow} B_1 \overset{\pi_2}{\longleftarrow} B_2
\overset{\pi_3}{\longleftarrow} \cdots \overset{\pi_n}{\longleftarrow} B_n \equiv E.
\end{equation}
The chain of fibrations naturally induces a chain of pushforward maps on density spaces:
\begin{equation}\label{eq:pushforward_chain}
\mathbf{B}_0 \overset{\pi_{1*}}{\longleftarrow} \mathbf{B}_1
\overset{\pi_{2*}}{\longleftarrow} \mathbf{B}_2
\overset{\pi_{3*}}{\longleftarrow} \cdots
\overset{\pi_{n*}}{\longleftarrow} \mathbf{B}_n \equiv \mathbf{E},
\end{equation}
where $\mathbf{B}_k$ denotes the space of densities on $B_k$.
\end{definition}

\noindent
Building on the uniform fibration hierarchy, we equip each stratum with gating
infrastructure analogous to Fiber Bundle Gating (Definition~\ref{def:fbg}).

\begin{definition}[Fibration Gating Hierarchy (FGH)]\label{def:fgh}
A \emph{Fibration Gating Hierarchy} consists of the following objects:
\begin{enumerate}
    \item[\textbf{Object~1.}] \textbf{Uniform fibration hierarchy}
    $\{B_k, \pi_k\}$ (Definition~\ref{def:uniform_fibration}).

    \item[\textbf{Object~2.}] \textbf{Reflecting Markov kernels}
    $\{K_k\}_{k\in\{0,\ldots,n{-}1\}}$, $K_k: \mathbf{B}_k \to \mathbf{B}_{k+1}$:
    \[
    \mathbf{B}_0 \overset{K_0}{\longrightarrow} \mathbf{B}_1
    \overset{K_1}{\longrightarrow} \mathbf{B}_2
    \overset{K_2}{\longrightarrow} \cdots
    \overset{K_{n-1}}{\longrightarrow} \mathbf{B}_n \equiv \mathbf{E}.
    \]

    \item[\textbf{Object~3.}] \textbf{Atomic gating functions} on each stratum,
    $\{g_{k,p_{<k}}\}_{k\in\{0,\ldots,n\}}$:
    \[
    g_{k,p_{<k}}: \mathbf{B}_k \to \mathbf{B}_k,
    \]
    where $g_{k,p_{<k}}$ may depend on the densities on all lower strata
    $p_{<k} \equiv (p_0, p_1, \ldots, p_{k-1})$.

    \item[\textbf{Object~4.}] \textbf{Fibration decomposition and recovery maps}, the
    same as in FBG (Definition~\ref{def:fbg}):
    $\mathcal{F}: \mathbb{R}^{\mathcal{E}} \to \mathbf{E}$, \quad
    $\mathcal{R}: \mathbf{E} \to \mathbb{R}^{\mathcal{E}}$.
\end{enumerate}

\noindent
These objects are subject to the following conditions:
\begin{enumerate}[label=(\roman*)]
    \item \textbf{Reflecting condition:}
    $\pi_{k*} \circ K_{k-1} = \mathrm{id}_{\mathbf{B}_{k-1}}$ for $0 < k \leq n$,
    equivalently
    $\pi_{k*}(\sigma - K_{k-1}(\pi_{k*}\sigma)) = 0$, ensuring that the residual at
    each level carries no lower-stratum information.

    \item \textbf{Recovery condition:}
    $\mathcal{R} \circ \mathcal{F} = \mathrm{id}_{\mathbb{R}^{\mathcal{E}}}$.

    \item \textbf{Locality condition:}
    $g_{k,p_{<k}} = \bigoplus_{b \in B_k} g_{k,p_{<k},b}$, i.e.,
    $g_{k,p_{<k}}(\sigma_k)_b = g_{k,p_{<k},b}((\sigma_k)_b)$ for all
    $\sigma_k \in \mathbf{B}_k$ and $b \in B_k$. Each atomic gating function acts
    pointwise over the elements of its stratum.
\end{enumerate}
\end{definition}

\noindent
For notational convenience, we define sequential compositions of the fibration and
kernel maps. Let
\begin{align}
K_{m\leftarrow k} &:= K_{m-1} \circ K_{m-2} \circ \cdots \circ K_{k+1} \circ K_k
\quad \text{for } k < m, \qquad K_{k\leftarrow k} := \mathrm{id}_{\mathbf{B}_k},
\label{eq:K_composition}\\
\pi_{k\leftarrow m} &:= \pi_{k+1} \circ \pi_{k+2} \circ \cdots \circ \pi_{m-1} \circ \pi_m
\quad \text{for } k < m, \qquad \pi_{k\leftarrow k} := \mathrm{id}_{B_k},
\label{eq:pi_composition}\\
\pi_{k\leftarrow m\,*} &:= \pi_{k+1\,*} \circ \pi_{k+2\,*} \circ \cdots \circ \pi_{m-1\,*} \circ \pi_{m*}
\quad \text{for } k < m, \qquad \pi_{k\leftarrow k\,*} := \mathrm{id}_{\mathbf{B}_k}.
\label{eq:pi_push_composition}
\end{align}
We also write $\pi_{<k\,*}\sigma \equiv (\pi_{0\leftarrow n\,*}\sigma,\;
\pi_{1\leftarrow n\,*}\sigma,\; \ldots,\; \pi_{k-1\leftarrow n\,*}\sigma)$
for $\sigma \in \mathbf{E} \equiv \mathbf{B}_n$, collecting the pushed-forward
densities on all strata below level~$k$.

\medskip\noindent
The FGH gating function generalizes the FBG operator~\eqref{eq:fbg} by gating the
residual at each stratum and reflecting all contributions back to the total space:
\begin{equation}\label{eq:fgh_gating}
G(\sigma) \;:=\; K_{n\leftarrow 0} \circ g_0 \circ \pi_{0\leftarrow n\,*}(\sigma)
\;+\; \sum_{0 < k \leq n} K_{n\leftarrow k} \circ
g_{k,\pi_{<k\,*}\sigma}\!\left(\pi_{k\leftarrow n\,*}\sigma
- K_{k-1}\!\left(\pi_{k-1\leftarrow n\,*}\sigma\right)\right),
\end{equation}
which admits the compact form
\begin{equation}\label{eq:fgh_gating_compact}
G(\sigma) \;=\; \sum_{0 \leq k \leq n} K_{n\leftarrow k} \circ G_k \circ
\pi_{k\leftarrow n\,*}(\sigma),
\end{equation}
where the per-stratum operator $G_k$ is defined by
\begin{equation}\label{eq:Gk_def}
G_k(\sigma_k) := \begin{cases}
g_0(\sigma_k) & \text{if } k = 0, \\[0.3em]
g_{k,\pi_{<k\,*}\sigma}\!\left(\sigma_k - K_{k-1} \circ \pi_{k*}(\sigma_k)\right)
& \text{otherwise.}
\end{cases}
\end{equation}
The structure mirrors that of Fiber Bundle Gating: the density $\sigma$ is pushed forward
to each stratum $\mathbf{B}_k$ via $\pi_{k\leftarrow n\,*}$, the contribution already
accounted for at stratum $k{-}1$ is subtracted through
$\sigma_k - K_{k-1} \circ \pi_{k*}(\sigma_k)$, the residual is gated by
$g_{k,\pi_{<k\,*}\sigma}$, and all gated components are reflected back to $\mathbf{E}$
via $K_{n\leftarrow k}$ and summed.

\medskip\noindent
The associated FGH surrogate objective is defined analogously to the FBG gating
form~\eqref{eq:fbg_rgf}:
\begin{equation}\label{eq:fgh_objective}
\hat{J}^{\rm FGH}(\theta|\theta_{\rm old})
= \sum_{i \in \mathcal{E}} \mu_i \;\mathcal{R} \circ G \circ
\mathcal{F}(r_\bullet(\theta|\theta_{\rm old}))_i \;\hat{A}_i.
\end{equation}

\begin{remark}[Recovery of Fiber Bundle Gating]\label{rem:fbg_recovery}
When $n = 1$, the fibration hierarchy reduces to a single fiber bundle $\pi_1: B_1 \equiv E \to B_0 \equiv B$, and the FGH definitions recover exactly the FBG framework of Definition~\ref{def:fbg}. In this case, the gating function~\eqref{eq:fgh_gating} reduces to $G(\sigma) = K_0(g_0(\pi_{1*}\sigma)) + g_{1,\pi_{1*}\sigma}(\sigma - K_0(\pi_{1*}\sigma))$, which is precisely~\eqref{eq:fbg}. The FiberPO objective (Definition~\ref{def:fiberpo_main}), which operates at the trajectory and token levels, corresponds to $n = 1$.
\end{remark}

\subsection{FiberPO-Domain}
\label{subsubsec:domain_fiberpo}

We now instantiate the Fibration Gating Hierarchy with $n = 3$ to obtain FiberPO-Domain,
which extends FiberPO's two-level gating (trajectory, token) to a four-level hierarchy:
domain, prompt group, trajectory, and token.

We first introduce the relevant notation. Let $D \in
\mathrm{Domain}$ denote a domain index (e.g., mathematics, code,
instruction following), and let $D_g$ denote the domain of prompt
group~$g$. For simplicity, we assume that training prompts are
not composite domain tasks, so that we may also write $D_\tau$
for the domain of trajectory $\tau$ and $D_i:=D_\tau$ for the
domain of an augmented state-action pair $i \equiv (s, a, \tau,
t) \in \bar{\mathcal{X}}$. The prompt group index of
trajectory~$\tau$ is denoted $g_\tau$, and $g_i$ denotes the
prompt group of augmented pair~$i$. Let $\tau_i$ denote the
trajectory to which $i$ belongs, $T_\tau$ the length of
trajectory $\tau$, $T_D := \sum_{\tau \in D} T_\tau$ the total
number of state-action pairs classified to domain~$D$, and $T_g
:= \sum_{\tau \in g} T_\tau$ the total number of state-action
pairs within prompt group~$g$. We write $\mathrm{Tj}^{\theta_{\rm
    old}}(D)$ for the set of trajectories in domain~$D$ and
$\mathrm{Tj}^{\theta_{\rm old}}(g)$ for those in prompt
group~$g$.

\begin{definition}[FiberPO-Domain]\label{def:domain_fiberpo}
The FiberPO-Domain objective is:
\begin{equation}\label{eq:domain_fiberpo}
\hat{J}^{\text{FiberPO-Domain}}(\theta|\theta_{\rm old})
= \sum_{i \in \bar{\mathcal{X}}} \frac{1}{|\mathrm{Tj}^{\theta_{\rm old}}|}
\,\frac{1}{T_\tau}\; w_i^{\rm Base} \cdot \tilde{r}_i^{\rm Fiber} \cdot \hat{A}_i.
\end{equation}
The base weight decomposes into six gated aggregate terms, three
per sign channel, corresponding to the domain, prompt group, and
trajectory levels:
\begin{equation}\label{eq:domain_fiberpo_base_weight}
w_i^{\rm Base} := \exp\!\left(\begin{array}{rl}
& g^{\rm agg}(\log s^+_{D_i},\, C^+,\, T_{D_i})
\;+\; g^{\rm agg}(\log s^+_{g_i} - \log s^+_{D_i},\, C^+,\, T_{g_i})
\;+\; g^{\rm agg}(\log s^+_{\tau_i} - \log s^+_{g_i},\, C^+,\, T_{\tau_i})
\\[0.5em]
-\; & g^{\rm agg}(\log s^-_{D_i},\, C^-,\, T_{D_i})
\;-\; g^{\rm agg}(\log s^-_{g_i} - \log s^-_{D_i},\, C^-,\, T_{g_i})
\;-\; g^{\rm agg}(\log s^-_{\tau_i} - \log s^-_{g_i},\, C^-,\, T_{\tau_i})
\end{array}\right),
\end{equation}
where the aggregate gate $g^{\rm agg}$ is defined in
Eq.~\ref{eq:gagg_main}. The gated residual is identical to
FiberPO's fiber gate (Eq.~\ref{eq:gated_residual}):
\begin{equation}\label{eq:domain_fiberpo_residual}
\tilde{r}_i^{\rm Fiber}
:= \frac{\operatorname{logclip}\!\bigl((s_{\tau_i}^{(l_i)})^{-l_i}\, r_i,\;
\epsilon\bigr)}{\operatorname{logclip}\!\bigl((s_{\tau_i}^{(-l_i)})^{-l_i},\;
\epsilon\bigr)},
\end{equation}
where $l_i := \operatorname{sign}(\log r_i)$. The
prompt-group-level and domain-level aggregate log-ratios
appearing in the base weight are defined as:
\begin{equation}\label{eq:group_domain_aggregates}
\log s_g^{(l)} := \frac{1}{|\mathrm{Tj}^{\theta_{\rm old}}(g)|}
\sum_{\tau \in g} \log s_\tau^{(l)}, \qquad
\log s_D^{(l)} := \sum_{g \in D}
\frac{|\mathrm{Tj}^{\theta_{\rm old}}(g)|}{|\mathrm{Tj}^{\theta_{\rm old}}(D)|}
\log s_g^{(l)}.
\end{equation}
The associated FGH objects (strata, reflecting kernels, atomic gates, and
decomposition/recovery maps) that realize this objective as an $n{=}3$ FGH instance
are specified in Appendix~\ref{app:fiberpo_domain_fgh}.
\end{definition}

\noindent
The base weight~\eqref{eq:domain_fiberpo_base_weight} has a
transparent hierarchical structure. Each of the three positive
terms gates the residual drift at one level of the hierarchy:
$\log s^+_{D_i}$ captures domain-level drift, $\log s^+_{g_i} -
\log s^+_{D_i}$ captures the prompt group's deviation from its
domain aggregate, and $\log s^+_{\tau_i} - \log s^+_{g_i}$
captures the trajectory's deviation from its prompt group
aggregate. The negative channel is treated symmetrically. At each
level, $g^{\rm agg}$ independently maintains a trust-region
budget with the same three-regime (pass-through, rollback,
zeroed) behavior as in FiberPO-Trajectory
(Section~\ref{subsubsec:base_weight}), scaled by the appropriate
count ($T_{D_i}$, $T_{g_i}$, or $T_{\tau_i}$).

FiberPO-Domain inherits the structural merits that
FiberPO-Trajectory demonstrates at the trajectory and token
levels. The first-order agreement with the true RL objective near
on-policy extends to the full hierarchy under the same identity
conditions on the atomic gates (Theorem~\ref{thm:fbg_first_order}
applied to the FGH framework). The ratio decoupling property
(Section~\ref{subsubsec:ratio_decoupling}), which ensures that
tokens are regulated by their deviation from the trajectory mean
rather than by their absolute log-ratio, now extends upward: each
trajectory is regulated by its deviation from the prompt group
aggregate, and each prompt group is regulated by its deviation
from the domain aggregate. The rollback mechanism operates
independently at each level, so that budget violation at one
level of the hierarchy does not contaminate the gating at other
levels. This information decoupling between hierarchical levels
provides substantially more versatile optimization control for
multi-domain LLM RL training than any single-scale or two-scale
method.

The FiberPO-Domain objective~\eqref{eq:domain_fiberpo} admits an
equivalent FGH gating form $\mathcal{G}(r_\bullet)_i = w_i^{\rm
  Base} \cdot \tilde{r}_i^{\rm Fiber} = \mathcal{R} \circ G \circ
\mathcal{F}(r_\bullet)_i$ corresponding to the $n = 3$ case of
Definition~\ref{def:fgh}. The detailed specification of the
constituent FGH objects (fibration hierarchy, reflecting Markov
kernels, atomic gating functions, and decomposition/recovery maps)
is given in Appendix~\ref{app:fiberpo_domain_fgh}.

\begin{remark}[Heterogeneous fibration systems]\label{rem:heterogeneous_fibration}
  The uniform fibration hierarchy assumes a linear chain of
  strata $B_0 \leftarrow B_1 \leftarrow \cdots \leftarrow B_n$,
  which suffices for the domain--prompt-group--trajectory--token
  hierarchy considered here. A more general construction replaces
  the linear chain with a rooted tree (or more generally, a
  quiver), where each edge carries a fibration and each vertex
  carries a gating function. Given a rooted tree with source and
  target maps $s, t: \mathrm{Edge} \to \mathrm{Vertex}$ (edges
  directed toward the root), one assigns a space $F_v$ to each
  vertex and a fibration $F_1(e): F_{s(e)} \to F_{t(e)}$ to each
  edge, together with reflecting kernels and atomic gates at each
  vertex. The resulting gating operator aggregates gated
  residuals from all vertices back to the total leaf space,
  generalizing the linear sum~\eqref{eq:fgh_gating} to a
  tree-indexed sum. This heterogeneous fibration system
  accommodates branching hierarchies (e.g., separate domain and
  modality classifications at the same level) and provides a
  framework for expressing gating algebras over arbitrary
  dependency structures.
\end{remark}

%% file: rlvr_conclusion.tex
\section{Conclusion}

This paper develops an algebraic approach to multi-scale stability
control in reinforcement learning for large language models. We
begin by proving that classical TRPO trust-region guarantees
collapse at discount factor $\gamma = 1$, the regime required by
sparse-reward LLM tasks, motivating a separation of the
trust-region maintenance mechanism from the specific radius
prescribed by the classical bound. We then derive Aggregational Policy Censoring Objective (APC-Obj), an unconstrained
clipping-based surrogate that is provably equivalent to
sample-based TV-TRPO under complete parametrization, establishing
that clipping-based surrogate design and trust-region policy
optimization are dual formulations of the same problem. Through
the Ratio Gating Formalism (RGF), we show that PPO, GRPO, and
GSPO each arise from APC-Obj via identifiable relaxation steps,
making each method's departure from the trust-region optimum
explicit.

Building on this analytical foundation, we introduce Fiber Bundle
Gating (FBG), a framework that organizes sampled RLHF data as a
fiber bundle and decomposes ratio gating into coordinated global
and local components operating on densities over the base and
total spaces. FBG preserves first-order agreement with the true
RL objective near on-policy whenever the atomic gates reduce to
identity at the reference point. From a relaxed APC-Obj formulation,
we derive FiberPO, a concrete FBG instantiation whose
trajectory-level aggregate gate and token-level residual gate
provide independent trust-region budgets at two scales. The
FiberPO Jacobian is block-diagonal over trajectories, reduces to
identity at on-policy, and exhibits a restorative gradient
structure in the rollback regime that is absent in all prior
methods.

To our knowledge, this is the first work to provide a principled
algebraic framework that unifies trust-region maintenance,
multi-scale gating, and gradient-level guarantees within a single
coherent formalism. A natural question is whether the fiber
bundle apparatus is necessary, or whether the same multi-scale
gating could be achieved with a carefully hand-crafted structured
clipping rule. We observe that methods coupling global and local
stability control necessarily involves a dependency structure in
which many local quantities depend on a shared global context.
This dependency defines a fibration $\pi_E: E \to B$, and the
three objects of a fiber bundle, the total space~$E$, the base
space~$B$, and the projection~$\pi_E$, correspond directly to the
components already present in the data. Methods that aggregate
local quantities into context-level statistics implicitly invoke
the pushforward $\pi_{E*}$, and methods that distribute
context-level information back to local entries implicitly
construct a Markov kernel~$K$. The formalism makes these
operations explicit and thereby exposes the reflecting condition
$\pi_{E*} \circ K = \mathrm{id}_{\mathbf{B}}$ (for methods
coupling global and local information to satisfy), a structural
constraint ensuring that the residual $\sigma -
K(\pi_{E*}\sigma)$ carries no base-level information
(Appendix~\ref{app:fbg_transmit}). Without this condition, the
fiber gate re-gates global information already handled by the
base gate, or the base gate inadvertently constrains genuinely
local variation. Designing a gating map that avoids this
double-counting while preserving first-order agreement
(Theorem~\ref{thm:fbg_first_order}) amounts to satisfying exactly
the reflecting condition, i.e., implicitly re-deriving the fiber
bundle decomposition. The formalism identifies and makes explicit
an underlying decomposition structure that multi-scale gating
has, and provides a principled framework for reasoning about its
correctness.

Because fibrations compose algebraically, the same three-object
template chains into the Fibration Gating Hierarchy (FGH) without
introducing new primitives, with each stratum's residual
capturing the purest local variation at that level after all
coarser influence has been removed. FiberPO-Domain (domain,
prompt group, trajectory, token) is one such instantiation, and
deeper hierarchies follow by appending another fibration layer.
Achieving the same compositional guarantee with ad-hoc notation
would require re-proving orthogonality and first-order agreement
at every new level. The algebraic compositionality of fibrations
provides this automatically. This positions FBG as a foundation
for the stability challenges that arise in large-scale agentic
systems and heterogeneous multi-domain training.

%% file: rlvr_appendix_notations.tex
\section{Notation and Preliminaries}
\label{app:notation}

\subsection{On MDP}
\label{app:mdp}

For a given MDP, we denote its action space $\mathcal A$, state space $\mathcal S$, state transition probability $\mathcal T(s'|s,a)$, its $\theta$-parameter dependent policy $\pi_\theta(a|s)$, and reward function $R:\mathcal S\to\mathbb R$.

We denote the space of all signed measures over a space $X$ (equipped with a sigma algebra on it) as $\mathbb M(X)$. The space of probability measures is denoted $\mathbb P(X)$, and the space of smooth probability measures is denoted $C^{\infty}\mathbb P(X)$ when $X$ is equipped with a smooth structure.

We denote $\mathcal P_{p_0,\pi_{\theta},\mathcal T}$ (or simply $\mathcal P_\theta$ if $\mathcal T$ and $p_0$ are clear or assumed from context) as the probabilistic distribution of RL trajectories given the prior distribution $p_0$ for initial states, the $\pi_\theta$ action policy, and the state transition distribution $\mathcal T$.

The discounted RL objective is denoted
\[
    J(\theta):=\mathbb E_{\tau \sim \mathcal P_\theta}\Bigl[\sum_{t=0}^{\infty }\gamma^t R(s_t(\tau))\Bigr].
\]

Given a reference policy parameter $\theta_{\rm old}$, we define:

\paragraph{Policy-likelihood ratio.}
The policy-likelihood ratio (or policy ratio) is defined to be
\[
    r_{s,a}(\theta|\theta_{\rm old}):=\frac{\pi_{\theta}(a|s)}{\pi_{\theta_{\rm old}}(a|s)},
    \quad a\in\mathcal A,~s\in\mathcal S.
\]
If $\theta,\theta_{\rm old}$ are clear from context, we abbreviate this as $r_{s,a}$.

\paragraph{Linear surrogate RL objective.}
The linear surrogate RL objective (with respect to $\theta_{\rm old}$) is denoted
\[
    J^{(1)}(\theta|\theta_{\rm old}):= J(\theta_{\rm old})+{\mathbb E}_{\tau\sim\mathcal P_{\theta_{\rm old}}}\Bigl[   \sum_{t=0}^{\infty}\gamma^t r_{s_{t}(\tau),a_{t}(\tau)}(\theta|\theta_{\rm old})\; A^{({\theta_{\rm old}})}_{s_{t}(\tau),a_{t}(\tau)}\Bigr],
\]
where $ A^{(\theta_{\rm old})}_{s,a}$ is the advantage value for the state-action pair $(s,a)$ given the reference policy $\pi_{\theta_{\rm old}}$.

The linear surrogate RL objective is the first-order approximation to the true RL objective near on-policy ($\theta \approx \theta_{\rm old}$), and thus has the same gradient as the true RL objective at $\theta=\theta_{\rm old}$~\cite{kakade2002approximately}:
\[
    \nabla_\theta J^{(1)}(\theta|\theta_{\rm old})\big|_{\theta=\theta_{\rm old}}=\nabla_\theta J(\theta)\big|_{\theta=\theta_{\rm old}}.
\]

\paragraph{Advantage norm.}
We further denote $\max_{(s,a)\in\mathcal S\times\mathcal A} |A^{(\theta)}_{s,a}|\equiv \|A^{(\theta)}_\bullet\|_{\infty}$, the $L^\infty$ norm of the advantage $A^{(\theta)}_{\bullet}$ (for placeholder $\bullet$ running over all state-action pairs), a notation agreeing with our later convention.

\paragraph{Discounted state distribution.}
We denote the discounted state distribution $\rho^{\pi_\theta}_{(\gamma)}$, defined as
\[
    \rho^{\pi_{\theta}}_{(\gamma)}(U_{\mathcal S}):=\mathbb E_{\tau\sim\mathcal P_{\theta}}\Bigl[\sum_{ t=0 }^\infty{\gamma^t \mathbb I_{U_{\mathcal S}}(\,s_t(\tau)\,) } \Bigr],
\]
where $U_{\mathcal S}\subseteq\mathcal S$ is a measurable set in state space $\mathcal S$, and $\mathbb I_X(x):=\left\{\begin{array}{ll} 1&\text{if }x\in X \\ 0&\text{otherwise} \end{array}\right.$ is the indicator function. Note that for LLM $\gamma=1$, this infinite sum does not necessarily converge. However, typical LLMs have finite response length (or at least clip overly long responses), so the discounted state distribution can still be defined.

\paragraph{Normalized state distribution.}
We define the normalized state distribution
\[
    \bar\rho^{\pi_\theta}:=\left\{\begin{array}{ll}  \displaystyle (1-\gamma)\rho^{\pi_\theta}_{(\gamma)} & \text{if }\gamma< 1 \\[1em] \displaystyle\lim_{\gamma \to 1^-}(1-\gamma)\rho^{\pi_\theta}_{(\gamma)} & \text{if }\gamma =1  \end{array}\right.
\]
and the trajectory-length-normalized state distribution
\[
    \bar \rho^{\pi_\theta}_{\rm (tj)}(U_{\cal S}):= \mathbb E_{\tau\sim\mathcal P_{\theta}}\Bigl[\sum_{ t=0 }^{T_{\tau}-1}{\frac{1}{T_\tau} \mathbb I_{U_{\mathcal S}}(\,s_t(\tau)\,) } \Bigr],
\]
where $T_\tau$ is the length of trajectory $\tau$.

\subsection{On Estimation for MDP}
\label{app:estimation}

In the context of estimation, we denote $\hat{\mathbb E}_{X\sim p}[X]$ the Monte Carlo estimation of $\mathbb E_{X\sim p}[X]$.

For trajectory distributions $\mathcal P_\theta$, we denote $\mathrm{Tj}^\theta$ the collection of all sampled trajectories, $\mathrm{Tj}^{\theta}(s_0)$ the set of all sampled trajectories from initial state $s_0$, and $\mathcal X$ the collection of all state-action pairs that appeared in sampled trajectories, i.e.,
\[
    \mathcal X:=\{ (s_t(\tau),a_t(\tau))\mid\tau\in\mathrm{Tj}^{\theta} \},
\]
in this given context of estimation.
For later convenience, we also denote $s_0[\mathrm{Tj}^\theta]$ the set of all initial states $s_0$ that appeared in the sampled trajectories.

We denote the augmented (sampled) state-action pair space
\[
    \bar{\cal X}:=\{ (s_t(\tau),a_t(\tau),\tau,t) \}\subseteq \mathcal X\times \mathrm{Tj}^{\pi_{\theta_{\rm old}}}\times \mathbb N
\]
(for mathematical completeness, even if a trajectory terminates at a finite step, we still include the null state-action pair $(\varnothing,\varnothing,\tau,t)$ after termination).

For any state-action pair $(s,a)$, we define $n_{s,a}\in\mathbb N$ as the number of times $(s,a)$ appears in all the sampled trajectories $\mathrm{Tj}^\theta$ (counting multiplicity inside the same trajectory as well). We further denote the normalized counting measure on space $\mathcal X$ as $\mu_{\mathcal X}$:
\[
    \mu_{\mathcal X}(\{(s,a)\})\equiv \mu_{s,a}=\frac{n_{s,a}}{\sum_{(s',a')\in\mathcal X}n_{s',a'}}.
\]

We also introduce, for a trajectory $\tau$, the quantity $n^{(\tau)}_{s,a}$ representing the number of times state-action pair $(s,a)$ appears in trajectory $\tau$. (For LLM, it never goes beyond 1. However, for rigorousness and generality to RL, we still employ this notation; furthermore, we always assign $n_{\varnothing,\varnothing}=0$ for the null state-action pair.) The normalized counting measure $\mu^{(\tau)}_{\mathcal X}$ is associated to it:
\[
    \mu_{\mathcal X}^{(\tau)}(\{(s,a)\})\equiv \mu_{s,a}^{(\tau)}=\frac{n^{(\tau)}_{s,a}}{\sum_{(s',a')\in\mathcal X^{(\tau)}}n^{(\tau)}_{s',a'}},
\]
where $\mathcal X^{(\tau)}$ is the set of all state-action pairs that appear in trajectory $\tau$.

We also denote $\mu_{\bar{\cal X}}^{(\tau)}:=\frac{1}{T_\tau}\mathbb I_{\in \tau}$ the normalized counting measure on trajectory $\tau$ over the augmented state-action space $\bar{\cal X}$, where
\[
    \mathbb I_{\in\tau}((s,a,\tau',t))=\left\{\begin{array}{ll} 1&\text{if }\tau=\tau' \\ 0&\text{otherwise} \end{array}\right.,
\]
and we conveniently denote $\bar \mu^{(\tau)}_{s,a,t}\equiv \bar\mu_{(s,a,\tau,t)}\equiv \mu^{(\tau)}_{\bar{\cal X}}(\{(s,a,\tau,t)\})$.

For convenience, we also define the normalized counting measure on natural numbers $\mathbb N$ (and subsets thereof) with respect to a trajectory as $\mu^{(\tau)}\in \mathbb P(\mathbb N)$, $\mu^{(\tau)}(\{ n\}):=\frac{1}{T_\tau}\mathbb I_{[0,T_\tau)}(n)$.

\subsection{On Divergences and Their Monte-Carlo Estimation}
\label{app:divergences}

We define the total variation (TV) distance and Kullback--Leibler (KL) divergence at three levels---per-state, average/max over states, and trajectory-normalized---and give the distributional definition, the equivalent ratio form, and the Monte-Carlo estimator for each.
All divergences are measured from $\pi_{\theta_{\rm old}}$ to $\pi_\theta$ using the policy ratio $r_{s,a} := \pi_\theta(a|s)/\pi_{\theta_{\rm old}}(a|s)$; we abbreviate $D(\theta_{\rm old}\|\theta)$ when the policies are clear from context.

\subsubsection{Per-State Divergences}
\label{app:per_state_div}

\paragraph{Definitions.}
For a fixed state $s\in\mathcal S$, the per-state TV distance and KL divergence between $\pi_{\theta_{\rm old}}(\cdot|s)$ and $\pi_\theta(\cdot|s)$ are
\begin{equation}\label{eq:dtv_def}
    D_{\rm TV}(s) \;:=\; D_{\rm TV}\!\bigl(\pi_{\theta_{\rm old}}(\cdot|s)\,\big\|\,\pi_\theta(\cdot|s)\bigr)
    \;=\; \frac{1}{2}\sum_{a\in\mathcal A}\bigl|\pi_{\theta_{\rm old}}(a|s)-\pi_\theta(a|s)\bigr|,
\end{equation}
\begin{equation}\label{eq:dkl_def}
    D_{\rm KL}(s) \;:=\; D_{\rm KL}\!\bigl(\pi_{\theta_{\rm old}}(\cdot|s)\,\big\|\,\pi_\theta(\cdot|s)\bigr)
    \;=\; \sum_{a\in\mathcal A} \pi_{\theta_{\rm old}}(a|s)\log\frac{\pi_{\theta_{\rm old}}(a|s)}{\pi_\theta(a|s)}.
\end{equation}
They are related by Pinsker's inequality:
\begin{equation}\label{eq:pinsker}
    D_{\rm TV}(s)^2 \;\leq\; \tfrac{1}{2}\,D_{\rm KL}(s),
\end{equation}
which is used throughout this paper to convert between KL-based and TV-based bounds.

\paragraph{Ratio forms.}
Using $r_{s,a} = \pi_\theta(a|s)/\pi_{\theta_{\rm old}}(a|s)$, the per-state divergences can be rewritten as expectations under $\pi_{\theta_{\rm old}}(\cdot|s)$:
\begin{equation}\label{eq:dtv_ratio}
    D_{\rm TV}(s) \;=\; \frac{1}{2}\,\mathbb{E}_{a\sim\pi_{\theta_{\rm old}}(\cdot|s)}\!\bigl[\,|r_{s,a} - 1|\,\bigr],
\end{equation}
\begin{equation}\label{eq:dkl_ratio}
    D_{\rm KL}(s) \;=\; -\,\mathbb{E}_{a\sim\pi_{\theta_{\rm old}}(\cdot|s)}\!\bigl[\log r_{s,a}\bigr].
\end{equation}
The TV ratio form~\eqref{eq:dtv_ratio} follows from factoring out $\pi_{\theta_{\rm old}}(a|s)$ in~\eqref{eq:dtv_def}:
\begin{align*}
    D_{\rm TV}(s)
    &= \frac{1}{2}\sum_{a\in\mathcal A}\bigl|\pi_{\theta_{\rm old}}(a|s)-\pi_\theta(a|s)\bigr|
    = \frac{1}{2}\sum_{a\in\mathcal A}\pi_{\theta_{\rm old}}(a|s)\left|\frac{\pi_{\theta_{\rm old}}(a|s)-\pi_\theta(a|s)}{\pi_{\theta_{\rm old}}(a|s)}\right| \\
    &= \frac{1}{2}\sum_{a\in\mathcal A}\pi_{\theta_{\rm old}}(a|s)\,\bigl|1 - r_{s,a}\bigr|
    = \frac{1}{2}\,\mathbb{E}_{a\sim\pi_{\theta_{\rm old}}(\cdot|s)}\!\bigl[\,|r_{s,a} - 1|\,\bigr].
\end{align*}
The KL ratio form~\eqref{eq:dkl_ratio} follows from $\log(\pi_{\theta_{\rm old}}/\pi_\theta) = -\log r_{s,a}$ in~\eqref{eq:dkl_def}.

\paragraph{Monte-Carlo estimation.}
Let $\mathcal{X}_s := \{a : (s,a) \in \mathcal{X}\}$ be the set of actions sampled at state~$s$, with multiplicities $n_{s,a}$ per-state sample count $T_s := \sum_{a\in\mathcal{X}_s} n_{s,a}$, total sample count $T = \sum_s T_s$ is the total sample count, and the augmented state-action space $\bar{\mathcal{X}}$ (Appendix~\ref{app:estimation}).

The Monte-Carlo estimators of the per-state divergences are:
\begin{equation}\label{eq:dtv_mc}
    \hat{D}_{\rm TV}(s) \;:=\; \frac{1}{T_s}\sum_{a\in\mathcal{X}_s} n_{s,a}\,|r_{s,a} - 1|,
\end{equation}
\begin{equation}\label{eq:dkl_mc}
    \hat{D}_{\rm KL}(s) \;:=\; -\frac{1}{T_s}\sum_{a\in\mathcal{X}_s} n_{s,a}\,\log r_{s,a},
\end{equation}
which are the empirical means of the ratio-form integrands~\eqref{eq:dtv_ratio}--\eqref{eq:dkl_ratio} over the samples at state~$s$.

\subsubsection{Average and Max Divergences}
\label{app:avg_max_div}

\paragraph{Average divergence.}
The average TV and KL divergences weight the per-state quantities by the normalized state distribution $\bar\rho^{\pi_{\theta_{\rm old}}}$ (Appendix~\ref{app:mdp}):
\begin{equation}\label{eq:dbar_def}
    \bar D_{\rm TV/KL}(\theta_{\rm old}\|\theta) \;:=\; \mathbb{E}_{s\sim\bar\rho^{\pi_{\theta_{\rm old}}}}\!\bigl[D_{\rm TV/KL}(s)\bigr].
\end{equation}

\paragraph{Estimation of the average divergence.}
The augmented state-action space $\bar{\mathcal{X}}$ (Appendix~\ref{app:estimation}) pools tokens across all trajectories, so that the empirical state frequencies in $\bar{\mathcal{X}}$ are proportional to $\bar\rho^{\pi_{\theta_{\rm old}}}$, i.e., $\frac{T_s}{T} \sim \bar\rho^{\pi_{\theta_{\rm old}}}(s)$ as $T \to \infty$. 
The average divergence is therefore estimated by weighting the per-state estimators~\eqref{eq:dtv_mc}--\eqref{eq:dkl_mc} by their sample counts, which can be expanded in two equivalent ways---over the states $\mathcal{S}[\mathcal{X}]$ with multiplicities, or directly over the augmented space $\bar{\mathcal{X}}$:
\begin{equation}\label{eq:dbar_mc}
    \hat{\bar D}_{\rm TV/KL}(\theta_{\rm old}\|\theta)
    \;=\;
    \frac{1}{T}\sum_{s\in\mathcal{S}[\mathcal{X}]} T_s\;\hat{D}_{\rm TV/KL}(s)
    \;=\;
    \frac{1}{T}\sum_{s\in\mathcal{S}[\mathcal{X}]}\sum_{a\in\mathcal{X}_s} n_{s,a}\,f(r_{s,a})
    \;=\;
    \frac{1}{T}\sum_{(s,a,\tau,t)\in\bar{\mathcal{X}}} f(r_{s,a}),
\end{equation}
where $f(r_{s,a}) = |r_{s,a}-1|$ for TV, and respectively $f(r_{s,a})=-\log r_{s,a}$ for KL, and $\mathcal{S}[\mathcal{X}]$ is the set of states appearing in $\mathcal{X}$.
The second equality substitutes the definition of $\hat{D}_{\rm TV/KL}(s)$ (Eqs.~\ref{eq:dtv_mc}--\ref{eq:dkl_mc}); the third unfolds the state-action multiplicities into the augmented index $(s,a,\tau,t)$, since each occurrence of $(s,a)$ in $\bar{\mathcal{X}}$ contributes exactly one term.

Note that the estimation here is for the undiscounted case.
We do not use discounted averages of TV or KL(or their estimation) elsewhere in this paper; for the generic form of discounted estimators, see Appendix~\ref{app:apc_obj_equiv_notation}.

\paragraph{Max TV divergence.}
The max TV divergence takes the worst-case per-state TV over the entire state space:
\begin{equation}\label{eq:dtv_max}
    D^{\max}_{\rm TV}(\theta_{\rm old}\|\theta)
    \;:=\;
    \max_{s\in\mathcal{S}} D_{\rm TV}(s).
\end{equation}
This quantity appears in the TRPO surrogate-gap bound (Eq.~\ref{eq:trpo_bound}).
Its sample-based estimator restricts the maximum to states observed in the batch:
\begin{equation}\label{eq:dtv_max_hat}
    \hat{D}^{\max}_{\rm TV}(\theta_{\rm old}\|\theta)
    \;:=\;
    \max_{s\in\mathcal{S}[\mathcal{X}]} \hat{D}_{\rm TV}(s),
\end{equation}
which is used in the sample-based TV-TRPO formulation (Appendix~\ref{app:sample_tv_trpo}).

\subsubsection{Trajectory-Normalized Divergences}
\label{app:traj_div}

\paragraph{Definition.}
The trajectory-normalized TV and KL divergences average the per-state quantities over the trajectory-level normalized state distribution $\bar\rho^{\pi_{\theta_{\rm old}}}_{\rm (tj)}$, which weights each state by its frequency within individual trajectories (Appendix~\ref{app:mdp}):
\begin{equation}\label{eq:dbar_tj_def}
    \bar D_{\rm TV/KL}^{\rm (tj)}(\theta_{\rm old}\|\theta) \;:=\; \mathbb{E}_{s\sim\bar\rho^{\pi_{\theta_{\rm old}}}_{\rm (tj)}}\!\bigl[D_{\rm TV/KL}(s)\bigr].
\end{equation}
The key distinction from $\bar D_{\rm TV/KL}$~\eqref{eq:dbar_def} is the state distribution: $\bar\rho^{\pi_{\theta_{\rm old}}}$ weights states by their overall frequency across the sample, whereas $\bar\rho^{\pi_{\theta_{\rm old}}}_{\rm (tj)}$ gives each trajectory equal weight and distributes that weight uniformly among the trajectory's states.

\paragraph{Ratio forms.}
Substituting the per-state ratio forms~\eqref{eq:dtv_ratio}--\eqref{eq:dkl_ratio} into~\eqref{eq:dbar_tj_def} and using the law of total expectation yields:
\begin{equation}\label{eq:dbar_tj_ratio}
    \bar D_{\rm TV}^{\rm (tj)}(\theta_{\rm old}\|\theta)
    = \frac{1}{2}\,\mathbb{E}_{\substack{s\sim\bar\rho^{\pi_{\theta_{\rm old}}}_{\rm (tj)},\\ a\sim\pi_{\theta_{\rm old}}(\cdot|s)}}\!\bigl[|r_{s,a}-1|\bigr],
    \qquad
    \bar D_{\rm KL}^{\rm (tj)}(\theta_{\rm old}\|\theta)
    = -\,\mathbb{E}_{\substack{s\sim\bar\rho^{\pi_{\theta_{\rm old}}}_{\rm (tj)},\\ a\sim\pi_{\theta_{\rm old}}(\cdot|s)}}\!\bigl[\log r_{s,a}\bigr].
\end{equation}

\paragraph{Estimation.}
We describe three related estimations for $\bar D^{\rm (tj)}_{\rm TV/KL}$, and $\bar D_{\rm TV/KL}$.

\smallskip\noindent\textit{(i) Single-trajectory estimation via $\hat{\mathbb{E}}^{(\tau)}$.}\;
For a fixed trajectory $\tau\sim\mathcal{P}^{\theta_{\rm old}}$, the sequence of states $\{s_t(\tau)\}_{t=0}^{T_\tau-1}$ constitutes samples from $\bar\rho^{\pi_{\theta_{\rm old}}}_{\rm (tj)}$, with actions drawn from $\pi_{\theta_{\rm old}}(\cdot|s_t(\tau))$.
We define the single-trajectory Monte-Carlo estimator $\hat{\mathbb{E}}^{(\tau)}$ as the empirical average over the state-action pairs visited by trajectory~$\tau$:
\begin{equation}\label{eq:single_traj_mc}
    \hat{\mathbb{E}}^{(\tau)}[f(s,a)] \;:=\; \frac{1}{T_\tau}\sum_{t=0}^{T_\tau-1} f\!\bigl(r_{s_t(\tau),~a_t(\tau)}\bigr),
\end{equation}
where all samples come from the single trajectory~$\tau$.
Applying this to the per-state divergences:
\begin{equation}\label{eq:dhat_traj}
    \hat{\bar D}^{(\tau)}_{\rm TV}(\theta_{\rm old}\|\theta)
    \;:=\; \hat{\mathbb{E}}^{(\tau)}\!\bigl[D_{\rm TV}(s)\bigr]
    \;=\; \frac{1}{T_\tau}\sum_{t=0}^{T_\tau-1} |r_{s_t(\tau),a_t(\tau)} - 1|,
\end{equation}
\begin{equation}\label{eq:dhat_kl_traj}
    \hat{\bar D}^{(\tau)}_{\rm KL}(\theta_{\rm old}\|\theta)
    \;:=\; \hat{\mathbb{E}}^{(\tau)}\!\bigl[D_{\rm KL}(s)\bigr]
    \;=\; -\frac{1}{T_\tau}\sum_{t=0}^{T_\tau-1} \log r_{s_t(\tau),a_t(\tau)}.
\end{equation}

\noindent
Averaging the single-trajectory estimators over trajectories recovers the population quantity:
\begin{equation}\label{eq:dbar_tj_as_expectation}
    \bar D_{\rm TV/KL}^{\rm (tj)}(\theta_{\rm old}\|\theta) \;=\; \mathbb{E}_{\tau\sim\mathcal{P}^{\theta_{\rm old}}}\!\bigl[\hat{\bar D}^{(\tau)}_{\rm TV/KL}(\theta_{\rm old}\|\theta)\bigr] = \mathbb{E}_{(s,a)\sim\bar\rho^{\pi_{\theta_{\rm old}}}_{\rm (tj)}\cdot \pi_{\theta_{\rm old}}}\!\bigl[f(r_{s,a})\bigr] .
\end{equation}

It is worth mentioning that the single-trajectory estimator $\hat{\bar D}^{(\tau)}_{\rm TV/KL}$ is not a single-sample estimator of the trajectory-normalized divergence $\bar D_{\rm TV/KL}^{\rm (tj)}$. We have mentioned the sequence of states(resp. state-action pairs) from a single trajectory $\tau$ are many samples from the trajectory-length-normalized state distribution $\bar\rho^{\pi_{\theta_{\rm old}}}_{\rm (tj)}$ (resp. $\bar\rho^{\pi_{\theta_{\rm old}}}_{\rm (tj)}\cdot \pi_{\theta_{\rm old}}$), so the single-trajectory estimator $\hat{\bar D}^{(\tau)}_{\rm TV/KL}$ is an empirical average over those samples, and thus is a more accurate many sample($T_\tau$ many samples) Monte-Carlo estimator of $\bar D_{\rm TV/KL}^{\rm (tj)}$ according to the second equality.

\smallskip\noindent\textit{(ii) Estimation over the augmented state-action space.}\;
Equivalently, $\bar D^{\rm (tj)}_{\rm TV/KL}$ can be estimated by averaging over the full augmented state-action space $\bar{\mathcal{X}}$(which contains not just one, but many trajectories), but weighting each sample by the inverse length of the trajectory it belongs to:
\begin{equation}\label{eq:dbar_tj_mc_augmented}
    \hat{\bar D}^{\rm (tj)}_{\rm TV/KL}(\theta_{\rm old}\|\theta)
    \;=\;
    \frac{1}{|\mathrm{Tj}^{\theta_{\rm old}}|}\sum_{(s,a,\tau,t)\in\bar{\mathcal{X}}} \frac{1}{T_\tau}\,f(r_{s,a}),
\end{equation}
where $f(r_{s,a}) = |r_{s,a}-1|$ for TV (resp.\ $-\log r_{s,a}$ for KL), and $|\mathrm{Tj}^{\theta_{\rm old}}|$ is the number of sampled trajectories.
The factor $1/T_\tau$ ensures that each trajectory contributes equally regardless of its length, matching the trajectory-length-normalized state distribution $\bar\rho^{\pi_{\theta_{\rm old}}}_{\rm (tj)}$.

\smallskip\noindent\textit{(iii) Recovering the average divergence from trajectory estimates.}\;
Given a set of sampled trajectories $\mathrm{Tj}^{\theta_{\rm old}}$, weighting each single-trajectory estimate by its trajectory length recovers the (non-trajectory-normalized) average divergence~\eqref{eq:dbar_def}:
\begin{equation}\label{eq:dbar_from_traj}
    \hat{\bar D}_{\rm TV/KL}(\theta_{\rm old}\|\theta)
    \;=\;
    \frac{1}{T}\sum_{\tau\in\mathrm{Tj}^{\theta_{\rm old}}} T_\tau\;\hat{\bar D}^{(\tau)}_{\rm TV/KL}(\theta_{\rm old}\|\theta)
    \;=\;
    \frac{1}{T}\sum_{(s,a,\tau,t)\in\bar{\mathcal{X}}} f(r_{s,a}),
\end{equation}
where $T = \sum_{\tau} T_\tau$ is the total sample count.
This follows because the trajectory-level estimator $\hat{\bar D}^{(\tau)}_{\rm TV/KL}$ averages uniformly over $T_\tau$ tokens within trajectory~$\tau$, so multiplying by $T_\tau/T$ reweights each token equally across all trajectories, which is exactly the empirical average over $\bar{\mathcal{X}}$ and thus estimates $\bar D_{\rm TV/KL}$ under the normalized state distribution $\bar\rho^{\pi_{\theta_{\rm old}}}$.

\subsection{On Ratio Space}
\label{app:ratio_space}

For sets $X,Y$, we denote $Y^X$ the collection of all functions from $X$ to $Y$.

\begin{definition}[Policy-ratio space]\label{def:ratio_space}
Given the new policy $\pi_{\theta}$ and old policy $\pi_{\theta_{\rm old}}$, we can sample a set of state-action pairs $\mathcal X$. The space $\mathbb R_{\geq 0}^{\mathcal X}\equiv (\mathbb R_{\geq 0})^{\mathcal X}$ assigns a non-negative real number to every $(s,a)\in \mathcal X$. We define $\mathbb R^{\mathcal X}_{>0}$ as the policy ratio space of all possible policy ratios over the sampled state-action space $\mathcal X$. We typically denote its elements (the tuples of policy ratios) as $r_\bullet(\theta|\theta_{\rm old}) \equiv r_\bullet$ (for placeholder $\bullet$ running over all $\mathcal X$). A tuple consists of components:
\[
    r_{s_t,a_t}(\theta|\theta_{\rm old})\equiv {r_{s_t,a_t}}:={\frac{\pi_{\theta}(a_t|s_t)}{\pi_{\theta_{\rm old}}(a_t|s_t)}}.
\]
\end{definition}

\begin{remark}
To incorporate more information, we sometimes use the space of state-action pairs with extra information instead of $\mathcal X$. For example, in our fiber bundle and FiberPO, we use $\bar{\mathcal X}$ and $\mathbb R^{\bar{\mathcal X}}_{>0}$. In general RGF forms, we use $\mathcal E$ and $\mathbb R^{\cal E}_{>0}$.
\end{remark}

\paragraph{Subscript conventions.}
Similarly, for elements in $Y^{\mathcal X}$, we use the bullet subscript $(-)_{\bullet}$ to denote its placeholder index, which runs over $\mathcal X$, and $(-)_{s,a}\in Y$ is the element evaluated at state-action pair $(s,a)$.
We denote $|q_{\bullet}|_p$ the $\mathrm L^p$ norm of quantity $q_\bullet\in \mathbb C^{\mathcal X}$ in the measure space $(\mathcal X,\mu_{\mathcal X})$.

\paragraph{Parity superscript convention.}
In FiberPO, the sign assignment $l_i := \operatorname{sign}(\log r_i) \in \{+1,-1\}$ partitions tokens into positive and negative channels.
When $l$ (or $l_i$, $l_{s,a}$) appears as a superscript in parentheses, it selects the corresponding sign channel: $s_\tau^{(l)}$ denotes $s_\tau^+$ when $l = +1$ and $s_\tau^-$ when $l = -1$, and similarly $C^{(l)}$, $\gamma_\tau^{(l)}$, etc.
The negated form $s_\tau^{(-l)}$ denotes the opposite channel.

\paragraph{Logit parametrization.}
We also introduce the quantity $z_\bullet\in\mathbb R^\mathcal X$ or $z_{\bullet,\bullet}\in\mathbb R^{\cal S\times A}$, called the logit parameter, the output of the second-to-last layer of the neural network before the softmax activation. That is,
\[
    \pi_\theta(a|s)=\frac{\exp(z_{s,a}(\theta))}{\sum_{a':\,(s,a')\in \mathcal X} \exp(z_{s,a'}(\theta))}.
\]
This is a natural parametrization of probability measures, with one redundant dimension given by adding a constant to all logits (which maps to the same probability measure).

%% file: rlvr_appendix_trpo_acpo.tex
\section{TRPO: Background, Proof, and Discussion}
\label{app:trpo}

\subsection{TV-Based TRPO (TV-TRPO)}
\label{app:tv_trpo}

Consider the discounted RL objective
\begin{equation}\label{eq:rl_obj_app}
    J(\theta) \;:=\; \mathbb{E}_{\tau \sim \mathcal{P}_\theta}\!\Bigl[\,\sum_{t=0}^{\infty} \gamma^t\, R_t(\tau)\Bigr],
\end{equation}
and its linear surrogate with respect to a reference policy $\pi_{\theta_{\rm old}}$:
\begin{equation}\label{eq:linear_surrogate_app}
    J^{(1)}(\theta\,|\,\theta_{\rm old})
    \;:=\;
    J(\theta_{\rm old})
    + \mathbb{E}_{\tau \sim \mathcal{P}_{\theta_{\rm old}}}\!\Bigl[\,
        \sum_{t=0}^{\infty} \gamma^t\,
        r_{s_t(\tau),a_t(\tau)}(\theta\,|\,\theta_{\rm old})\;
        A^{(\theta_{\rm old})}_{s_t(\tau),a_t(\tau)}
    \Bigr],
\end{equation}
where $r_{s,a}(\theta\,|\,\theta_{\rm old}) := \pi_\theta(a\,|\,s)\,/\,\pi_{\theta_{\rm old}}(a\,|\,s)$ is the importance-sampling ratio and $A^{(\theta_{\rm old})}_{s,a}$ is the advantage under $\pi_{\theta_{\rm old}}$.
The linear surrogate is the first-order approximation of the true objective near on-policy ($\theta \approx \theta_{\rm old}$); in particular, $\nabla_\theta J^{(1)}(\theta\,|\,\theta_{\rm old})\big|_{\theta=\theta_{\rm old}} = \nabla_\theta J(\theta)\big|_{\theta=\theta_{\rm old}}$.

Trust region policy optimization (TRPO) ~\cite{schulman2015trust} bounds the surrogate gap:
\begin{equation}\label{eq:trpo_bound_app}
    J(\theta) - J^{(1)}(\theta\,|\,\theta_{\rm old})
    \;\geq\;
    -\,\frac{4\gamma\,\|A^{(\theta_{\rm old})}_\bullet\|_\infty}{(1-\gamma)^2}\;
    D_{\rm TV}^{\max}(\theta\,\|\,\theta_{\rm old})^2,
\end{equation}
where $D_{\rm TV}^{\max}(\theta\,\|\,\theta_{\rm old}) := \max_{s} D_{\rm TV}\!\bigl(\pi_{\theta_{\rm old}}(\cdot\,|\,s)\,\|\,\pi_\theta(\cdot\,|\,s)\bigr)$.

Note that, unlike the original TRPO formulation of~\cite{schulman2015trust} which uses a KL-divergence penalty, here we work with the total-variation (TV) form throughout, since TV distance is the divergence that arises naturally from the performance-difference bound above.

\paragraph{TRPO's monotonic improvement guarantee.}
Based on this bound, TRPO constructs a \emph{minorization--maximization} (MM) algorithm: at each step, \eqref{eq:trpo_bound_app} provides a minorizer of the true objective $J(\theta)$, and the update maximizes this minorizer, i.e.\ $\theta_{\rm new}$ is obtained by maximizing the penalized surrogate
\begin{equation}\label{eq:trpo_penalized_app}
    J^{(1)}(\theta\,|\,\theta_{\rm old})
    \;-\;
    \frac{4\gamma\,\|A^{(\theta_{\rm old})}_\bullet\|_\infty}{(1-\gamma)^2}\;
    D_{\rm TV}^{\max}(\theta\,\|\,\theta_{\rm old})^2.
\end{equation}
Since the true objective is bounded below by~\eqref{eq:trpo_penalized_app} and the optimum of~\eqref{eq:trpo_penalized_app} is at least $J(\theta_{\rm old})$ (attained at $\theta = \theta_{\rm old}$), the true objective is guaranteed to be non-decreasing at every update, making the policy optimization stable.

We refer to this MM algorithm as \textbf{TV-based TRPO} or \textbf{TV-TRPO}.

\paragraph{Sample-based TV-TRPO.}
In practice, $D_{\rm TV}^{\max}$ is intractable to compute over large state spaces, so it is replaced by the sample maximum $\hat{\max}$ to estimate $D_{\rm TV}^{\max}$, this will be termed as \textbf{sample-based TV-TRPO} objective that serves as the theoretical foundation for the remainder of this paper.

\subsection{Proof of the TRPO Vanishing Theorem}
\label{app:vanishing_proof}

We first introduce the necessary definitions, then state and prove the theorem.

\begin{definition}[Trust Region]\label{def:trust_region}
For $\delta \geq 0$, define the \textbf{TV trust region} as
\[
    \mathcal{B}^{\rm TV\text{-}TR}_{\delta}(\theta_{\rm old})
    \;:=\;
    \bigl\{\,\pi \in (\mathcal{S} \to \mathbb{P}(\mathcal{A}))
    \;\big|\;
    D_{\rm TV}^{\max}(\pi_{\theta_{\rm old}}\,\|\,\pi) \leq \delta
    \,\bigr\},
\]
and the \textbf{KL trust region} as
\[
    \mathcal{B}^{\rm KL\text{-}TR}_{\delta}(\theta_{\rm old})
    \;:=\;
    \bigl\{\,\pi \in (\mathcal{S} \to \mathbb{P}(\mathcal{A}))
    \;\big|\;
    D_{\rm KL}^{\max}(\pi_{\theta_{\rm old}}\,\|\,\pi) \leq \delta
    \,\bigr\},
\]
where $\mathbb{P}(\mathcal{A})$ denotes the space of probability measures over the action space.
\end{definition}

The surrogate-gap bound~\eqref{eq:trpo_bound_app} implies that TRPO updates are confined to a trust region whose radius depends on $\gamma$:

\begin{lemma}[TV-TRPO update lies in the trust region]\label{lem:trpo_update_bounded}
Define the TRPO trust-region radius $\delta^{(\rm TRPO)} := \frac{1-\gamma}{8\gamma}$ (see also~\eqref{eq:delta_trpo}).
For any $\theta_{\rm new}$ produced by sample-based TV-TRPO, $\pi_{\theta_{\rm new}} \in \mathcal{B}^{\rm TV\text{-}TR}_{\delta^{(\rm TRPO)}}(\theta_{\rm old})$ (This is a restatement of Lemma~\ref{lem:tv_trpo_apc_obj_in_trust_region}).
\end{lemma}

\begin{theorem}[TRPO vanishing theorem; restatement of Theorem~\ref{thm:trpo_vanishing}]\label{thm:trpo_vanishing_app}
When $\gamma = 1$, both the TV-based and KL-based TRPO trust regions collapse to the reference policy:
\[
    \mathcal{B}^{\rm TV\text{-}TR}_{\delta^{(\rm TRPO)}}(\theta_{\rm old})
    \;=\;
    \mathcal{B}^{\rm KL\text{-}TR}_{\delta^{(\rm TRPO)}}(\theta_{\rm old})
    \;=\;
    \{\,\pi_{\theta_{\rm old}}\,\}
\]
\end{theorem}

\begin{proof}
By Lemma~\ref{lem:trpo_update_bounded}, taking $\gamma \to 1^-$ in $\delta^{(\rm TRPO)} = \frac{1-\gamma}{8\gamma}$ yields $\delta^{(\rm TRPO)} \to 0^+$; at $\gamma = 1$ we have $\delta^{(\rm TRPO)} = 0$.
Any policy $\pi_\theta$ inside the trust region must therefore satisfy $D_{\rm TV}^{\max}(\pi_{\theta_{\rm old}}\,\|\,\pi) = 0$ (and likewise for KL). This implies $\pi_{\theta_{\rm old}}(\cdot\,|\,s) = \pi_\theta(\cdot\,|\,s)$ for every~$s$, i.e., $\pi_\theta$ is identical to $\pi_{\theta_{\rm old}}$.
\end{proof}

\subsection{On $\gamma$-Annealing as a Workaround}
\label{app:gamma_annealing}

One might attempt to circumvent the vanishing theorem by introducing a variable discount factor $\gamma < 1$ into LLM RL, applying TRPO during training, and gradually annealing $\gamma \to 1$.
Similar discount-scheduling strategies have been explored in other RL settings~\cite{kim2022adaptive}.

However, this approach faces a fundamental practical obstacle in LLM RL.
Language model reward functions are typically sparse and discontinuous: they yield a meaningful signal only at the end of a response, and often in an all-or-nothing fashion---the response is either fully correct or receives no reward.
Some approaches incorporate a partial, continuous reward signal via a KL divergence penalty term~\cite{ouyang2022training}, but this remains the exception rather than the rule in RLVR settings.

Introducing a variable discount factor $\gamma < 1$ into such reward landscapes would exponentially suppress later tokens in the training objective $J(\theta) = \mathbb{E}[\sum_t \gamma^t R_t]$, systematically downweighting the very signals (end-of-response correctness) that RL is meant to optimize.
Moreover, the interaction between a changing $\gamma$ and the already-discontinuous reward surface can further destabilize reward estimation, as the effective reward scale shifts across training.
These difficulties make a $\gamma$-annealing strategy impractical for LLM RL, reinforcing the need for a direct approach to trust-region-style stabilization in the $\gamma = 1$ case.

\section{RGF and APC-Obj: Derivations}
\label{app:rgf_apc_obj}

\subsection{APC-Obj in RGF form}
\label{app:apc_obj_rgf}

\begin{definition}[APC-Obj policy iteration]\label{def:apc_obj}
The APC-Obj update selects $\theta_{\rm new} = \arg\max_\theta\, \hat{J}^{\rm APC\text{-}Obj}(\theta|\theta_{\rm old})$, where
\begin{equation}\label{eq:apc_obj_equiv_obj}
    \hat{J}^{\rm APC\text{-}Obj}(\theta|\theta_{\rm old})
    \;=\;
    \frac{1}{T}\sum_{(s,a,\tau,t) \in \bar{\mathcal{X}}}
    \left[
        \operatorname{clip}\!\Bigl(r_{s,a}-1,\;
        T_s\delta^{(\rm APC\text{-}Obj)}
        - \!\!\!\!\!\!
        \sum_{\substack{(s,a',\tau',t') \in \bar{\mathcal{X}} \\
          (a',\tau',t') \neq (a,\tau,t)}}
        \!\!\!\!\!\!|r_{s,a'}-1|\Bigr)
        \,\hat{A}^{\theta_{\rm old}}_{s,a}
        + \hat{A}^{\theta_{\rm old}}_{s,a}
    \right],
\end{equation}
where $\delta^{(\rm APC\text{-}Obj)} := \frac{(1-\gamma)^2}{8\gamma\,\|\hat{A}^{\theta_{\rm old}}_\bullet\|_\infty}\,\frac{M(\hat{A}^{\theta_{\rm old}}_\bullet)}{T}$
(the argument $\hat{A}^{\theta_{\rm old}}_\bullet$ makes the dependence on the advantage explicit),
the per-entry clip bound is
$B_{s,a,\tau,t} := T_s\delta^{(\rm APC\text{-}Obj)} - \sum_{\substack{(s,a',\tau',t') \in \bar{\mathcal{X}}_s,\, (a',\tau',t') \neq (a,\tau,t)}} |r_{s,a'}-1|$,
and $\operatorname{clip}(a, B) := \operatorname{clip}(a, -B^+, B^+)$ with $B^+ := \max(B, 0)$.
\end{definition}

We point out that the APC-Obj objective in Definition~\ref{def:apc_obj}
removes the $(1-\gamma)^{-1}$ normalization constant in \ref{def:apc_obj_discounted}. As the only factor (Sitting aside $\delta^{(\rm APC\text{-}Obj)}$'s dependence on $\gamma$, as we eventually will relax it to a constant hyperparameter. Even if $\delta^{\rm(APC-Obj)} = 0$, it will not make APC-Obj objective diverge) in the discounted APC-Obj objective that prevents convergence as $\gamma\to 1$ is the multiplier $\frac{1}{1-\gamma}$. This factor is a constant that does not depend on the policy ratios $r_{s,a}$, thus removing it does not change the $\arg\max$ over~$\theta$ and the APC-Obj policy update, for any value of $\gamma$.
This removal corresponds to estimating the discounted objective with respect to the \emph{normalized} discounted state distribution $\bar\rho_{(\gamma)}^{\theta_{\rm old}}$ than the unnormalized one $\rho_{(\gamma)}^{\theta_{\rm old}} = \frac{1}{1-\gamma}\bar\rho_{(\gamma)}^{\theta_{\rm old}}$.
With this convention, the APC-Obj objective converges and is well-defined at $\gamma=1$.
\footnote{We also mention that the APC-Obj objective varies smoothly with $\gamma$, as
the per-state discount weight $\frac{\Gamma(s)}{\Gamma}$ ($\Gamma(s)=\sum_{(s',a,\tau,t)\in\bar{\mathcal X},\,s'=s}\gamma^t$ and $\Gamma=\sum_s\Gamma(s)$) in APC-Obj algorithm varies smoothly with $\gamma$.
When $\gamma=1$, $\gamma^t=1$ for the summand, so $\Gamma(s)=T_s$ and $\Gamma=T$, giving
$\frac{\Gamma(s)}{\Gamma}=\frac{T_s}{T}$ for every state. Adding $\frac1{T_s}$ weight for each token given by the APC-Obj in \ref{def:apc_obj_discounted}, one recovering the normalized counting measure over all sampled tokens, the form displayed in Definition~\ref{def:apc_obj} here. Thus removing of the $(1-\gamma)^{-1}$ factor smoothly and analytically extends the APC-Obj objective to $\gamma=1$ without any discontinuity or singularity, resulting the exact form in Definition~\ref{def:apc_obj}.}

\noindent
The APC-Obj objective can equivalently be written in the RGF form of
Definition~\ref{def:rgf}: 

\begin{definition}[APC-Obj surrogate objective, in RGF form]\label{def:apc_obj_rgf}
\begin{equation}\label{eq:apc_obj_rgf}
    \hat{J}^{\rm APC\text{-}Obj}(\theta|\theta_{\rm old})
    \;=\;
    \frac{1}{T}\sum_{(s,a,\mathcal{I})\in\mathcal{E}}
    \left[
        \operatorname{clip}\!\Bigl(r_{s,a}-1,\;
        \,\Bigl(T_s\delta^{(\rm APC\text{-}Obj)}
        - \!\!\!\!\!\!
        \sum_{\substack{(s,a',\mathcal{I}')\in\mathcal{E} \\
          (a',\mathcal{I}')\ne(a,\mathcal{I}),}}
        \!\!\!\!\!\!|r_{s,a'}-1|\Bigr)\Bigr)
        \;\hat{A}^{\theta_{\rm old}}_{s,a}
        \;+\; \hat{A}^{\theta_{\rm old}}_{s,a}
    \right],
\end{equation}
where $\mathcal{I}=(\tau,t)$ carries the trajectory membership and
time step associated with state--action pair $(s,a)$ during
sampling,
$\mathcal{E}\subseteq\mathcal{X}\times
  \mathrm{Tj}^{\theta_{\rm old}}\times\mathbb{N}$
is the augmented index set in which each sampled token appears as a
separate entry,
and $\mu_{s,a,\mathcal{I}}=\frac{1}{T}$. 
\end{definition}

\noindent
Reading off the ratio gating map from~\eqref{eq:apc_obj_rgf}:

\begin{definition}[APC-Obj ratio gating map]\label{def:apc_obj_gating}
The APC-Obj ratio gating map $\mathcal{G}^{\rm APC\text{-}Obj}:\mathbb{R}_{>0}^{\mathcal{E}}\to\mathbb{R}_{>0}^{\mathcal{E}}$ is
\begin{equation}\label{eq:apc_obj_gating}
    \mathcal{G}^{\rm APC\text{-}Obj}(r_\bullet)_{s,a,\mathcal{I}}
    \;=\;
    \operatorname{clip}\!\Bigl(r_{s,a}-1,\;
    \Bigl(T_s\delta^{(\rm APC\text{-}Obj)}
    - \!\!\!\!\!\!
    \sum_{\substack{(s,a',\mathcal{I}')\in\mathcal{E} \\
      (a',\mathcal{I}')\ne(a,\mathcal{I})}}
    \!\!\!\!\!\!|r_{s,a'}-1|\Bigr)\Bigr)
    \;+\; 1.
\end{equation}
\end{definition}

\noindent
Unlike PPO and GRPO, the APC-Obj gating map depends on the
\emph{entire} ratio tuple $r_\bullet$ through the cross-action
coupling in the clip bound, which is what enforces the per-state total-variation constraint in aggregate.

\subsection{APC-Obj Budget Consumption and Clipping Regimes}
\label{app:apc_obj_zones}

The APC-Obj gating map (Definition~\ref{def:apc_obj_gating}) contains a policy-ratio-dependent budget inside the clipping term that enforces the per-state trust-region constraint.
To understand how this budget shapes the gated ratios, we analyze the per-token clipping behavior by decomposing it into three regimes according to budget consumption.

\paragraph{Budget and consumption terms.}
Consider a single per-token objective term from~\eqref{eq:apc_obj_rgf}:
\begin{equation}\label{eq:apc_obj_per_token}
    \Bigl(\operatorname{clip}\!\Bigl(r_{s,a}-1,\;
    \underbrace{T_s\delta^{(\rm APC\text{-}Obj)}}_{\text{(B): budget}}
    - \!\!\!\!\!\!
    \sum_{\substack{(s,a',\mathcal{I}')\in\mathcal{E} \\
      (a',\mathcal{I}')\ne(a,\mathcal{I})}}
    \!\!\!\!\!\!|r_{s,a'}-1|\Bigr) + 1\Bigr)\;\hat{A}^{\theta_{\rm old}}_{s,a}.
\end{equation}
We identify three components that govern the clipping behavior:
\begin{itemize}
    \item \textbf{(B)} The \emph{total budget}: $T_s\,\delta^{(\rm APC\text{-}Obj)}$.
    \item \textbf{(C1)} The \emph{cross-token consumption}: $\displaystyle\sum_{\substack{(s,a',\mathcal{I}')\in\mathcal{E},\; (a',\mathcal{I}')\ne(a,\mathcal{I})}} |r_{s,a'}-1| \;=\; T_s\,\hat{D}_{\rm TV}(s) - |r_{s,a}-1|$, the total budget already consumed by all other tokens (except $a$) at state~$s$.
    \item \textbf{(C2)} The \emph{current-token deviation}: $|r_{s,a}-1|$, the amount consumed by the token under consideration.
\end{itemize}
The sum $\text{(C1)} + \text{(C2)} = T_s\,\hat{D}_{\rm TV}(s)$ is the total per-state TV consumption.
The effective clip bound for the current token is $\text{(B)} - \text{(C1)}$, and the clipping activates when (C2) exceeds this residual budget.

\paragraph{Three clipping regimes.}
The clipping behavior at each token $(s,a,\mathcal{I})$ falls into exactly one of three regimes, determined by the relationship between the budget~(B) and the consumption terms (C1),~(C2):

\begin{enumerate}
    \item[\textbf{(P)}] \textbf{Pass-through regime} (clip inactive).
    The clip acts as the identity, so the gated ratio equals the ungated ratio.
    This occurs when
    \begin{equation}\label{eq:zone_P}
        |r_{s,a}-1| \;\leq\; T_s\,\delta^{(\rm APC\text{-}Obj)} - \text{(C1)},
    \end{equation}
    which, after substituting $\text{(C1)} + \text{(C2)} = T_s\,\hat{D}_{\rm TV}(s)$, reduces to a uniform condition across tokens:
    \[
        \hat{D}_{\rm TV}(s) \;\leq\; \delta^{(\rm APC\text{-}Obj)}.
    \]
    In this regime, the aggregate consumption (C1)~+~(C2) does not exceed the total budget~(B), and the per-token objective coincides with the linear surrogate.

    \item[\textbf{(R)}] \textbf{Rollback regime} (clip active, output nonzero).
    The clip is active and suppresses $|r_{s,a}-1|$ to a smaller value, but does not zero it out.
    This requires both
    \begin{equation}\label{eq:zone_R}
        \hat{D}_{\rm TV}(s) \;>\; \delta^{(\rm APC\text{-}Obj)}
        \qquad\text{and}\qquad
        \frac{|r_{s,a}-1|}{T_s} \;>\; \hat{D}_{\rm TV}(s) - \delta^{(\rm APC\text{-}Obj)}.
    \end{equation}
    The first condition is the same uniform budget violation as in~(P) (with reversed inequality).
    The second condition is token-specific: the cross-token consumption (C1) alone has not exhausted the entire budget, so the residual clip bound $\text{(B)} - \text{(C1)} > 0$ provides a nonzero but reduced clip output.
    In terms of the clip operator: $0 < |\operatorname{clip}(r_{s,a}-1,\, \text{(B)} - \text{(C1)})| < |r_{s,a}-1|$.

    \item[\textbf{(Z)}] \textbf{Zeroed regime} (clip output is zero).
    The clip output is exactly zero, and the token no longer contributes to the surrogate objective.
    This occurs when
    \begin{equation}\label{eq:zone_Z}
        \hat{D}_{\rm TV}(s) \;>\; \delta^{(\rm APC\text{-}Obj)}
        \qquad\text{and}\qquad
        \frac{|r_{s,a}-1|}{T_s} \;\leq\; \hat{D}_{\rm TV}(s) - \delta^{(\rm APC\text{-}Obj)}.
    \end{equation}
    Here the cross-token consumption~(C1) alone already exceeds the total budget~(B), making the residual clip bound $\text{(B)} - \text{(C1)} \leq 0$, which forces the clip output to zero and eliminates the token's contribution entirely.
\end{enumerate}

\paragraph{Connection to the retraction property.}
By Lemma~\ref{lem:retraction}, whenever the policy ratios enter either the rollback regime~(R) or the zeroed regime~(Z)---that is, whenever the per-state TV divergence exceeds the trust-region radius $\delta^{(\rm APC\text{-}Obj)}$---the APC-Obj clipping mechanism suppresses the gated ratios' per-state TV divergence strictly below $\delta^{(\rm APC\text{-}Obj)}$.
This ensures that the $\arg\max$ solution, i.e.\ the policy update, does not lie outside the trust region $\mathcal{B}^{\rm TV\text{-}TRPO}_{\delta^{(\rm APC\text{-}Obj)}}$.

\begin{remark}[From per-state regimes to aggregate regimes in FiberPO]\label{rem:apc_obj_zones_to_gagg}
The three APC-Obj clipping regimes~(P), (R), (Z) are defined at the per-state level with total budget $T_s\,\delta^{(\rm APC\text{-}Obj)}$.
After trajectory-level aggregation (replacing the per-state count $T_s$ with the per-trajectory length $T_\tau$; see Appendix~\ref{app:fiberpo_derivation3}), logarithmic approximation ($r_{s,a}-1 \to \log r_{s,a}$; see Appendix~\ref{app:fiberpo_derivation2}), and clipping decomposition into the fiber bundle gating form (Eq.~\ref{eq:fiberpo_decomposed}), these three per-state regimes lift to the three regimes of the aggregate gating function $g^{\rm agg}$ (Eq.~\ref{eq:gagg_main}): pass-through, rollback, and zeroed, respectively.
\end{remark}

\subsection{LLM-to-RL Notation Translation}
\label{app:llm_rl_notation}

To express PPO, GRPO, and GSPO in RGF form, we first establish the
correspondence between the LLM indexing convention $(g,j,i)$ and
the RL notation $(\tau,t)$ used throughout this paper.

\paragraph{Index correspondence.}
In LLM notation, $g$ indexes the \emph{query} (prompt),
$j$ indexes the \emph{response} (completion) generated for that
query, and $i$ indexes the \emph{token position} within the
response.
The RL counterparts are the initial state $s_0(\tau)$, the
trajectory $\tau$, and the time step $t$:
\[
    g \,:=\, s_0(\tau),
    \qquad
    j \,:=\, \tau,
    \qquad
    i \,:=\, t.
\]
Consequently, any per-token quantity
$q_{s_t(\tau),a_t(\tau)}$ in RL notation corresponds to
$q_{j,i}^{(g)}$ in LLM notation.

\paragraph{Aggregate ratios.}
The positive/negative sequence aggregate ratios $s^{\pm}_{\tau}$
in RL notation are denoted $s^{(g)\pm}_j$ in LLM notation.

\subsection{PPO, GRPO, and GSPO in RGF Form}
\label{app:xxpo_rgf}

We give the explicit RGF specifications $(\mathcal{E},\,\mu,\,\mathcal{G})$ for each method.
Recall the general RGF form (Definition~\ref{def:rgf}):
\[
    \hat J(\theta|\theta_{\rm old}) =\sum_{(s,a,\mathcal I)\in {\cal E}}{\mu_{s,a,\mathcal I}\;
    \mathcal G(r_\bullet)_{s,a,\mathcal I}\;\hat A_{s,a}}.
\]

\paragraph{PPO~\cite{schulman2017proximal}.}
$\mathcal{I} = \varnothing$ (no extra information), $\mathcal{E} \simeq \mathcal{X}$, $\mu_{s,a} = \mu_{\mathcal{X}}(\{(s,a)\})=\frac{n_{s,a}}T$ (normalized counting measure of state-action pair $(s,a)$, counting appearance multiplicity from all sampled trajectories. $T$ is number of sampled state-action pairs(counting multiplicity) and $n_{s,a}$ is the multiplicity for state-action pair $(s,a)$. Definition is also given in Appendix~\ref{app:estimation}), and
\[
    \mathcal{G}^{\rm PPO}(r_\bullet)_{s,a}
    \;=\;
    \operatorname{sign}(\hat{A}_{s,a})\cdot
    \min\!\bigl(
        \operatorname{clip}(r_{s,a},\,1\pm\epsilon)\,\operatorname{sign}(\hat{A}_{s,a}),\;
        r_{s,a}\,\operatorname{sign}(\hat{A}_{s,a})
    \bigr).
\]

\paragraph{GRPO~\cite{shao2024deepseekmath}.}
$\mathcal{I} = \tau$ (trajectory membership), $\mathcal{E} \subseteq \mathcal{X} \times \mathrm{Tj}^{\theta_{\rm old}}$, $\mu_{s,a,\tau} = \frac{1}{|\mathrm{Tj}^{\theta_{\rm old}}|}\cdot\frac{1}{T_\tau}$ (where $T_\tau$ is the total number of steps within trajectory $\tau$), and $\mathcal{G}^{\rm GRPO} = \mathcal{G}^{\rm PPO}$ (the same token-wise clipping, but with trajectory-normalized weights).

\paragraph{GSPO~\cite{zheng2025group}.}
$\mathcal{I} = \tau$ (trajectory membership), $\mathcal{E} \subseteq \mathcal{X} \times \mathrm{Tj}^{\theta_{\rm old}}$, $\mu_{s,a,\tau} = \frac{1}{|\mathrm{Tj}^{\theta_{\rm old}}|}\cdot\frac{1}{T_\tau}$, and
\[
    \mathcal{G}^{\rm GSPO}(r_\bullet)_{s,a,\tau}
    \;=\;
    l_{\tau} \cdot
    \min\!\bigl(
        \operatorname{clip}(s_\tau,\,1\pm\epsilon)\,l_\tau,\;
        s_\tau\,l_\tau
    \bigr),
\]
where $T_\tau$ is the total number of steps within trajectory $\tau$, $l_\tau :=\operatorname{sign}\bigl(\sum_{t=0}^{T_\tau-1}\hat A^{\theta_{\rm old}}_{s_t(\tau),a_t(\tau)}\bigr)$, and $s_\tau:=\bigl( \prod_{t=0}^{T_\tau-1}r_{s_t(\tau),a_t(\tau)} \bigr)^{1/T_\tau}$ is the geometric-mean aggregate ratio over trajectory~$\tau$.
Note that GSPO gates all tokens in a trajectory by the same aggregate $s_\tau$, moving the clipping from the token level to the trajectory level.

\begin{remark}
If we remove the minimize function from the clippings (i.e., let $\mathcal G(r_\bullet)_{s,a,\mathcal I}=\operatorname{clip}(r_{s,a},1\pm\epsilon)$ for PPO/GRPO, and $\mathcal G(r_\bullet)_{s,a,\mathcal I}=\operatorname{clip}(s_{\tau},1\pm\epsilon)$ for GSPO), then the gating functions contain no sign functions ($\operatorname{sign}$ and $l_\tau$).
\end{remark}

\subsection{Detailed Derivation: From APC-Obj to PPO, GRPO, and GSPO}
\label{app:xxpo_derivation}

We now trace the precise steps that transform the APC-Obj objective into each of PPO, GRPO, and GSPO, identifying the relaxation or approximation heuristic at each step.
Throughout, $r_{s,a}$ implicitly represents $r_{s,a}(\theta|\theta_{\rm old})$.
We present the derivations without the minimize function, as it is not derived from trust-region theory but from other considerations (e.g., preventing excessive reward for actions that happen to have high ratios).

Recall the APC-Obj RGF form~\eqref{eq:apc_obj_rgf}:
\begin{equation}\label{eq:apc_obj_rgf_app}
    \hat{J}^{\rm APC\text{-}Obj}(\theta|\theta_{\rm old})
    \;=\;
    \frac{1}{T}\sum_{(s,a,\mathcal{I})\in\mathcal{E}}
    \left[
        \operatorname{clip}\!\Bigl(r_{s,a}-1,\;
        \Bigl(T_s\delta^{(\rm APC\text{-}Obj)}
        - \!\!\!\!\!\!
        \sum_{\substack{(s,a',\mathcal{I}')\in\mathcal{E} \\
          (a',\mathcal{I}')\ne(a,\mathcal{I}),}}
        \!\!\!\!\!\!|r_{s,a'}-1|\Bigr)\Bigr)
        \;\hat{A}^{\theta_{\rm old}}_{s,a}
        \;+\; \hat{A}^{\theta_{\rm old}}_{s,a}
    \right].
\end{equation}

\subsubsection{APC-Obj to PPO}
\label{app:apc_obj_to_ppo}

The RGF form of PPO (without the minimize function) is:
\[
    \hat J^{\rm PPO}(\theta|\theta_{\rm old})=\sum_{(s,a)\in\mathcal X}
    \mu_{s,a}\;\operatorname{clip}(r_{s,a},\,1\pm \epsilon^{\text{(PPO)}})\;\hat A_{s,a}.
\]

\paragraph{Step 1 ($\delta$-relaxation).}
Replace the vanishing $\delta^{(\rm APC\text{-}Obj)}$ (which is zero at $\gamma=1$) with a positive tunable hyperparameter $\delta^{(\rm PPO)}$; the exact trust-region guarantee is traded for a controllable approximation whose departure is quantified by $\delta^{(\rm PPO)}$ itself.
The relaxed objective becomes:
\[
    \hat{J}(\theta|\theta_{\rm old})
    \;=\;
    \frac{1}{T}\sum_{(s,a,\tau,t)\in\bar{\mathcal{X}}}
    \Bigl[
        \operatorname{clip}\!\Bigl(
            r_{s,a}-1,\;
            \pm\,\bigl(T_s\delta^{(\rm PPO)}
            - \!\!\!\!\!\!
            \sum_{\substack{(s,a',\tau',t') \in \bar{\mathcal{X}}_s \\
              (a',\tau',t') \neq (a,\tau,t)}}
            \!\!\!\!\!\!|r_{s,a'}-1|\bigr)
        \Bigr)
        + 1
    \Bigr]\,\hat{A}^{\theta_{\rm old}}_{s,a}.
\]

\paragraph{Step 2 (decoupling).}
Decouple the clipping of individual ratios from other ratios by dropping the cross-entry sum from the clip bound and replacing it with a per-token constant $\epsilon^{(\rm PPO)}$.
This yields the PPO objective:
\[
    \hat{J}^{\rm PPO}(\theta|\theta_{\rm old})
    \;=\;
    \sum_{(s,a)\in\mathcal{X}}
    \mu_{s,a}\;\operatorname{clip}(r_{s,a},\,1\pm\epsilon^{(\rm PPO)})\;\hat{A}^{\theta_{\rm old}}_{s,a}.
\]
This step eliminates the gradient coupling between ratios of different state-action pairs, improving stability and robustness.

\subsubsection{APC-Obj to GRPO}
\label{app:apc_obj_to_grpo}

The RGF form of GRPO (without the minimize function) is:
\[
    \hat J^{\rm GRPO}(\theta|\theta_{\rm old})=\sum_{\tau\in \mathrm{Tj}^{\theta_{\rm old}}}\sum_{t=0}^{T_\tau-1}
    \frac{1}{|\mathrm{Tj}^{\theta_{\rm old}}|}\cdot\frac{1}{T_\tau}\;\operatorname{clip}(r_{s,a},\,1\pm \epsilon^{\text{(GRPO)}})\;\hat A_{s,a},
\]
where $s\equiv s_t(\tau),\,a\equiv a_t(\tau)$.
The derivation uses the same first two steps as PPO:

\paragraph{Step 1 ($\delta$-relaxation).}
$\delta^{\text{(APC-Obj)}}\mapsto\delta^{\text{(GRPO)}}$.

\paragraph{Step 2 (decoupling).}
Replace the per-entry clip bound $B_{s,a,\tau,t}:= T_s\delta^{(\rm APC\text{-}Obj)} - \sum_{\substack{(s,a',\tau',t') \in \bar{\mathcal{X}}_s,\\(a',\tau',t') \neq (a,\tau,t)}} |r_{s,a'}-1|$ with a per-token constant $\epsilon^{(\rm GRPO)}$.

\paragraph{Step 3 (sequence length weighting).}
Change the sampling weight from the normalized counting measure $\mu_{s,a}$ to the trajectory-normalized measure $\mu_{\text{(GRPO)}}(\{(s_t(\tau),a_t(\tau),\tau)\}):=\frac{1}{|\mathrm{Tj}^{\theta_{\rm old}}|}\cdot\frac{1}{T_\tau}$.
This yields the final GRPO objective:
\[
    \hat J^{\rm GRPO}(\theta|\theta_{\rm old})=\sum_{\tau\in \mathrm{Tj}^{\theta_{\rm old}}}\sum_{t=0}^{T_\tau-1}
    \frac{1}{|\mathrm{Tj}^{\theta_{\rm old}}|}\cdot\frac{1}{T_\tau}\;\operatorname{clip}(r_{s,a}-1,\,\pm \epsilon^{\text{(GRPO)}})\;\hat A_{s,a}+\hat A_{s,a}.
\]
Compared with PPO, GRPO penalizes extremely long sequences/trajectories via the $1/T_\tau$ normalization.

\subsubsection{APC-Obj to GSPO}
\label{app:apc_obj_to_gspo}

The RGF form of GSPO (without the minimize function) is:
\[
    \hat J^{\rm GSPO}(\theta|\theta_{\rm old})=\sum_{\tau\in \mathrm{Tj}^{\theta_{\rm old}}}\sum_{t=0}^{T_\tau-1}
    \frac{1}{|\mathrm{Tj}^{\theta_{\rm old}}|}\cdot\frac{1}{T_\tau}\;\operatorname{clip}(s_{\tau},\,1\pm \epsilon^{\text{(GSPO)}})\;\hat A_{s,a},
\]
where $s\equiv s_t(\tau)$, $a\equiv a_t(\tau)$, and $s_\tau := \bigl(\prod^{T_\tau -1}_{t=0} r_{s_t(\tau),a_t(\tau)} \bigr)^{1/T_\tau}=\exp\bigl( \frac{1}{T_\tau}\sum_{t=0}^{T_\tau-1}\log r_{s,a} \bigr)$ is the geometric mean of the ratios.
The derivation proceeds as follows:

\paragraph{Step 1 ($\delta$-relaxation).}
$\delta^{\text{(APC-Obj)}}\mapsto\delta^{\text{(GSPO)}}$.

\paragraph{Step 2 (trajectory-level decoupling approximation).}
Decouple the clipping of individual ratios at the trajectory level while performing trajectory-level reweighting.
The per-entry clip bound $B_{s,a,\tau,t}:= {T_s\delta^{(\rm GSPO)} - \sum_{\substack{(s,a',\tau',t') \in \bar{\mathcal{X}}_s,\\(a',\tau',t') \neq (a,\tau,t)}} |r_{s,a'}-1|}$ (which couples all entries at the same state) is replaced by a per-trajectory bound $T_\tau \delta^{(\rm GSPO)}_{seq}  - \sum_{t'\ne t}|r_{s_{t'}(\tau),a_{t'}(\tau)} - 1|$.
This yields:
\[
    \hat J(\theta|\theta_{\rm old})=\sum_{\tau\in \mathrm{Tj}^{\theta_{\rm old}}}\sum_{t=0}^{T_\tau-1}
    \frac{1}{|\mathrm{Tj}^{\theta_{\rm old}}|}\cdot\frac{1}{T_\tau}\;\Bigl[\operatorname{clip}\bigl(r_{s,a}-1,\;\pm \bigl( T_\tau \delta^{\text{(GSPO)}}_{seq}  - \!\!\!\sum_{\substack{t'=0 \\ t'\ne t}}^{T_\tau -1}\!\!\!|r_{s_{t'}(\tau),a_{t'}(\tau)} - 1|\bigr)\bigr)+1\Bigr]\;\hat A_{s,a}.
\]

It is also worth mentioning that in GSPO, the advantage estimation is chosen so that $\hat A_{s,a}$ agrees across all tokens within the same trajectory, so we can denote $\hat A_\tau$ as the common advantage for trajectory $\tau$. If the advantage estimation is not chosen in this way simply denote $\hat A_\tau$ as the average advantage across all tokens within trajectory $\tau$.

\paragraph{Step 3 (log-ratio approximation).}
Replace $r-1$ by $\log r$, which provides better algebraic properties.
Recall $\hat {\bar D}^{(\tau)}_{\rm TV}:=\frac{1}{T_\tau}\sum_{t'=0}^{T_\tau - 1}|r_{s_t(\tau),a_t(\tau)}-1|$, the Monte Carlo estimation of the average total variation distance given a sampled trajectory $\tau$, and $-\log s_\tau :=\hat {\bar D}^{(\tau)}_{\rm KL}:=\frac{1}{T_\tau}\sum_{t'=0}^{T_\tau - 1}\log r_{s_t(\tau),a_t(\tau)}$, the sequence aggregate policy ratio (recall from Claim~\ref{claim:gspo_tv} that $\log s_\tau$ is the Monte Carlo estimation for the negative average KL divergence).
The objective becomes:
\[
    \hat J(\theta|\theta_{\rm old})=\sum_{\tau\in \mathrm{Tj}^{\theta_{\rm old}}}\sum_{t=0}^{T_\tau-1}
    \frac{1}{|\mathrm{Tj}^{\theta_{\rm old}}|}\cdot\frac{1}{T_\tau}\;\Bigl[\operatorname{clip}\bigl(\log r_{s,a},\;\pm \bigl( T_\tau \delta^{\text{(GSPO)}}_{seq}  - T_\tau\hat {\bar D}^{(\tau)}_{\rm TV}+|\log r_{s,a}|\bigr)\bigr)+1\Bigr]\;\hat A_{s,a}.
\]
Inspecting the clipping term, one sees that it activates only when $T_\tau\hat {\bar D}^{(\tau)}_{\rm TV} \geq T_\tau \delta^{\text{(GSPO)}}_{seq}$, regardless of $\log r_{s,a}$.
We therefore consider two cases:
\begin{enumerate}
    \item $T_\tau\hat {\bar D}^{(\tau)}_{\rm TV} \geq T_\tau \delta^{\text{(GSPO)}}_{seq}$: clipping applies to all individual log policy ratios within the trajectory.
    \item $T_\tau\hat {\bar D}^{(\tau)}_{\rm TV} < T_\tau \delta^{\text{(GSPO)}}_{seq}$: all clipping functions reduce to the identity.
\end{enumerate}
In case~(2), the objective simplifies to:
\[
    \hat J^{\rm Case\,(2)}(\theta|\theta_{\rm old})
    =\sum_{\tau\in \mathrm{Tj}^{\theta_{\rm old}}}
    \frac{1}{|\mathrm{Tj}^{\theta_{\rm old}}|}\cdot(\log s_\tau+1)\;\hat A_{\tau}.
\]

\paragraph{Step 4 (clip-bound approximation).}
The unclipped case~(2) holds if and only if $T_\tau \delta^{\text{(GSPO)}}_{seq} \geq T_\tau\hat {\bar D}^{(\tau)}_{\rm TV}$.
By Pinsker's inequality $\frac{1}{2} D_{\rm KL}\geq D_{\rm TV}^2$, we approximately have $-\frac{1}{2}\log s_\tau \equiv \frac{1}{2} \hat {\bar D}^{(\tau)}_{\rm KL}\;\dot\geq\; (\hat {\bar D}^{(\tau)}_{\rm TV})^2$.
Thus when $-\log s_\tau \leq 2(\delta^{\rm(GSPO)}_{seq})^2$, we obtain $T_\tau\delta^{\rm(GSPO)}_{seq}\geq T_\tau \sqrt{-\frac{1}{2}\log s_\tau}\;\dot \geq\; T_\tau \hat {\bar D}^{(\tau)}_{\rm TV}$, a condition for case~(2) to be likely satisfied.

Furthermore, GSPO implicitly requires the gradient to vanish when case~(1) holds. Using this, we can combine both cases using the bound we just obtained via a simple clipping function on the aggregated ratio:
\[
    \hat J(\theta|\theta_{\rm old})=\frac{1}{|\mathrm{Tj}^{\theta_{\rm old}}|}\sum_{\tau\in \mathrm{Tj}^{\theta_{\rm old}}} \bigl(\operatorname{clip}(\log s_\tau,\,\pm 2(\delta^{\rm(GSPO)}_{seq})^2) +1\bigr)\;\hat A_{\tau}.
\]
Applying the log approximation $\log s_\tau \approx s_\tau - 1$ once more:
\[
    \hat J(\theta|\theta_{\rm old})=\frac{1}{|\mathrm{Tj}^{\theta_{\rm old}}|}\sum_{\tau\in \mathrm{Tj}^{\theta_{\rm old}}} \bigl(\operatorname{clip}(s_\tau-1,\,\pm 2(\delta^{\rm(GSPO)}_{seq})^2)+1\bigr)\;\hat A_{\tau}.
\]
Replacing the constant $2(\delta^{\rm(GSPO)}_{seq})^2$ with $\epsilon^{\rm (GSPO)}$, we arrive at the GSPO objective (without the minimize function):
\begin{align*}
    \hat J^{\rm GSPO}(\theta|\theta_{\rm old})
    &=\frac{1}{|\mathrm{Tj}^{\theta_{\rm old}}|}\sum_{\tau\in \mathrm{Tj}^{\theta_{\rm old}}} \operatorname{clip}(s_\tau,\,1\pm \epsilon^{\rm (GSPO)})\;\hat A_{\tau}\\
    &=
    \hat{J}^{\rm GSPO}(\theta|\theta_{\rm old})
    \;=\;
    \underset{\quad(s\equiv s_t(\tau),\,a\equiv a_t(\tau))}
    {\sum_{\tau\in\mathrm{Tj}^{\theta_{\rm old}}}
    \sum_{t=0}^{T_\tau-1}}
    \frac{1}{|\mathrm{Tj}^{\theta_{\rm old}}|}
    \cdot\frac{1}{T_\tau}\;
    \operatorname{clip}(s_\tau,\,1\pm\epsilon^{(\rm GSPO)})\;
    \hat{A}_{s,a}.
\end{align*}

\subsection{Proof of Claim~\ref{claim:gspo_tv}: GSPO Maintains a trajectory-level TV distance bound}
\label{app:gspo_tv_proof}

\begin{proof}
For notation, see: Appendix~\ref{app:mdp}.
Recall the GSPO objective function in RGF form (without minimization clipping):
\[
    \hat{J}^{\rm GSPO}(\theta|\theta_{\rm old})
    \;=\;
    \underset{\quad(s\equiv s_t(\tau),\,a\equiv a_t(\tau))}
    {\sum_{\tau\in\mathrm{Tj}^{\theta_{\rm old}}}
    \sum_{t=0}^{T_\tau-1}}
    \frac{1}{|\mathrm{Tj}^{\theta_{\rm old}}|}
    \cdot\frac{1}{T_\tau}\;
    \operatorname{clip}(s_\tau,\,1\pm\epsilon^{(\rm GSPO)})\;
    \hat{A}_{s,a}.
\]
We write $\epsilon$ instead of $\epsilon^{(\rm GSPO)}$ for convenience.
Given a sampled trajectory $\tau$, the log-aggregate ratio is defined to be
$\log s_\tau := \frac{1}{T_\tau}\sum_{t=0}^{T_\tau-1}
  \log r_{s_t(\tau),a_t(\tau)}$.
This is a Monte Carlo estimation for the negative trajectory-normalized average KL
divergence (at the trajectory level), see \eqref{eq:dhat_kl_traj}.
Thus, the clipping of the aggregate ratio in the objective function
maintains
    \[
        {\bar{D}}_{\rm KL}^{(\rm tj)} \overset{\eqref{eq:dbar_tj_as_expectation}}{\approx} \hat{\bar{D}}_{\rm KL}^{(\tau)}
        \;\leq\;
        -\log(1-\epsilon)
        \;\approx\;
        \epsilon
        \qquad(\text{when } \epsilon\sim 0).
    \]
Furthermore, based on
$\tfrac{1}{2}D_{\rm KL} \geq (D_{\rm TV})^2$ (Pinsker's
inequality) and
$\mathbb{E}_p[X^2] \geq \mathbb{E}_p[X]^2$ for probability
measures, we have:
\[
    \bigl(\bar{D}_{\rm TV}^{(\rm tj)}\bigr)^2
    \;=\;
    \mathbb{E}_{s\sim\bar{\rho}_{\rm (tj)}^{\theta_{\rm old}}}
    [D_{\rm TV}(s)]^2
    \;\leq\;
    \mathbb{E}_{s\sim\bar{\rho}_{\rm (tj)}^{\theta_{\rm old}}}
    [D_{\rm TV}(s)^2]
    \;\leq\;
    \mathbb{E}_{s\sim\bar{\rho}_{\rm (tj)}^{\theta_{\rm old}}}
    \!\left[\tfrac{1}{2}D_{\rm KL}(s)\right]
    \;=\;
    \tfrac{1}{2}\,\bar{D}_{\rm KL}^{(\rm tj)}
    \;\lesssim \;
    \tfrac{\epsilon}{2}.
\]
Therefore
$\bar{D}_{\rm TV}^{(\rm tj)} \lesssim \sqrt{\epsilon/2}$,
which provides a rough estimation for $\delta$ which also bounds
the trajectory-normalized average TV divergence.
We roughly conclude that $\delta \sim \sqrt{\epsilon/2}$.
\end{proof}

%% file: rlvr_appendix_acpo_equiv.tex
\section{Equivalence of APC-Obj and sample-based TV-TRPO}
\label{app:apc_obj_equiv}

This appendix establishes that the Aggregational Policy Censoring Objective (APC-Obj) algorithm and sample-based Total Variation Trust Region Policy Optimization (TV-TRPO) produce identical policy updates under standard function approximation assumptions.
The equivalence is nontrivial because APC-Obj and sample-based TV-TRPO have superficially different structures: APC-Obj applies per-token clipping with cross-action coupling, while sample-based TV-TRPO maximizes a globally penalized objective.
The proof proceeds in three stages: (i) we characterize the maximizer of sample-based TV-TRPO (Theorem~\ref{thm:tv_trpo_max}), (ii) we show that APC-Obj's clipping mechanism implicitly enforces the same trust-region constraint (Lemma~\ref{lem:retraction}), and (iii) we unite the two by showing both produce the same policy ratio deviation (Theorem~\ref{thm:apc_obj_trpo_equiv}).

\subsection{Sampling Notation and Standing Assumptions}
\label{app:apc_obj_equiv_notation}

Let $\mathcal{X}$ denote the set of all state--action pairs encountered in the sampled trajectories, and let $n_{s,a}$ denote the sample multiplicity of pair $(s,a)$ in the sampled set.
Write $\mathcal{S}[\mathcal{X}]$ for the set of all states appearing in $\mathcal{X}$, and for each state $s \in \mathcal{S}[\mathcal{X}]$, let $\mathcal{X}_s := \{a \in \mathcal{A} \mid (s,a) \in \mathcal{X}\}$ be the set of actions co-occurring with $s$.
Define $T := \sum_{(s,a) \in \mathcal{X}} n_{s,a}$ as the total sample count (with multiplicity), and $T_s := \sum_{a \in \mathcal{X}_s} n_{s,a}$ as the per-state sample count.
We also denote the \emph{augmented} state--action pair space
$\bar{\mathcal{X}} := \{(s,a,\tau,t) \mid (s,a) \in \mathcal{X},\; \tau \in \mathrm{Tj}^{\theta_{\rm old}},\; t \in \{0,\dots,T_\tau-1\}\}$,
the set of all sampled state--action pairs augmented with their trajectory membership $\tau$ and trajectory time step $t$.
Each element of $\bar{\mathcal{X}}$ is a distinct sampled token, so $|\bar{\mathcal{X}}| = T$.
The multiplicity $n_{s,a}$ counts how many augmented entries in $\bar{\mathcal{X}}$ share the pair $(s,a)$: $n_{s,a} = |\{(\tau,t) \mid (s,a,\tau,t) \in \bar{\mathcal{X}}\}|$.

The policy importance-sampling ratio is $r_{s,a}(\theta|\theta_{\rm old}) := \pi_\theta(a|s)\,/\,\pi_{\theta_{\rm old}}(a|s)$, and we write the ratio deviation as $\Delta_{s,a} := r_{s,a} - 1$.

We adopt two standing identities.
Both hold \emph{exactly} for any policy $\pi_\theta$; their sample-based counterparts are approximations that we assume hold exactly throughout this section.
\begin{enumerate}
    \item \textbf{Probability constraint.}
    Since $\sum_{a \in \mathcal{A}} \pi_\theta(a|s) = 1$ for \emph{any} policy $\pi_\theta$ (not only $\pi_{\theta_{\rm old}}$), the identity
    \[
        \mathbb{E}_{a \sim \pi_{\theta_{\rm old}}(\cdot|s)}[r_{s,a}]
        \;=\; \sum_{a \in \mathcal{A}} \pi_{\theta_{\rm old}}(a|s)\,\frac{\pi_\theta(a|s)}{\pi_{\theta_{\rm old}}(a|s)}
        \;=\; \sum_{a \in \mathcal{A}} \pi_\theta(a|s)
        \;=\; 1
    \]
    holds for every $\theta$; the $\pi_{\theta_{\rm old}}$ factors cancel algebraically, so the identity does not require $\pi_\theta = \pi_{\theta_{\rm old}}$.
    The sample-based counterpart replaces the expectation with an average over the observed action set $\mathcal{X}_s$:
    \begin{equation}\label{eq:ratio_sum_identity}
        \frac{1}{T_s}\sum_{a \in \mathcal{X}_s} n_{s,a}\, r_{s,a} \;=\; 1
        \qquad \Longleftrightarrow \qquad
        \sum_{a \in \mathcal{X}_s} n_{s,a}\, \Delta_{s,a} \;=\; 0,
        \quad \forall\, s \in \mathcal{S}[\mathcal{X}].
    \end{equation}
    This approximation incurs error from two sources: (i)~actions in $\mathcal{A} \setminus \mathcal{X}_s$ where $\pi_\theta$ may place mass but no samples exist, and (ii)~finite-sample noise in the empirical frequencies $n_{s,a}/T_s$. We assume both errors are negligible for the purposes of this analysis.

    \item \textbf{Advantage centering.}
    By definition of the advantage, $\mathbb{E}_{a \sim \pi_{\theta_{\rm old}}(\cdot|s)}[A^{\theta_{\rm old}}_{s,a}] = 0$, so under sampling:
    \begin{equation}\label{eq:adv_centering}
        \frac{1}{T_s}\sum_{a \in \mathcal{X}_s} n_{s,a}\, \hat{A}^{\theta_{\rm old}}_{s,a} \;=\; 0,
        \quad \forall\, s \in \mathcal{S}[\mathcal{X}].
    \end{equation}
\end{enumerate}
We assume both sample-based identities hold exactly throughout
this section.

\paragraph{Discounted estimators.}
In the discounted case, there are a few consistent estimators for discounted quantities.
For any $X:\mathcal S\times\mathcal A\to\mathbb R$,
\[
    \mathbb E_{\tau\sim \mathcal P_\theta}\Bigl[\sum_{t=0}^{T_\tau-1}\gamma ^t X_{s_t(\tau),a_t(\tau)}\Bigr]
    = \mathbb E_{(s,a)\sim \rho_{(\gamma)}^{\theta}\cdot \pi_\theta}[X_{s,a}]
    = \frac{1}{1-\gamma}\mathbb E_{(s,a)\sim \bar\rho_{(\gamma)}^{\theta}\cdot \pi_\theta}[X_{s,a}]
    = \frac{1}{1-\gamma}\mathbb E_{s\sim \bar\rho_{(\gamma)}^{\theta}}[\mathbb E_{a\sim \pi_\theta(\cdot|s)}[X_{s,a}]]
\]
(the second from the definition of the normalized discounted distribution).

The \emph{discounted naive Monte Carlo estimator} is
\[
    \hat{\mathbb E}_{\tau\sim \mathcal P_\theta}\Bigl[\sum_{t=0}^{T_\tau-1}\gamma ^t X_{s_t(\tau),a_t(\tau)}\Bigr]
    =\frac{1}{|\mathrm{Tj}^\theta|}\sum_{\tau \in \mathrm{Tj}^\theta}\sum_{t=0}^{T_\tau -1}\gamma^t X_{s_t(\tau),a_t(\tau)}.
\]
An alternative estimator leverages the time independence of the policy $\pi_\theta(a|s)$:
\[
    \frac{1}{1-\gamma}\hat{\mathbb E}_{(s,a)\sim \bar\rho_{(\gamma)}^{\theta}\cdot \pi_\theta}[X_{s,a}]
    =\frac{1}{1-\gamma}\sum_{s\in\mathcal S[\mathcal X]}\frac{\Gamma(s)}{\Gamma}\left(\frac{1}{T_s}\sum_{a\in\mathcal X_s}n_{s,a}X_{s,a} \right),
\]
where $\Gamma(s):=\sum_{(s',a,\tau,t)\in \bar{\mathcal X},\;s'=s}\gamma ^t$ and $\Gamma = \sum_{s\in\mathcal S[\mathcal X]}\Gamma(s)$. It is worth noting that $\Gamma(s)/\Gamma \sim \bar\rho^{\theta}_{(\gamma)}(s)$ when number of samples $T\to\infty$.

We choose this estimator as our discounted-case sampling estimator, for the following reason: no matter at which step a state $s_t=s$ is encountered, the action $a_t\sim \pi_\theta(\cdot |s_t)$ is always a valid sample from $\pi_\theta(\cdot |s)$ and contributes equally to the estimate of $\mathbb E_{a\sim \pi_\theta(\cdot | s_t)}[X_{s_t,a}]$, because the policy is always time-independent. However, the naive discounted Monte Carlo estimator always weights this $(s_t,a_t)$ by $\gamma^t$, reducing the contribution of valid policy samples drawn at large time steps and thereby wasting their potential for better estimation.

\subsection{The Sample-Based TV-TRPO Objective}
\label{app:sample_tv_trpo}

The sample-based TV-TRPO objective is the Monte Carlo counterpart of the penalized surrogate~\eqref{eq:trpo_penalized_app}.

\begin{definition}[Sample-based TV-TRPO policy iteration]\label{def:sample_tv_trpo}
The updated parameter $\theta_{\rm next}$ at each iteration is the $\arg\max$ of the sample-based TV-TRPO objective, obtained by replacing the expectations with sample averages:
\begin{align}
    \hat{J}^{(\rm TV\text{-}TRPO)}(\theta|\theta_{\rm old})
    &\;=\;
    \hat{J}^{(1)}(\theta|\theta_{\rm old})
    \;-\;
    \frac{4\gamma\,\|\hat{A}^{\theta_{\rm old}}_\bullet\|_\infty}{(1-\gamma)^2}
    \widehat{\max_{s \sim \rho_{(\gamma)}^{\theta_{\rm old}}}} \hat{D}_{\rm TV}(s,\theta|\theta_{\rm old})^2
    \notag\\[0.5em]
    &\;=\;
    J(\theta_{\rm old})+\hat{\mathbb{E}}_{(s,a) \sim \rho_{(\gamma)}^{\theta_{\rm old}} \cdot \pi_{\theta_{\rm old}}}
    \bigl[r_{s,a}\,\hat{A}_{s,a}\bigr]
    \notag\\
    &\qquad\;-\;
    \frac{4\gamma\,\|\hat{A}^{\theta_{\rm old}}_\bullet\|_\infty}{(1-\gamma)^2}
    \left(\widehat{\max_{s \sim \rho_{(\gamma)}^{\theta_{\rm old}}}}
    \hat{\mathbb{E}}_{a \sim \pi_{\theta_{\rm old}}(\cdot|s)}\bigl[|r_{s,a}-1|\bigr]\right)^{\!2}
    \notag\\[0.5em]
    &\;=\;
    J(\theta_{\rm old})+\Edc n_{s,a}\, r_{s,a}\, \hat{A}^{\theta_{\rm old}}_{s,a}
    \notag\\
    &\qquad -\,
    \frac{4\gamma\,\|\hat{A}^{\theta_{\rm old}}_\bullet\|_\infty}{(1-\gamma)^2}
    \!\left(\max_{s \in \mathcal{S}[\mathcal{X}]} \underbrace{\frac{1}{T_s}\sum_{a \in \mathcal{X}_s} n_{s,a}\,|r_{s,a} - 1|}_{\hat{D}_{\rm TV}(s)\text{: per-state TV divergence}}\right)^{\!2}\!.
    \label{eq:sample_tv_trpo_obj}
\end{align}
When multiple maximizers exist, sample-based TV-TRPO selects the one with smallest ratio-deviation $L^2$ norm: $\min \sum_{s \in \mathcal{S}[\mathcal{X}]} \frac{1}{T_s}\sum_{a \in \mathcal{X}_s} |r_{s,a}-1|^2$.
\end{definition}

Note the second term contains the per-state estimated TV divergence $\hat{D}_{\rm TV}(s)$ (Eq.~\ref{eq:dtv_mc}), the Monte-Carlo estimator of the per-state TV distance $D_{\rm TV}(s)$ (Eq.~\ref{eq:dtv_def}); see Appendix~\ref{app:per_state_div} for the full definitions and ratio forms. 
The max-TV penalty couples all actions at every state, making the penalized objective a quadratic function of max per-state TV divergence $\hat{D}^{\max}_{\rm TV} = \max_{s \in \mathcal{S}[\mathcal{X}]} \hat{D}_{\rm TV}(s)$.

\paragraph{Reformulation in ratio deviations.}
Denoting the policy ratio deviation $\Delta_{s,a} = r_{s,a} - 1$, the max-TV divergence becomes
\[
    \hat{D}^{\max}_{\rm TV}
    \;=\;
    \max_{s \in \mathcal{S}[\mathcal{X}]} \hat{D}_{\rm TV}(s)
    \;=\;
    \max_{s \in \mathcal{S}[\mathcal{X}]} \frac{1}{T_s}\sum_{a \in \mathcal{X}_s} n_{s,a}\,|\Delta_{s,a}|.
\]
Since $r_{s,a} = 1 + \Delta_{s,a}$, the linear surrogate rewrites as
\begin{multline*}
    J(\theta_{\rm old})+\Edc n_{s,a}\, r_{s,a}\, \hat{A}^{\theta_{\rm old}}_{s,a}
    \;=\;
    \Edc n_{s,a}\, \Delta_{s,a}\, \hat{A}^{\theta_{\rm old}}_{s,a}\\
    +\;
    \underbrace{\Edc n_{s,a}\, \hat{A}^{\theta_{\rm old}}_{s,a} ~+~ J(\theta_{\rm old}) }_{\text{constant in } \Delta}.
\end{multline*}
The constant term is independent of $\Delta_\bullet$ and therefore irrelevant for maximization.
We may thus equivalently maximize the reduced objective
\begin{equation}\label{eq:j_delta}
    j(\Delta)
    \;:=\;
    \Edc n_{s,a}\, \Delta_{s,a}\, \hat{A}^{\theta_{\rm old}}_{s,a}
    \;-\;
    \frac{4\gamma\,\|\hat{A}^{\theta_{\rm old}}_\bullet\|_\infty}{(1-\gamma)^2}
    \left(\max_{s \in \mathcal{S}[\mathcal{X}]} \underbrace{\frac{1}{T_s}\sum_{a \in \mathcal{X}_s} n_{s,a}\,|\Delta_{s,a}|}_{\hat{D}_{\rm TV}(s,\Delta)\text{: per-state TV divergence}}\right)^{\!2},
\end{equation}
That is, $j(\Delta)$ differs from $\hat{J}^{(\rm TV\text{-}TRPO)}$ only by the additive constant
$J(\theta_{\rm old})+\Edc n_{s,a}\,\hat{A}^{\theta_{\rm old}}_{s,a}$,
which is independent of $r_{s,a}$ and $\Delta_{s,a}$, so $j(\Delta)$ and $\hat{J}^{(\rm TV\text{-}TRPO)}(\Delta|\theta_{\rm old})$ share the same maximizer.

\subsection{Characterizing the Sample-Based TV-TRPO Maximizer}
\label{app:tv_trpo_maximizer}

The key insight is that TV-TRPO's quadratic penalty on max TV divergence forces the optimal solution to equalize the TV divergence across all states.
This section makes this statement precise.

\paragraph{Per-state normalized-TV maximizers.}
For each state $s$, consider maximizing the following auxiliary objective over the ratio deviations:
\begin{equation}\label{eq:unit_tv_max}
    \hat{\Delta}^*_{s,\bullet}
    \;:=\;
    \arg\max_{\substack{\Delta_{s,\bullet}:\; \hat{D}_{\rm TV}(s,\Delta)=1,\\ \sum_{a \in \mathcal{X}_s} n_{s,a}\,\Delta_{s,a}=0}}
    \sum_{a \in \mathcal{X}_s} n_{s,a}\, \Delta_{s,a}\, \hat{A}^{\theta_{\rm old}}_{s,a},
\end{equation}
with two constraints: (i)~the estimated per-state TV divergence $\hat{D}_{\rm TV}(s,\Delta) = \frac{1}{T_s}\sum_{a \in \mathcal{X}_s} n_{s,a}\,|\Delta_{s,a}| = \frac{1}{T_s}\sum_{a \in \mathcal{X}_s} n_{s,a}\,|r_{s,a} - 1|$ is fixed to $1$, and (ii)~the probability constraint $\sum_{a \in \mathcal{X}_s} n_{s,a}\,\Delta_{s,a} = 0$.

Let $M(s) := \sum_{a \in \mathcal{X}_s} n_{s,a}\, \hat{\Delta}^*_{s,a}\, \hat{A}^{\theta_{\rm old}}_{s,a}$ denote the optimal objective value of this per-state problem.

\paragraph{Linearity of TV maximizer.}
Because the objective is linear in $\Delta_{s,\bullet}$, the TV divergence constraint $\hat{D}_{\rm TV}(s,\Delta) = c$ is also linear in $|\Delta_{s,\bullet}|$, and the probability constraint $\sum_{a} n_{s,a}\,\Delta_{s,a} = 0$ is linear, the maximizer at radius $c$ is simply $c\,\hat{\Delta}^*_{s,\bullet}$, with optimal value $c\,M(s)$:
\begin{equation}\label{eq:linear_scaling}
    c\,\hat{\Delta}^*_{s,\bullet}
    \;=\;
    \arg\max_{\substack{\Delta_{s,\bullet}:\; \hat{D}_{\rm TV}(s,\Delta)=c,\\ \sum_{a} n_{s,a}\Delta_{s,a}=0}}
    \sum_{a \in \mathcal{X}_s} n_{s,a}\, \Delta_{s,a}\, \hat{A}^{\theta_{\rm old}}_{s,a}
    \;=\;
    c\,M(s),
    \quad c \geq 0.
\end{equation}

The same conclusion holds when the equality constraint is relaxed to $\hat{D}_{\rm TV}(s,\Delta) \leq c$: Writing $\Delta_{s,\bullet} = c^*\hat{\Delta}_{s,\bullet}$ with $c^* \leq c$ and $\hat{D}_{\rm TV}(s,\hat{\Delta}) = 1$, one can reduce the problem to $\max_{c^* \leq c} c^* M(s)$, which is attained at $c^* = c$:
\begin{align}\label{eq:ineq_tv_max}
    &\arg\max_{\substack{\Delta_{s,\bullet}:\; \hat{D}_{\rm TV}(s,\Delta)\leq c,\\ \sum_{a \in \mathcal{X}_s} n_{s,a}\,\Delta_{s,a}=0}}
    \sum_{a \in \mathcal{X}_s} n_{s,a}\, \Delta_{s,a}\, \hat{A}^{\theta_{\rm old}}_{s,a}
    \notag\\
    &\quad\;=\;\quad
    \arg\max_{\substack{c^*\,\hat{\Delta}_{s,\bullet}:\; c^* \leq c,\;\hat{D}_{\rm TV}(s,\hat{\Delta})=1,\\ \sum_{a \in \mathcal{X}_s} n_{s,a}\,\hat{\Delta}_{s,a}=0}}
    \sum_{a \in \mathcal{X}_s} n_{s,a}\, c^*\,\hat{\Delta}_{s,a}\, \hat{A}^{\theta_{\rm old}}_{s,a}
    \notag\\
    &\quad\;=\;\quad
    \left(\arg\max_{c^* \leq c}\; c^*\,M(s)\right) \hat{\Delta}^*_{s,\bullet}
    \;=\;
    c\,\hat{\Delta}^*_{s,\bullet}.
\end{align}

\paragraph{Aggregate notation.}
Define the aggregate optimal value
$M := \frac{1}{\Gamma}\sum_{s \in \mathcal{S}[\mathcal{X}]} \frac{\Gamma(s)}{T_s}\, M(s) = \frac{1}{\Gamma}\sum_{s\in\mathcal S[\mathcal X]} \frac{\Gamma(s)}{T_s}\sum_{a\in\mathcal X_s} n_{s,a}\, \hat{\Delta}^*_{s,a}\, \hat{A}^{\theta_{\rm old}}_{s,a}$.
This quantity is the expectation of $\frac{M(s)}{T_s}$ under the probability distribution $\frac{\Gamma(s)}{\Gamma}$ over $\mathcal S[\mathcal X]\subseteq\mathcal S$. This distribution varies smoothly with $\gamma$; when $\gamma=1$, $\frac{\Gamma(s)}{\Gamma}$ becomes $\frac{T_s}{T}$, the normalized counting distribution of $\bar{\mathcal X}$ over $\mathcal S[\mathcal X]$.

\begin{lemma}[$M$ is bounded by the advantage norm]\label{lem:M_bound}
$M \leq \|\hat{A}^{\theta_{\rm old}}_\bullet\|_\infty$.
\end{lemma}

\begin{proof}
Since $1 = \hat{D}_{\rm TV}(s,\hat\Delta^*) = \frac{1}{T_s}\sum_{a\in\mathcal X_s} n_{s,a}\,|\hat\Delta^*_{s,a}|$,
\[
    \frac{M(s)}{T_s}
    = \frac{1}{T_s}\sum_{a\in\mathcal X_s} n_{s,a}\,\hat\Delta^*_{s,a}\,\hat A^{\theta_{\rm old}}_{s,a}
    \leq \frac{1}{T_s}\sum_{a\in\mathcal X_s} n_{s,a}\,|\hat\Delta^*_{s,a}|\,|\hat A^{\theta_{\rm old}}_\bullet|_\infty
    = |\hat A^{\theta_{\rm old}}_\bullet|_\infty \cdot \frac{1}{T_s}\sum_{a\in\mathcal X_s} n_{s,a}\,|\hat\Delta^*_{s,a}|
    = |\hat A^{\theta_{\rm old}}_\bullet|_\infty.
\]
Therefore, $M = \mathbb{E}_{s\sim \frac{\Gamma(\cdot)}{\Gamma}}\!\bigl[\frac{M(s)}{T_s}\bigr] \leq \mathbb{E}_{s\sim \frac{\Gamma(\cdot)}{\Gamma}}\!\bigl[|\hat A^{\theta_{\rm old}}_\bullet|_\infty\bigr] = |\hat A^{\theta_{\rm old}}_\bullet|_\infty$.
\end{proof}

\begin{theorem}[Sample-based TV-TRPO maximizer]\label{thm:tv_trpo_max}
Suppose $\Delta^*_\bullet$ maximizes $j(\Delta)$ (equivalently, $\hat{J}^{(\rm TV\text{-}TRPO)}$) subject to the probability constraint~\eqref{eq:ratio_sum_identity}.
Then
\[
    \Delta^*_{s,a} \;=\; t^*\, \hat{\Delta}^*_{s,a},
    \qquad
    t^* \;=\; \frac{1-\gamma}{8\gamma\,\|\hat{A}^{\theta_{\rm old}}_\bullet\|_\infty}\,M.
\]
\end{theorem}

\begin{proof}
Let $t := \hat{D}^{\max}_{\rm TV}(\Delta^*) = \max_{s \in \mathcal{S}[\mathcal{X}]} \hat{D}_{\rm TV}(s,\Delta^*)$.

We first establish that $\hat{D}_{\rm TV}(s,\Delta^*) = t$ for every $s \in \mathcal{S}[\mathcal{X}]$.
Note that if $M(s) = 0$ for every state $s$, then $M = 0$, $t^* = 0$, and $\max j(\Delta) = 0$, so the theorem holds trivially with $\Delta^* = 0$.
We may therefore assume $M(s) > 0$ for at least one state.

Suppose for contradiction that some state $s$ with $M(s) > 0$ has $\hat{D}_{\rm TV}(s,\Delta^*) = t(s) < t$.
Since $\Delta^*$ maximizes $j$, the per-state component $\Delta^*_{s,\bullet}$ must maximize $\sum_{a} n_{s,a}\,\Delta_{s,a}\,\hat{A}^{\theta_{\rm old}}_{s,a}$ subject to $\hat{D}_{\rm TV}(s,\Delta) = t(s)$ and the probability constraint.
By the maximizer linearity~\eqref{eq:linear_scaling}, $\Delta^*_{s,\bullet} = t(s)\,\hat{\Delta}^*_{s,\bullet}$, achieving value $t(s)\,M(s)$.

Now construct $\Delta^{**}$ by setting $\Delta^{**}_{s,\bullet} := \frac{t(s)+t}{2}\,\hat{\Delta}^*_{s,\bullet}$ while keeping $\Delta^{**}_{s',\bullet} = \Delta^*_{s',\bullet}$ for all $s' \neq s$.
The per-state advantage contribution at $s$ increases to $\frac{t(s)+t}{2}\,M(s) > t(s)\,M(s)$, while $\hat{D}^{\max}_{\rm TV}(\Delta^{**}) = t$ since $\hat{D}_{\rm TV}(s,\Delta^{**}) = \frac{t(s)+t}{2} < t$.
Therefore:
\begin{align*}
    j(\Delta^{**})
    &\;=\;
    \Edc n_{s,a}\,\Delta^{**}_{s,a}\,\hat{A}^{\theta_{\rm old}}_{s,a}
    \;-\;
    \frac{4\gamma\,\|\hat{A}^{\theta_{\rm old}}_\bullet\|_\infty}{(1-\gamma)^2}\,t^2 \\
    &\;=\;
    \frac{\Gamma(s)}{(1-\gamma)\Gamma}\,\frac{t - t(s)}{2T_s}\,M(s)
    \;+\;
    \Edc n_{s,a}\,\Delta^{*}_{s,a}\,\hat{A}^{\theta_{\rm old}}_{s,a}
    \;-\;
    \frac{4\gamma\,\|\hat{A}^{\theta_{\rm old}}_\bullet\|_\infty}{(1-\gamma)^2}\,t^2 \\
    &\;=\;
    \frac{\Gamma(s)}{(1-\gamma)\Gamma}\,\frac{t - t(s)}{2T_s}\,M(s)
    \;+\; j(\Delta^*)
    \;>\; j(\Delta^*),
\end{align*}
where the strict inequality holds since $M(s) > 0$ and $t > t(s)$.
This contradicts the optimality of $\Delta^*$.
For states with $M(s) = 0$, the per-state advantage contribution is zero(by TV-TRPO algorithm, when multiple argmax exists, the smallest L2 norm solution is chosen, thus forces $\hat D_{\rm TV}(s)=0$), thus we can simply not consider them without loss of generality.
Hence $\hat{D}_{\rm TV}(s,\Delta^*) = t$ for all $s$ (except those with $M(s) = 0$).

Since $\hat{D}_{\rm TV}(s,\Delta^*) = t$ for every $s$, and each $\Delta^*_{s,\bullet}$ must maximize its per-state advantage contribution subject to this constraint, the maximizer linearity~\eqref{eq:linear_scaling} gives $\Delta^*_{s,\bullet} = t\,\hat{\Delta}^*_{s,\bullet}$ for all $s$.
The discount-weighted total advantage contribution is then:
\[
    \frac{1}{\Gamma}\sum_{s\in\mathcal S[\mathcal X]} \frac{\Gamma(s)}{T_s}\sum_{a\in\mathcal X_s} n_{s,a}\, \Delta^*_{s,a}\, \hat{A}^{\theta_{\rm old}}_{s,a}
    \;=\;
    \frac{t}{\Gamma}\sum_{s \in \mathcal{S}[\mathcal{X}]} \frac{\Gamma(s)}{T_s}\, M(s)
    \;=\; t\,M.
\]

Substituting into $j$:
\[
    j(\Delta^*)
    \;=\;
    \frac{t\,M}{1-\gamma}
    \;-\;
    \frac{4\gamma\,\|\hat{A}^{\theta_{\rm old}}_\bullet\|_\infty}{(1-\gamma)^2}\,t^2.
\]
This is a concave quadratic in $t$ with maximum at
\[
    t^* \;=\; \frac{1-\gamma}{8\gamma\,\|\hat{A}^{\theta_{\rm old}}_\bullet\|_\infty}\,M. \qedhere
\]
\end{proof}

\subsection{The APC-Obj Objective (In discounted case)}
\label{app:apc_obj_equiv_def}

\begin{definition}[APC-Obj policy iteration, in discounted case]\label{def:apc_obj_discounted}
The APC-Obj update selects $\theta_{\rm new} = \arg\max_\theta\, \hat{J}^{\rm APC\text{-}Obj}(\theta|\theta_{\rm old})$, where
\begin{multline}\label{eq:apc_obj_discounted_obj}
    \hat{J}^{\rm APC\text{-}Obj}(\theta|\theta_{\rm old})
    \;=\;
    \Edc
    n_{s,a}\Bigl[
        \operatorname{clip}\!\Bigl(r_{s,a}-1,\;
        T_s\delta^{(\rm APC\text{-}Obj)}\\
        - \!\!\!\!\!\!
        \sum_{\substack{(s,a',\tau',t') \in \bar{\mathcal{X}} \\
          (a',\tau',t') \neq (a,\tau,t)}}
        \!\!\!\!\!\!|r_{s,a'}-1|\Bigr)
        \,\hat{A}^{\theta_{\rm old}}_{s,a}
        + \hat{A}^{\theta_{\rm old}}_{s,a}
    \Bigr],
\end{multline}
where
$\delta^{(\rm APC\text{-}Obj)} :=
\frac{1-\gamma}{8\gamma\,\|\hat{A}^{\theta_{\rm old}}_\bullet\|_\infty}\,M$,
and $M$ is also used in the sample-based TV-TRPO maximizer (Theorem~\ref{thm:tv_trpo_max}).
The per-entry clip bound is
\[
B_{s,a,\tau,t} := T_s\delta^{(\rm APC\text{-}Obj)} -
\!\!\!\!\sum_{\substack{(s,a',\tau',t') \in \bar{\mathcal{X}},\,
    (a',\tau',t') \neq (a,\tau,t)}}\!\!\!\! |r_{s,a'}-1|,
\]
and
$\operatorname{clip}(a, B) := \operatorname{clip}(a, -B^+, B^+)$
with $B^+ := \max(B, 0)$. When multiple argmax exists, we assume APC-Obj
chooses the one with smallest ratio-deviation $L^2$
norm.
\end{definition}

\subsection{Surrogated APC-Obj Objective and Objective Equivalence}
\label{app:surrogated_apc_obj}

To analyze APC-Obj, we introduce a surrogate formulation that replaces the policy parameter $\theta$ with the ratio deviation $\Delta$ directly.

\begin{definition}[Surrogated APC-Obj objective]\label{def:surrogated_apc_obj}
For ratio deviations $\Delta_\bullet$ and threshold $\delta$, define
\begin{multline}\label{eq:surrogated_apc_obj}
    j^{(\rm APC\text{-}Obj)}(\Delta, \delta)
    \;:=\;
    \Edc
    \Bigl[
        n_{s,a}\operatorname{clip}\!\Bigl(\Delta_{s,a},\;
        T_s\delta - \!\!\!\sum_{a' \in \mathcal{X}_s}\!\! (n_{s,a'} - \delta_{a,a'})\,|\Delta_{s,a'}|\Bigr)
        \,\hat{A}^{\theta_{\rm old}}_{s,a}\\
        + n_{s,a}(1 - \mathcal{C}(\Delta)_s)\,\hat{A}^{\theta_{\rm old}}_{s,a}
    \Bigr],
\end{multline}
where $\delta_{a,a'}$ is the Kronecker delta, and the correction term is
\begin{equation}\label{eq:correction_term}
    \mathcal{C}(\Delta)_s
    \;:=\;
    \frac{1}{T_s}\sum_{a \in \mathcal{X}_s}
    n_{s,a}\operatorname{clip}\!\Bigl(\Delta_{s,a},\;
    T_s\delta^{(\rm APC\text{-}Obj)} - \!\!\!\sum_{a' \in \mathcal{X}_s}\!\! (n_{s,a'} - \delta_{a,a'})\,|\Delta_{s,a'}|\Bigr).
\end{equation}
\end{definition}

The correction term $\mathcal{C}(\Delta)_s$ vanishes whenever $\hat{D}_{\rm TV}(s,\Delta) < \delta^{(\rm APC\text{-}Obj)}$: in that regime, the clipping is inactive (all clip terms are identities), and the probability constraint~\eqref{eq:ratio_sum_identity} gives $\mathcal{C}(\Delta)_s = \frac{1}{T_s}\sum_{a} n_{s,a}\,\Delta_{s,a} = 0$.

\begin{lemma}[Objective equivalence]\label{lem:apc_obj_equiv_surrogate}
Let $\Delta_\bullet(\theta) := r_\bullet(\theta|\theta_{\rm old}) - 1$.
Then
\[
    j^{(\rm APC\text{-}Obj)}(\Delta(\theta),\, \delta^{(\rm APC\text{-}Obj)})
    \;=\;
    \hat{J}^{\rm APC\text{-}Obj}(\theta|\theta_{\rm old}).
\]
\end{lemma}

\begin{proof}
Expanding the surrogated objective and collecting terms:
\begin{align*}
    j^{(\rm APC\text{-}Obj)}(\Delta(\theta), \delta^{(\rm APC\text{-}Obj)})
    &\;=\;
    \hat{J}^{\rm APC\text{-}Obj}(\theta|\theta_{\rm old})
    \;-\;
    \Edc n_{s,a}\,\mathcal{C}(\Delta(\theta))_s\,\hat{A}^{\theta_{\rm old}}_{s,a} \\
    &\;=\;
    \hat{J}^{\rm APC\text{-}Obj}(\theta|\theta_{\rm old})
    \;-\;
    \frac{1}{(1-\gamma)\Gamma}\sum_{s \in \mathcal{S}[\mathcal{X}]} \Gamma(s)\,\mathcal{C}(\Delta(\theta))_s
    \cdot \frac{1}{T_s}\sum_{a \in \mathcal{X}_s} n_{s,a}\,\hat{A}^{\theta_{\rm old}}_{s,a}.
\end{align*}
The advantage centering identity~\eqref{eq:adv_centering} ensures $\frac{1}{T_s}\sum_{a \in \mathcal{X}_s} n_{s,a}\,\hat{A}^{\theta_{\rm old}}_{s,a} = 0$ for each $s$, so the correction vanishes.
\end{proof}
\begin{remark}[Equivalence with the original form using $\bar X$]
The surrogated form is obtained from the original APC-Obj
(Definition~\ref{def:apc_obj_discounted}) by grouping entries that share the
same pair $(s,a)$ and this extra constant $\frac1{(1-\gamma)\Gamma}\sum_{s\in\cal S[X]}\frac{\Gamma(s)C(\Delta)_s}{T_s}\sum_{a \in \mathcal{X}_s} n_{s,a}\,\Delta_{s,a}$, term that evaluates to $0$.  The augmented state-action pairs $(s,a,\tau,t)$ in the space $\bar{\mathcal{X}}$
with the same $(s,a)$ has the same ratio
$r_{s,a}$ and the same clip bound. Collecting
the $n_{s,a}$ identical terms produces
$n_{s,a}\operatorname{clip}(\Delta_{s,a},\;
T_s\delta - \sum_{a'}(n_{s,a'}-\delta_{a,a'})|\Delta_{s,a'}|)$.
\end{remark}

\subsection{Retraction Property of APC-Obj Clipping}
\label{app:retraction_lemma}

The following lemma establishes that the APC-Obj clipping mechanism acts as a retraction: when a candidate solution violates the TV constraint, the clipping projects the gated ratios back strictly inside the trust region.
Combining this retraction property with the property that smaller $\hat D_{\rm TV}(s)$ always produce argmax solutions with less objective value, the next theorem \ref{thm:apc_obj_trpo_equiv} used these two properties to show APC-Obj implicitly enforces the update policy, given by the argmax, lies within TV trust-region constraint. Thus, this property can be seen as the mechanism by which APC-Obj enforces the trust-region constraint, without explicitly constraining the optimization problem.

\begin{definition}[Clipping operator and centered projection]\label{def:clip_proj}
For any vector $u_{s,a} \in \mathbb{R}^{\mathcal{X}}$ satisfying $\sum_{a \in \mathcal{X}_s} n_{s,a}\,u_{s,a} = 0$ for all $s \in \mathcal{S}[\mathcal{X}]$, define the clipping operator
\[
    \mathcal{G}(u_\bullet)_{s,a}
    \;:=\;
    \operatorname{clip}\!\bigl(u_{s,a},\;
    T_s(\delta - \hat{D}_{\rm TV}(s,u)) + |u_{s,a}|\bigr),
\]
and the centered projection
\[
    \operatorname{proj}(\mathcal{G}(u))_{s,a}
    \;:=\;
    \mathcal{G}(u)_{s,a} - \mathcal{C}(\mathcal{G}(u))_s.
\]
\end{definition}

\begin{lemma}[Retraction property]\label{lem:retraction}
For any $u_\bullet$ satisfying the probability constraint~\eqref{eq:ratio_sum_identity}:
\begin{enumerate}
    \item $\hat{D}_{\rm TV}\!\bigl(s,\, \operatorname{proj}(\mathcal{G}(u))\bigr) \leq \delta$ for all $s$.
    \item If $\hat{D}_{\rm TV}(s, u) > \delta$, then $\hat{D}_{\rm TV}\!\bigl(s,\, \operatorname{proj}(\mathcal{G}(u))\bigr) < \delta$ \textup{(}strict inequality\textup{)}.
\end{enumerate}
\end{lemma}

\begin{proof}
When $\hat{D}_{\rm TV}(s,u) \leq \delta$, the clipping is inactive, $\operatorname{proj}(\mathcal{G}(u)) = u$, and the claim holds trivially.
It remains to prove the case for $\hat{D}_{\rm TV}(s,u) > \delta$.
Since all expressions aggregate only over $\mathcal{X}_s \subseteq \mathcal{A}$, we fix a state $s$ without loss of generality.

For any vector $w_a \in \mathbb{R}^{\mathcal{X}_s}$, we have $T_s\,\hat{D}_{\rm TV}(s,w) = \sum_{a \in \mathcal{X}_s} n_{s,a}\,|w_a|$.
Write $v_a := \mathcal{G}(u)_{s,a}$ and $\bar{v} := \frac{1}{T_s}\sum_{a \in \mathcal{X}_s} n_{s,a}\,v_a$; by definition, $\bar{v} = \mathcal{C}(u)_s$.
Observe that $\mathcal{G}(-w) = -\mathcal{G}(w)$, $\mathcal{C}(-w) = -\mathcal{C}(w)$, and $\hat{D}_{\rm TV}(s,-w) = \hat{D}_{\rm TV}(s,w)$; hence proving the result for $u$ implies it holding for $-u$ as well.
Thus we may therefore proceed to prove only the $\bar{v} \geq 0$ case without loss of generality: if $\bar{v} < 0$, prove with $-u$, so that $\overline{(-u)} = \mathcal{C}(-u) = -\bar{v} > 0$, and by proving the result for $-u$, the above symmetric property implies it holding for $u$ that has $\bar{v} < 0$ as well.

Let $\mathcal{X}_s^+(w) := \{a \in \mathcal{X}_s \mid w_a > 0\}$.
By definition of $\bar{v}$, $\sum_{a \in \mathcal{X}_s} n_{s,a}(v_a - \bar{v}) = 0$.
Splitting the absolute value sum into its positive and negative parts:
\begin{align*}
    T_s\,\hat{D}_{\rm TV}(s, v - \bar{v})
    &\;=\;
    \sum_{a \in \mathcal{X}_s} n_{s,a}\,|v_a - \bar{v}| \\
    &\;=\;
    \sum_{a \in \mathcal{X}_s^+(v-\bar{v})} n_{s,a}(v_a - \bar{v})
    \;-\;
    \sum_{a \in \mathcal{X}_s^+(\bar{v}-v)} n_{s,a}(v_a - \bar{v}).
\end{align*}
Since $\bigl(\sum_{a \in \mathcal{X}_s^+(v-\bar{v})} n_{s,a}(v_a - \bar{v})\bigr) + \bigl(\sum_{a \in \mathcal{X}_s^+(\bar{v}-v)} n_{s,a}(v_a - \bar{v})\bigr) = \sum_{a \in \mathcal{X}_s} n_{s,a}(v_a - \bar{v}) = 0$, the two partial sums are equal in magnitude, giving:
\[
    T_s\,\hat{D}_{\rm TV}(s, v - \bar{v})
    \;=\;
    2\!\sum_{a \in \mathcal{X}_s^+(v - \bar{v})} n_{s,a}(v_a - \bar{v}).
\]

Since $\bar{v} \geq 0$, every $a \in \mathcal{X}_s^+(v - \bar{v})$ satisfies $v_a > 0$, which implies $u_{s,a} > 0$ (the clipping $v_a = \mathcal{G}(u)_{s,a}$ preserves the sign of $u_{s,a}$ whenever $v_a \neq 0$).
When $\hat{D}_{\rm TV}(s,u) > \delta$, the clipping reduces the absolute value of $u_{s,a}$ to $v_a$ by exactly $T_s(\hat{D}_{\rm TV}(s,u) - \delta)$ for $v_a>0$, i.e.:
\[
    u_{s,a} \;=\; v_a + T_s(\hat{D}_{\rm TV}(s,u) - \delta) \;>\; v_a.
\]

Substituting $v_a = u_{s,a} - T_s(\hat{D}_{\rm TV}(s,u) - \delta)$ for each $a \in \mathcal{X}_s^+(v - \bar{v})$:
\begin{align*}
    T_s\,\hat{D}_{\rm TV}(s, v - \bar{v})
    &\;=\;
    2\!\sum_{a \in \mathcal{X}_s^+(v-\bar{v})} n_{s,a}(v_a - \bar{v}) \\
    &\;=\;
    2\!\sum_{a \in \mathcal{X}_s^+(v-\bar{v})} n_{s,a}\bigl(u_{s,a} - T_s(\hat{D}_{\rm TV}(s,u) - \delta) - \bar{v}\bigr) \\
    &\;=\;
    2\!\sum_{a \in \mathcal{X}_s^+(v-\bar{v})} n_{s,a}\,u_{s,a}
    \;-\;
    2\!\sum_{a \in \mathcal{X}_s^+(v-\bar{v})} n_{s,a}\bigl(T_s(\hat{D}_{\rm TV}(s,u) - \delta) + \bar{v}\bigr).
\end{align*}
The chain of inequalities $u_{s,a} > v_a \geq v_a - \bar{v}$ gives $\mathcal{X}_s^+(v-\bar{v}) \subseteq \mathcal{X}_s^+(v) \subseteq \mathcal{X}_s^+(u)$.
The probability constraint yields $\sum_{a \in \mathcal{X}_s^+(u)} n_{s,a}\,u_{s,a} = \sum_{a \in \mathcal{X}_s^+(-u)} n_{s,a}\,(-u_{s,a})$ (from $\sum_a n_{s,a}\,u_{s,a} = 0$) and consequently $2\sum_{a \in \mathcal{X}_s^+(u)} n_{s,a}\,u_{s,a} = T_s\,\hat{D}_{\rm TV}(s,u)$.
Combining these:
\begin{align*}
    T_s\,\hat{D}_{\rm TV}(s, v-\bar{v})
    &\;\leq\;
    2\!\sum_{a \in \mathcal{X}_s^+(u)} n_{s,a}\,u_{s,a}
    \;-\;
    2\!\sum_{a \in \mathcal{X}_s^+(v-\bar{v})} n_{s,a}\bigl(T_s(\hat{D}_{\rm TV}(s,u) - \delta) + \bar{v}\bigr) \\
    &\;=\;
    T_s\,\hat{D}_{\rm TV}(s,u)
    \;-\;
    N\bigl(T_s(\hat{D}_{\rm TV}(s,u) - \delta) + \bar{v}\bigr) \\
    &\;=\;
    T_s N\delta - T_s(N - 1)\,\hat{D}_{\rm TV}(s,u) - N\bar{v} \\
    &\;=\;
    T_s\delta - T_s(N-1)\bigl(\hat{D}_{\rm TV}(s,u) - \delta\bigr) - N\bar{v},
\end{align*}
where $N := 2\sum_{a \in \mathcal{X}_s^+(v-\bar{v})} n_{s,a} \in \mathbb{N}$ is necessarily even (being twice a sum of positive integers).

Noting that $\operatorname{proj}(\mathcal{G}(u)) = v - \bar{v}$, two cases arise.
If $N = 0$, then $v_a \leq \bar{v}$ for all $a$, so $\hat{D}_{\rm TV}(s, \operatorname{proj}(\mathcal{G}(u))) = 0 < \delta$ (since $\sum_a n_{s,a}(v_a - \bar{v}) = 0$ forces $v_a = \bar{v}$ for all $a$).
If $N \geq 2$ (since $N$ is even and positive), then $T_s(N-1) \geq T_s \geq 1 > 0$. Combined with $\hat{D}_{\rm TV}(s,u) - \delta > 0$ and $N\bar{v} \geq 0$, this gives $T_s(N-1)(\hat{D}_{\rm TV}(s,u) - \delta) + N\bar{v} > 0$, and therefore $T_s\,\hat{D}_{\rm TV}(s, \operatorname{proj}(\mathcal{G}(u))) < T_s\delta$, i.e., $\hat{D}_{\rm TV}(s, \operatorname{proj}(\mathcal{G}(u))) < \delta$.
In either case, the strict inequality holds.
\end{proof}

\subsection{Main Theorem: Sample-Based TV-TRPO and APC-Obj Equivalence}
\label{app:main_equiv}

\begin{theorem}[Sample-based TV-TRPO APC-Obj equivalence]\label{thm:apc_obj_trpo_equiv}
Suppose $\theta$ satisfies standard function approximation assumptions,
then APC-Obj and sample-based TV-TRPO produce the same policy update:
\[
    \pi_{\theta_{\rm new}^{(\rm APC\text{-}Obj)}} \;=\; \pi_{\theta_{\rm new}^{(\rm TV\text{-}TRPO)}}.
\]
\end{theorem}

\begin{proof}
By Lemma~\ref{lem:apc_obj_equiv_surrogate}, optimizing the APC-Obj objective over $\theta$ is equivalent to optimizing the surrogated APC-Obj objective $j^{(\rm APC\text{-}Obj)}(\Delta, \delta^{(\rm APC\text{-}Obj)})$ over ratio deviations $\Delta$.
We therefore work with $j^{(\rm APC\text{-}Obj)}$ directly.

Rewriting $j^{(\rm APC\text{-}Obj)}$ in terms of Definition~\ref{def:clip_proj}:
\begin{align}
    j^{(\rm APC\text{-}Obj)}(\Delta, \delta^{(\rm APC\text{-}Obj)})
    &\;=\;
    \Edc
    n_{s,a}\Bigl[
        \operatorname{clip}\bigl(\Delta_{s,a},\;
        T_s\delta^{(\rm APC\text{-}Obj)} - \textstyle\sum_{a'} (n_{s,a'} - \delta_{a,a'})\,|\Delta_{s,a'}|\bigr)
    \notag\\
    &\qquad\qquad
        \cdot\hat{A}^{\theta_{\rm old}}_{s,a}
        + (1 - \mathcal{C}(\Delta)_s)\,\hat{A}^{\theta_{\rm old}}_{s,a}
    \Bigr]
    \notag\\[0.5em]
    &\;=\;
    \Edc
    n_{s,a}\Bigl[
        \Bigl(\operatorname{clip}\bigl(\Delta_{s,a},\;
        T_s\delta^{(\rm APC\text{-}Obj)} - \textstyle\sum_{a'} (n_{s,a'} - \delta_{a,a'})\,|\Delta_{s,a'}|\bigr)
    \notag\\
    &\qquad\qquad
        - \mathcal{C}(\Delta)_s\Bigr)
        \hat{A}^{\theta_{\rm old}}_{s,a}
        + \hat{A}^{\theta_{\rm old}}_{s,a}
    \Bigr]
    \notag\\[0.5em]
    &\;=\;
    \Edc
    n_{s,a}\bigl[\operatorname{proj}(\mathcal{G}(\Delta))_{s,a}\,\hat{A}^{\theta_{\rm old}}_{s,a}
    \notag\\
    &\qquad\qquad\qquad
    + \hat{A}^{\theta_{\rm old}}_{s,a}\bigr].
    \label{eq:apc_obj_proj_form}
\end{align}
Here the second equality collects the $(1 - \mathcal{C}(\Delta)_s)$ factor, and the third identifies the clipping-minus-correction expression with $\operatorname{proj}(\mathcal{G}(\Delta))_{s,a}$ via Definition~\ref{def:clip_proj}.

Observe that the per-action clip term admits the equivalent form
\[
    \operatorname{clip}\!\Bigl(\Delta_{s,a},\;
    T_s\delta^{(\rm APC\text{-}Obj)} - \!\!\!\sum_{a' \in \mathcal{X}_s}\!\! (n_{s,a'} - \delta_{a,a'})\,|\Delta_{s,a'}|\Bigr)
    \;=\;
    \operatorname{clip}\!\bigl(\Delta_{s,a},\;
    T_s(\delta^{(\rm APC\text{-}Obj)} - \hat{D}_{\rm TV}(s,\Delta)) + |\Delta_{s,a}|\bigr),
\]
which connects the surrogated APC-Obj clipping to the operator $\mathcal{G}$ in Definition~\ref{def:clip_proj}.

\noindent\emph{The APC-Obj maximizer satisfies the trust-region constraint.}\;
Let $\Delta^*$ denote the maximizer of $j^{(\rm APC\text{-}Obj)}(\cdot, \delta^{(\rm APC\text{-}Obj)})$.
We claim $\hat{D}_{\rm TV}(s,\Delta^*) \leq \delta^{(\rm APC\text{-}Obj)}$ for all $s \in \mathcal{S}[\mathcal{X}]$.
Suppose for contradiction that $\hat{D}_{\rm TV}(s, \Delta^*) > \delta^{(\rm APC\text{-}Obj)}$ for some $s$.
By Lemma~\ref{lem:retraction}, $\hat{D}_{\rm TV}\!\bigl(s,\, \operatorname{proj}(\mathcal{G}(\Delta^*))\bigr) < \delta^{(\rm APC\text{-}Obj)}$.
The maximizer linearity~\eqref{eq:linear_scaling} then yields:
\[
    \sum_{a \in \mathcal{X}_s} n_{s,a}\,\operatorname{proj}(\mathcal{G}(\Delta^*))_{s,a}\,\hat{A}^{\theta_{\rm old}}_{s,a}
    \;\leq\;
    \hat{D}_{\rm TV}\!\bigl(s,\, \operatorname{proj}(\mathcal{G}(\Delta^*))\bigr)\,M(s).
\]
Construct $\Delta^{**}$ by setting $\Delta^{**}_{s,a} := \delta^{(\rm APC\text{-}Obj)}\,\hat{\Delta}^*_{s,a}$ for actions at state $s$ and $\Delta^{**}_{s',a'} := \Delta^*_{s',a'}$ for all other states.
Since $\hat{D}_{\rm TV}(s, \Delta^{**}) = \delta^{(\rm APC\text{-}Obj)}$, the clipping at state $s$ is inactive for $\Delta^{**}$, so $\operatorname{proj}(\mathcal{G}(\Delta^{**}))_{s,a} = \delta^{(\rm APC\text{-}Obj)}\,\hat{\Delta}^*_{s,a}$, achieving per-state value $\delta^{(\rm APC\text{-}Obj)}\,M(s)$.
Therefore:
\begin{align*}
    j^{(\rm APC\text{-}Obj)}(\Delta^*, \delta^{(\rm APC\text{-}Obj)})
    &\;=\;
    \Edc
    n_{s,a}\bigl[\operatorname{proj}(\mathcal{G}(\Delta^*))_{s,a}\,\hat{A}^{\theta_{\rm old}}_{s,a}
    + \hat{A}^{\theta_{\rm old}}_{s,a}\bigr] \\
    &\;=\;
    \Edc
    n_{s,a}\bigl[\operatorname{proj}(\mathcal{G}(\Delta^{**}))_{s,a}\,\hat{A}^{\theta_{\rm old}}_{s,a}
    + \hat{A}^{\theta_{\rm old}}_{s,a}\bigr]
    \notag\\
    &\qquad\;-\;
    \frac{\Gamma(s)}{(1{-}\gamma)\Gamma}\,\frac{M(s)}{T_s}\bigl(\delta^{(\rm APC\text{-}Obj)} - \hat{D}_{\rm TV}(s,\, \operatorname{proj}(\mathcal{G}(\Delta^*)))\bigr) \\
    &\;=\;
    j^{(\rm APC\text{-}Obj)}(\Delta^{**}, \delta^{(\rm APC\text{-}Obj)})
    \;-\;
    \frac{\Gamma(s)}{(1{-}\gamma)\Gamma}\,\frac{M(s)}{T_s}\bigl(\delta^{(\rm APC\text{-}Obj)} - \hat{D}_{\rm TV}(s,\, \operatorname{proj}(\mathcal{G}(\Delta^*)))\bigr) \\
    &\;<\;
    j^{(\rm APC\text{-}Obj)}(\Delta^{**}, \delta^{(\rm APC\text{-}Obj)}),
\end{align*}
contradicting the optimality of $\Delta^*$.
Hence $\hat{D}_{\rm TV}(s,\Delta^*) \leq \delta^{(\rm APC\text{-}Obj)}$ for all $s$.

\medskip
\noindent\emph{Reduction to constrained linear maximization.}\;
Since $\hat{D}_{\rm TV}(s,\Delta^*) \leq \delta^{(\rm APC\text{-}Obj)}$ for all $s$, the clipping terms in~\eqref{eq:apc_obj_proj_form} reduce to the identity ($\mathcal{G}(\Delta)_{s,a} = \Delta_{s,a}$) and $\mathcal{C}(\Delta^*)_s = 0$.
The surrogated APC-Obj objective thus simplifies to:
\begin{multline*}
    j^{(\rm APC\text{-}Obj)}(\Delta, \delta^{(\rm APC\text{-}Obj)})
    \;=\;
    \Edc n_{s,a}\,\Delta_{s,a}\,\hat{A}^{\theta_{\rm old}}_{s,a}\\
    +\;
    \Edc n_{s,a}\,\hat{A}^{\theta_{\rm old}}_{s,a}.
\end{multline*}
Dropping the constant, the maximization becomes:
\[
    \Delta^*_\bullet
    \;=\;
    \arg\max_{\substack{\Delta:\; \hat{D}_{\rm TV}(s,\Delta) \leq \delta^{(\rm APC\text{-}Obj)}\;\forall s,\\
    \sum_{a \in \mathcal{X}_s} n_{s,a}\,\Delta_{s,a}=0\;\forall s}}
    \Edc n_{s,a}\,\Delta_{s,a}\,\hat{A}^{\theta_{\rm old}}_{s,a}.
\]
By the maximizer linearity~\eqref{eq:linear_scaling}, the unique maximizer is $\Delta^*_\bullet = \delta^{(\rm APC\text{-}Obj)}\,\hat{\Delta}^*_\bullet$.

\medskip
\noindent\emph{Identification with the sample-based TV-TRPO maximizer.}\;
By definition, $\delta^{(\rm APC\text{-}Obj)} = \frac{1-\gamma}{8\gamma\,\|\hat{A}^{\theta_{\rm old}}_\bullet\|_\infty}\,M = t^*$, coinciding with the TV-TRPO optimal step size from Theorem~\ref{thm:tv_trpo_max}.
Therefore the APC-Obj maximizer is $\Delta^*_\bullet = \delta^{(\rm APC\text{-}Obj)}\,\hat{\Delta}^*_\bullet = t^*\,\hat{\Delta}^*_\bullet$, which is precisely the TV-TRPO maximizer.

\medskip
\noindent\emph{Lifting to policies.}\;
Let $\theta^*$ denote the policy update that maximizes the
surrogated APC-Obj objective $j^{(\rm APC\text{-}Obj)}(\Delta(\theta),
\delta^{(\rm APC\text{-}Obj)})$. If $\Delta(\theta^*)$ did not coincide
with $\Delta^*$, then under the standard function approximation
assumptions there would exist $\theta^{**}$ with
$\Delta(\theta^{**}) = \Delta^*$, giving $j^{(\rm
  APC-Obj)}(\Delta(\theta^{**}), \delta^{(\rm APC\text{-}Obj)}) > j^{(\rm
  APC-Obj)}(\Delta(\theta^*), \delta^{(\rm APC\text{-}Obj)})$, contradicting
the optimality of $\theta^*$. Hence $\Delta(\theta^*) = \Delta^*
= \delta^{(\rm APC\text{-}Obj)}\,\hat{\Delta}^*_\bullet$.

By Lemma~\ref{lem:apc_obj_equiv_surrogate}, $\Delta(\theta^*)$ also maximizes the original APC-Obj objective.
By Theorem~\ref{thm:tv_trpo_max}, $\Delta(\theta^*) = t^*\,\hat{\Delta}^*_\bullet$ is also the TV-TRPO maximizer.
Therefore:
\[
    \pi_{\theta_{\rm new}^{(\rm APC\text{-}Obj)}}
    \;=\;
    (\delta^{(\rm APC\text{-}Obj)}\,\hat{\Delta}^*_\bullet + 1)\,\pi_{\theta_{\rm old}}
    \;=\;
    \pi_{\theta_{\rm new}^{(\rm TV\text{-}TRPO)}}. \qedhere
\]
\end{proof}

\subsection{Trust-Region Bound for APC-Obj and TV-TRPO}
\label{app:apc_obj_trust_region_bound}

Define the \textbf{TRPO trust-region radius}
\begin{equation}\label{eq:delta_trpo}
    \delta^{(\rm TRPO)} := \frac{1-\gamma}{8\gamma}.
\end{equation}

\begin{lemma}[APC-Obj trust-region radius is bounded]\label{lem:apc_obj_tr}
$\delta^{(\rm APC\text{-}Obj)} \leq \delta^{(\rm TRPO)}$.
\end{lemma}

\begin{proof}
By Lemma~\ref{lem:M_bound}, $M \leq \|\hat{A}^{\theta_{\rm old}}_\bullet\|_\infty$. Therefore,
\[
    \delta^{(\rm APC\text{-}Obj)}
    = \frac{1-\gamma}{8\gamma\,\|\hat{A}^{\theta_{\rm old}}_\bullet\|_\infty}\,M
    \leq \frac{1-\gamma}{8\gamma\,\|\hat{A}^{\theta_{\rm old}}_\bullet\|_\infty}\,\|\hat{A}^{\theta_{\rm old}}_\bullet\|_\infty
    = \frac{1-\gamma}{8\gamma}
    = \delta^{(\rm TRPO)}. \qedhere
\]
\end{proof}

\begin{lemma}[TV-TRPO and APC-Obj updates lie within the trust region]\label{lem:tv_trpo_apc_obj_in_trust_region}
The TV-TRPO and APC-Obj policy updates lie in the trust region $\mathcal{B}^{\rm TV\text{-}TR}_{\delta^{(\rm TRPO)}}(\theta_{\rm old})$.
\end{lemma}

\begin{proof}
By Theorem~\ref{thm:tv_trpo_max},
\[
    D^{\max}_{\rm TV}(\pi_{\theta_{\rm old}}\,\|\,\pi_{\theta^{(\rm TV\text{-}TRPO)}_{\rm new}})
    = t^*
    = \frac{1-\gamma}{8\gamma\,\|\hat{A}^{\theta_{\rm old}}_\bullet\|_\infty}\,M
    \leq \frac{1-\gamma}{8\gamma}
    = \delta^{(\rm TRPO)},
\]
so $\pi_{\theta^{(\rm TV\text{-}TRPO)}_{\rm new}} \in \mathcal{B}^{\rm TV\text{-}TR}_{\delta^{(\rm TRPO)}}(\theta_{\rm old})$.
By Theorem~\ref{thm:apc_obj_trpo_equiv}, $\pi_{\theta^{(\rm APC\text{-}Obj)}_{\rm new}} = \pi_{\theta^{(\rm TV\text{-}TRPO)}_{\rm new}} \in \mathcal{B}^{\rm TV\text{-}TR}_{\delta^{(\rm TRPO)}}(\theta_{\rm old})$.
\end{proof}

%% file: rlvr_appendix_fiberpo.tex
\section{Fiber Bundle Theory}
\label{app:fiber_bundle_theory}

\subsection{Fiber Bundle}
\label{app:fiber_bundle}

The fiber bundle is the algebraic object underlying the local-global modelling.
We first recall the relevant concepts before defining the specific bundle used in this work.

\begin{definition}[Fiber bundle]\label{def:fiber_bundle}
A smooth fiber bundle is a tuple $(E, B, \pi_E, F)$ where $E, B, F$ are manifolds and $\pi_E : E \to B$ is a smooth surjection satisfying the local triviality condition.
Specifically, for every $b \in B$, there exists an open neighborhood $U \subset B$ containing $b$ and a diffeomorphism $\Phi_U$ (called a local trivialization):
\[
    \Phi_U: \pi_E^{-1}(U) \xrightarrow{\cong} U \times F
\]
such that
\[
    \mathrm{proj}_1(\Phi_U(x)) = \pi_E(x) \quad \text{for all } x \in \pi_E^{-1}(U),
\]
where $\mathrm{proj}_1: U \times F \to U$ is the projection onto the first factor, defined by $(u, f) \mapsto u$.
The preimage $E_b:=\pi_E^{-1}[\{b\}] \cong F$ is called the \emph{fiber} over $b$, and $B$ is called the \emph{base}.
\end{definition}

\begin{remark}[Triviality, base space, and extensibility]\label{rem:base_choice_app}
In the setting of this paper, the base space $B$ is finite (trajectory indices times sign channels), so every open cover admits a refinement into singletons and all transition functions are trivial.
The fiber bundle is therefore globally trivializable ($E \cong B \times F$).
Despite not yet having utilized much of the interesting topological structure that the fiber bundle enables, the fiber bundle formalism already has immediate merits in policy optimization despite this triviality. Moreover, the algebraic compositionality of fibrations, rather than their topology, is what enables the Fibration Gating Hierarchy (Section~\ref{subsubsec:fgh}) and its concrete instantiation FiberPO-Domain (Section~\ref{subsubsec:domain_fiberpo}), demonstrating that the formalism provides genuine structural utility beyond the topologically trivial two-level case. The immediate merits are as follows:\\[0.5em]

(i): It always formalizes general dependency, local-global structure, and for that it provides grounds to define well-behaved algebraic operations (in our paper, it is manifested with canonical decomposition structure of densities that contains policy ratio's information into base and fiber components so that the FBG gating operator (Definition~\ref{def:fbg}) can exploit upon). It enables treatment for \emph{continuous state-action spaces} without modification, which is also one of the central merit of fiber bundle gating(FBG) and fibration gating hierarchy(FGH) formulations.\\[0.5em]

(ii): It enables the synthetic treatment of heterogeneous
architectures at LLM settings. More specifically, the choice of
base space $B$ determines what constitutes ``global information''
in the fiber bundle model, and the choice $B$ is the primary key
for extending FBG to richer, more heterogeneous settings. With $B
= \mathrm{Tj}^{\theta_{\rm old}} \times \{-1,+1\}$ as in
FiberPO-Trajectory, the base gate already provides refined
control over trajectory-level drift while preserving update
signals from well-behaved individual tokens, manifested in
training stability and policy entropy behaviors.

The same framework accommodates finer-grained structure without
modifying the gating machinery. In
Section~\ref{subsec:fibration_hierarchy}, we generalize FBG to
the Fibration Gating Hierarchy (FGH) by composing fibrations into
a chain, and derive FiberPO-Domain, a four-level instantiation
that enriches the base with domain and prompt group indices to
produce per-domain and per-group aggregates, so the base gate
regulates each domain's drift independently. Even more, the
natural fibration of permitted actions over space of states forms
a natural and possibly highly nontrivial, geometrically rich
fiber bundle, which also fits naturally into the FBG framework,
and may be further exploited for stability control and achieving
better performance. In each case, the fiber $E_b:=
\pi_E^{-1}[\{b\}]\subseteq E$ collects exactly those tokens
belonging to class $b\in B$, the Markov kernel $K$ reflects gated
base densities back to the bundle, and the atomic gate $g_{\rm
  Base}$ operates per class, $g_{\rm Fiber}$ operates per token.
The local-global information decomposition is formalized by the
concept of densities and Markov kernels (see
Appendix~\ref{app:density} and \ref{app:fbg}) and remains
unchanged for all these cases.
\end{remark}

\subsection{Density on a Manifold}
\label{app:density}

\begin{definition}[Density on a vector space]\label{def:density_vector}
A density on an $n$-dimensional vector space $V$ is a function
\[
    \mu: \underbrace{V \times \cdots \times V}_{n} \to \mathbb{R}
\]
satisfying the following condition: if $T: V \to V$ is a linear map, then
\[
    \mu(Tv_1, \cdots , Tv_n) = |\mathrm{det}\,T|\,\mu(v_1, \cdots, v_n).
\]
Let $\mathcal D(V)$ denote the set of all densities on $V$.
\end{definition}

\begin{definition}[Density on a manifold]\label{def:density_manifold}
Let $\cal M$ be a smooth manifold.
The set $\mathcal D{\cal M} = \bigsqcup_{p \in \cal M} \mathcal D(\mathrm T_p \cal M)$ is called the density bundle of $\cal M$.
A section of $\mathcal D \cal M$ is called a density on $\cal M$.

We denote the space of all densities on the manifold $\mathcal M$ as $\mathscr D^1\mathcal M$ or $\mathbf M$.
We denote the space of all global sections of the density bundle of $\cal M$ as $\Gamma(\mathcal M, \mathcal D^1\mathcal M)$.
Thus,
\[
    \Gamma(\mathcal M, \mathcal D^1\mathcal M)=\mathscr D^1\mathcal M\equiv\mathbf M.
\]
The fiber over a point $x\in\mathcal M$ is denoted $\mathcal D^1_x\mathcal M$.
It can be understood as the space of densities localized at this particular point.
\end{definition}

\begin{remark}
Any density $\mu$ defines, at least locally, a smooth signed or complex measure $\nu$ on $\cal M$, so the integral $\int_K \mu = \nu(K)$ is well defined for any compact $K \subset \cal M$, as is $\int_K f \mu = \int_K f \,\mathrm d\nu$ for any $f \in C_c(\cal M)$.
\end{remark}

\subsection{Fiber Bundle Gating: Detailed Exposition}
\label{app:fbg}

We provide a detailed exposition of the Fiber Bundle Gating (FBG) framework introduced in Definition~\ref{def:fbg}.
The FBG instance is specified by the tuple $(E,B,\pi_E,K,g_{\rm Base},g_{\rm Fiber,-};\,\mathcal F,\mathcal R)$, with $\pi_E:E\to B$ being a fiber bundle.
Recall $\mathbf E,\mathbf B$ are the spaces of densities on the bundle space and base space $E, B$, respectively.
The associated FBG gating function $G_{(\pi_E,K,g_{\rm Base},g_{\rm Fiber,-})}:\mathbf E\to\mathbf E$ is defined as:
\[
    G(\sigma):=K\!\left(\Bigl(\bigoplus_{b \in B} g_{{\rm Base},b}\Bigr)(\pi_{E*}\sigma)\right)+\left(\Bigl(\bigoplus_{i \in E} g_{{\rm Fiber},i,\pi_{E*}(\sigma)}\Bigr)\bigl(\sigma - K(\pi_{E*}\sigma)\bigr) \right).
\]
We now discuss each constituent object in detail.

\paragraph{Object 1: Fiber bundle.}
$\pi_E : E\to B$ is a fiber bundle.
The local and global information of interest are associated with the points inside the bundle space and base space, respectively.
This association is formalized by the densities on the total space and base space.

The choice of $(E, B, \pi_E)$ is a design decision in FBG: it determines the granularity at which the model operates, while the bundle projection $\pi_E$ determines the structural decomposition into classes.
Different base spaces yield qualitatively different gating behaviors from the same FBG machinery.
With $B = \mathrm{Tj}^{\theta_{\rm old}} \times \{-1,+1\}$, the base gate controls trajectory-level drift.
The remaining FBG objects ($K$, $g_{\rm Base}$, $g_{\rm Fiber}$, $\mathcal F$, $\mathcal R$) then specialize to the chosen granularity.

\paragraph{Object 2: Fibration decomposition map and recovery map.}
Recall the definition of $\bar{\mathcal X}$, the augmented state-action space, and $\mathbb R_{>0}^{\bar{\mathcal X}}$, the (augmented) policy ratio space (see Appendix~\ref{app:ratio_space}).

To incorporate the policy ratios $r_\bullet$ (the real-world information available during training) into the fiber bundle (which is so far a purely mathematical construct), we need to define a way to associate local information with points in $E$ and global information with points in $B$.
Because ``local'' and ``global information'' are intuitive but imprecise notions, we formalize them as \emph{densities on a manifold}: the fibration decomposition map $\mathcal{F}$ converts policy ratios into densities on the total space $E$, these densities intuitively represent the weight, importance or relevance the corresponding local information has in the global quantities; algebraically, they are unique, canonical forms for the many possible expression of essentially the same global aggregation quantities. The pushforward $\pi_{E*}$ aggregates these to densities on the base $B$ to produce interested global quantities, and the gating operators act on these densities.
This density-based formulation makes the interaction between local and global information precise and compositional.
In summary, we introduce:
\begin{enumerate}
    \item The density spaces $\mathbf E$ and $\mathbf B$ (see Appendix~\ref{app:density}).
    \item The fibration decomposition map $\mathcal F:\mathbb R_{>0}^{\bar{\mathcal X}}\to \mathbf E$, which converts policy ratios $r_\bullet$ to densities on the total space.
    \item The policy ratio recovery map $\mathcal R:\mathbf E\to \mathbb R_{>0}^{\bar{\mathcal X}}$, which converts densities back to policy ratios. $\mathcal R$ is a left-inverse of $\mathcal F$.
\end{enumerate}

\begin{definition}[Local information in FBG]\label{def:local_qty}
Given any point $i\in E$ in the bundle space, the local information $q_i$ associated to $i$ is then represented as $q_i=\mathcal F(r_\bullet)_{i}$, the density $\mathcal F(r_\bullet)$ localized at the point $i$.
\end{definition}

\begin{definition}[Global information in FBG]\label{def:global_qty}
Given any point $b \in B$, the global information $q_b$ associated to $b$ is then represented as $q_b = (\pi_{E*}\mathcal F(r_\bullet))_b$.
Here $\pi_{E*}\mathcal F(r_\bullet)$ is the pushforward of the density $\mathcal F(r_\bullet)\in\mathbf E$ (which consists of local information) by the bundle projection map $\pi_E$.
\end{definition}

\noindent
Finally, the recovery map $\mathcal R:\mathbf E\to \mathbb R_{>0}^{\bar{\mathcal X}}$ retrieves the gated information from the density space back to policy ratios, enabling their use in the surrogate objective.

\paragraph{Object 3: Reflecting Markov kernel.}
$K:\mathbf{B}\to \mathbf{E}$ is a  Markov kernel that reflects densities on the base space to densities on the bundle space, satisfying $\pi_{E*}\circ K=\mathrm{id}_{\mathbf{B}}$.

The reflecting condition has a direct geometric interpretation.
For a density $\sigma \in \mathbf E$ (local information), the pushforward $\pi_{E*}(\sigma) \in \mathbf B$ yields the global information.
The subtraction $\sigma - K(\pi_{E*}\sigma)$ decouples the influence of the global information $\pi_{E*}\sigma$ from $\sigma$ through the kernel $K$.
After this decoupling, the local residual should no longer "contain" global information, i.e., no longer contribute to any global information, equivalent to requiring $\pi_{E*}(\sigma - K(\pi_{E*}\sigma))=0$.
This condition is equivalent to the reflecting property $\pi_{E*}\circ K=\mathrm{id}_{\mathbf{B}}$.

\paragraph{Intuition for the FBG gating function.}
The two terms in $G(\sigma)$ admit the following interpretation:
\begin{enumerate}
    \item The term $K\!\left(\bigl(\bigoplus_{b \in B} g_{{\rm Base},b}\bigr)(\pi_{E*}\sigma)\right)$: given the density (local information) $\sigma \in \mathbf E$, the atomic gating function gates its global portion $\pi_{E*}\sigma$, and the Markov kernel $K$ brings the gated global information back to the total space.
    \item The term $\bigl(\bigoplus_{i \in E} g_{{\rm Fiber},i,\pi_{E*}(\sigma)}\bigr)\bigl(\sigma - K(\pi_{E*}\sigma)\bigr)$: the residual $\sigma - K(\pi_{E*}\sigma)$ captures within-context variation after removing the global contribution. The fiber gating functions $g_{\rm Fiber}$ gate this residual independently at each point.
\end{enumerate}

\paragraph{Object 4: Atomic gating functions.}
Recall $\mathcal D^1_x\mathcal M$ is the space of densities on manifold $\mathcal M$ localized at point $x \in \mathcal M$ (Definition~\ref{def:density_manifold}).
The atomic gating functions are:
\[
    g_{{\rm Base},b}:\mathcal D^1_bB\to\mathcal D^1_bB,\quad g_{{\rm Fiber},i,p_B}:\mathcal D^1_iE\to\mathcal D^1_i E,
\]
defined pointwise, where $b \in B$, $i \in E$, and $p_B \in \mathbf B$.

Note that the fiber gating function $g_{{\rm Fiber}, i,p_B}$ may depend on an extra parameter\footnote{In the current FiberPO formulation, this dependence is not used. However, in general it is natural to introduce the global quantity/information influence on the gating of local quantities. We will be exploring this dependency in the future.}, the base density $p_B \in \mathbf B$.
This reflects the fact that local gating may be conditioned on global context.

\paragraph{FBG gating form.}
Recalling that an FBG instance induces a RGF surrogate objective (Definition~\ref{def:fbg_rgf}):
\[
    \hat J_{(G;\mathcal F,\mathcal R)}(\theta|\theta_{\rm old}):=\sum_{\tau\in \mathrm{Tj}^{\pi_{\theta_{\rm old}}}}\left[ \sum_{t=0}^{T_\tau-1}  {\frac{1}{T_\tau}\, \mathcal R\circ G\circ\mathcal F(r_{\bullet})_{s_t,a_t}\; \hat A_{s_t,a_t} }  \right].
\]

 Using the maps $\mathcal{F}, G, \mathcal{R}$, the final policy ratio transformation is $r_{\bullet}\mapsto \mathcal R\circ G\circ \mathcal F(r_\bullet) \in \mathbb R_{>0}^{\bar{\mathcal X}}$.
For all $r_\bullet \in \mathbb R_{>0}^{\bar{\mathcal X}}$, this transformation leverages the fiber bundle gating framework to treat local and global information in a structured way. 

\paragraph{Extensibility.}
All four objects above are defined relative to the chosen base space $B$.
Extending $B$ to include additional classifying factors (domain, expert index, temporal grouping) naturally extends the FBG instance: the pushforward $\pi_{E*}$ produces finer aggregates, the Markov kernel $K$ reflects finer base densities back to the bundle, and the atomic gates $g_{\rm Base}$, $g_{\rm Fiber}$ operate per class.
Theorem~\ref{thm:fbg_first_order} guarantees first-order agreement with the true RL objective for any such extension, provided the atomic gates satisfy the identity conditions at the on-policy point.
The FiberPO-Trajectory objective derived in Section~\ref{subsec:fiberpo} is the simplest nontrivial instantiation of this framework. The Fibration Gating Hierarchy (Section~\ref{subsubsec:fgh}) generalizes this construction to arbitrary hierarchical depth, and FiberPO-Domain (Section~\ref{subsubsec:domain_fiberpo}) demonstrates the concrete four-level case incorporating domain and prompt group structure.

\paragraph{Global-to-local information exchange of FBG.}
\label{app:fbg_transmit}
A crucial structural advantage of FBG, beyond first-order agreement near on-policy, is that it provides a principled mechanism for \emph{exchanging} global gating effects with individual token ratios.
We explain why this is non-trivial and why naive alternatives might fail.

In the RGF framework (Definition~\ref{def:rgf}), the ratio gating map $\mathcal{G}: \mathbb{R}_{>0}^{\mathcal{E}} \to \mathbb{R}_{>0}^{\mathcal{E}}$ takes the full ratio tuple $r_\bullet$ as input and produces a gated ratio for each entry.
Global quantities---such as the trajectory-level total variation distance $\bar{D}_{\rm TV}^{\rm (tj)}$, or the sequence aggregate ratio $s_\tau^\pm$---are computed from many individual ratios.
While these global quantities are well defined as \emph{functions} of $r_\bullet$, they have no natural counterpart at the level of a single token ratio $r_{s,a}$.
That is, there is no canonical way to ``assign'' or ``attribute'' a global quantity to any individual entry in ratio space.

This creates a fundamental difficulty for any approach that attempts to incorporate global gating directly in ratio space.
Suppose one tries to define a gating map $\mathcal{G}$ that simultaneously enforces a global budget (e.g., a trajectory-level TV bound) and local per-token clipping, without the structured decomposition that FBG provides.
Several problems arise:

\begin{enumerate}
    \item \textbf{Loss of first-order agreement.}
    The global quantity is a nonlinear function of many ratios. Naively incorporating it into per-token clipping changes the Jacobian $\partial \mathcal{G}(r_\bullet)_i / \partial r_j$ at the on-policy point $r_\bullet = \mathbf{1}$.
    Unless the coupling is carefully designed, the surrogate will no longer agree with the true RL objective to first order near on-policy.

    \item \textbf{Uncontrolled gradient directions.}
    When global and local gating are entangled in ratio space, the gradient $\nabla_\theta \hat{J}$ at points $r_\bullet \neq \mathbf{1}$ may point in directions that are undesirable: for instance, the gradient may encourage individual token ratios to move in a direction that \emph{increases} the trajectory-level divergence $\bar{D}_{\rm TV}^{\rm (tj)}$, rather than respecting the trust region.
    This is because the global constraint, when projected onto individual token ratios without a structured decomposition, can create competing objectives between the global budget and local per-token updates.

    \item \textbf{Loss of fine-grained control.}
    Without the orthogonal decomposition that the reflecting condition $\pi_{E*} \circ K = \mathrm{id}_{\mathbf{B}}$ provides, there is no way to gate the \emph{pure} local variation---the part of each token's density that remains after the global contribution has been cleanly removed.
    When global and local information are entangled, the local gate inevitably re-gates global information that has already been gated, or conversely, the global gate inadvertently constrains genuinely local variation.
    This double-counting makes it difficult for the two gates to operate at their intended granularity, potentially compromising the precision of the trajectory-level budget and the locality of per-token control.
\end{enumerate}

On top of having the first-order agreement, FBG avoids the  problems 2 and 3 through the density-based decoupling.
The pushforward $\pi_{E*}$ extracts global information from local densities in a canonical way.
The Markov kernel $K$ reflects the gated global signal back to the total space, and the reflecting condition ensures that the residual $\sigma - K(\pi_{E*}\sigma)$ carries \emph{no} base-level information.
As a result:
\begin{itemize}
    \item The base gate $g_{\rm Base}$ controls the global budget in isolation, operating on the trajectory-level base density $\pi_{E*}\sigma$ without interference from local variation.
    \item The fiber gate $g_{\rm Fiber}$ operates on the \emph{purest} form of local variation: the residual $\sigma - K(\pi_{E*}\sigma)$, from which all global influence has been subtracted.
    This means $g_{\rm Fiber}$ gates each token's deviation from the trajectory aggregate, not its absolute value---the most refined level of local control possible.
    \item Because global and local components are orthogonal ($\pi_{E*}(\sigma - K(\pi_{E*}\sigma)) = 0$), the two gates compose without cross-contamination: 
    Their effects are independent and additive in the density space, guaranteeing that the combined gating inherits the desired properties of each component.
\end{itemize}
In summary, the structured decomposition $\sigma \mapsto K(\pi_{E*}\sigma) + (\sigma - K(\pi_{E*}\sigma))$ is not merely an organizational convenience; it enables non-interfering gating at both the global and local scales, while faithfully exchanging global gating effects with individual token ratios.

\subsection{Proof of First-Order Agreement (Theorem~\ref{thm:fbg_first_order})}
\label{app:fbg_proof}

We prove that the FBG surrogate objective agrees with the true RL objective to first order at $\theta = \theta_{\rm old}$.

\begin{proof}
\textbf{Value agreement.}
At $\theta = \theta_{\rm old}$, all importance ratios equal $\mathbf{1}$, so $\mathcal F(r_\bullet) = \mathcal F(\mathbf{1})$.
The identity conditions on the atomic gating functions state:
\[
    g_{{\rm Fiber},i,p_B}(\mathcal F(\mathbf{1})_i  -  K(\pi_{E*}\mathcal F(\mathbf{1}))_i)=\mathcal F(\mathbf{1})_i - K(\pi_{E*}\mathcal F(\mathbf{1}))_i, \qquad
    g_{{\rm Base},b}(\pi_{E*}\mathcal F(\mathbf{1}))=\pi_{E*}\mathcal F(\mathbf{1}).
\]
We verify that $G(\mathcal F(\mathbf{1})) = \mathcal F(\mathbf{1})$.
The base component gives:
\[
    K\!\left(\Bigl(\bigoplus_{b \in B} g_{{\rm Base},b}\Bigr)(\pi_{E*}\mathcal F(\mathbf{1}))\right)
    = K(\pi_{E*}\mathcal F(\mathbf{1})).
\]
The fiber component gives:
\[
    \Bigl(\bigoplus_{i \in E} g_{{\rm Fiber},i,\pi_{E*}\mathcal F(\mathbf{1})}\Bigr)\bigl(\mathcal F(\mathbf{1}) - K(\pi_{E*}\mathcal F(\mathbf{1}))\bigr)
    = \mathcal F(\mathbf{1}) - K(\pi_{E*}\mathcal F(\mathbf{1})),
\]
where the last equality uses the identity condition on $g_{\rm Fiber}$ applied to each fiber component.
Summing gives $G(\mathcal F(\mathbf{1})) = K(\pi_{E*}\mathcal F(\mathbf{1})) + \mathcal F(\mathbf{1}) - K(\pi_{E*}\mathcal F(\mathbf{1})) = \mathcal F(\mathbf{1})$.
Since $\mathcal R$ is a left-inverse of $\mathcal F$, we obtain $\mathcal R\circ G\circ \mathcal F(\mathbf{1}) = \mathbf{1}$, and the surrogate reduces to the linear surrogate value $J(\theta_{\rm old})$.

\medskip
\noindent\textbf{Gradient agreement.}
Taking the gradient with respect to $\theta$ at $\theta = \theta_{\rm old}$, we need to show that $\nabla_\theta(\mathcal R \circ G \circ \mathcal F)|_{\theta_{\rm old}} = \nabla_\theta(\mathcal R \circ \mathcal F)|_{\theta_{\rm old}}$.
By the chain rule, it suffices to show that the Jacobian of $G$ at the on-policy density $\mathcal F(\mathbf{1})$ is the identity.

The derivative conditions $g'_{{\rm Fiber},i,p_B}(\mathcal F(\mathbf{1})_i-  K(\pi_{E*}\mathcal F(\mathbf{1}))_i)=1$ and $g'_{{\rm Base},b}(\pi_{E*}\mathcal F(\mathbf{1}))=1$ ensure that each atomic gating function has unit derivative at the on-policy point.
Since $G$ is composed of the base and fiber gating applied to the orthogonal decomposition $\sigma = K(\pi_{E*}\sigma) + (\sigma - K(\pi_{E*}\sigma))$, the Jacobian of $G$ at $\mathcal F(\mathbf{1})$ is the identity on both the base and fiber components.
Therefore:
\[
    \nabla_\theta(\mathcal R \circ G \circ \mathcal F)\big|_{\theta_{\rm old}} = \nabla_\theta(\mathcal R \circ \mathcal F)\big|_{\theta_{\rm old}},
\]
which recovers the true policy gradient $\nabla_\theta J(\theta)\big|_{\theta = \theta_{\rm old}}$.
\end{proof}

\subsection{Derivation of FiberPO}
\label{app:fiberpo_derivation}

Recall the APC-Obj objective (in RGF form):
\[
    {\hat J^{{\text{APC-Obj}}}}(\theta |\theta_{\rm old}) =
     \frac{1}{T}\sum_{(s,a,\mathcal I)\in\mathcal E}   \Bigl[\operatorname{clip}\bigl(r_{s,a} - 1,\; \pm \bigl(T_s\delta^{\text{(APC-Obj)}}  - \!\!\!\!\sum_{\substack{(s,a',\mathcal I')\in\mathcal E \\ (s,a',\mathcal I')\ne (s,a,\mathcal I),}}\!\!\!\!|r_{s,a'} - 1|\bigr)\Bigr)+1\Bigr]\hat A_{s,a}.
\], 
where $\mathcal{E} := \{(s,a, \mathcal{I})\} = \bar{\mathcal{X}} := \{(s,a,\tau,t)\}$ is the space of all sampled state--action pairs augmented with trajectory membership $\tau$ and time step $t$; each element is a distinct sampled token, so $|\bar{\cal X}| = T$ and $|\bar{\cal X}_s| = T_s$ (Definition~\ref{def:apc_obj_main}).

The derivation begins from the APC-Obj objective (Section~\ref{subsubsec:apc_obj}), which enforces a per-state TV budget $T_s\delta^{\text{(APC-Obj)}}$ through cross-token interaction terms $\bigl(\sum_{\substack{(s,a',\mathcal I')\in\mathcal E \\ (s,a',\mathcal I')\ne (s,a,\mathcal I),}}\!\!\!\!|r_{s,a'} - 1|\bigr)$.
We transform it into a form in which two scales of information are made explicit: the aggregate drift of a trajectory, measured by the positive and negative aggregate ratios $s_\tau^+, s_\tau^-$ (which reflect the trajectory-level TV distance, as established in Section~\ref{subsec:rgf_apc_obj}), and the per-token ratio $r_{s,a}$ (which captures how much each individual token has moved from the reference policy).
Our goal is to use the APC-Obj objective as a reliable starting point, make reasonable relaxations, and arrive at the FBG form (Definition~\ref{def:fbg}). The resulting objective can then inherit APC-Obj's stability properties and the structural guarantees of the FBG form (first-order agreement, global-local decoupling, extensibility). Figure~\ref{fig:apc_obj_to_fbg} illustrates this derivation roadmap.

\begin{figure}[htbp]
\centering
\includegraphics[width=0.75\textwidth]{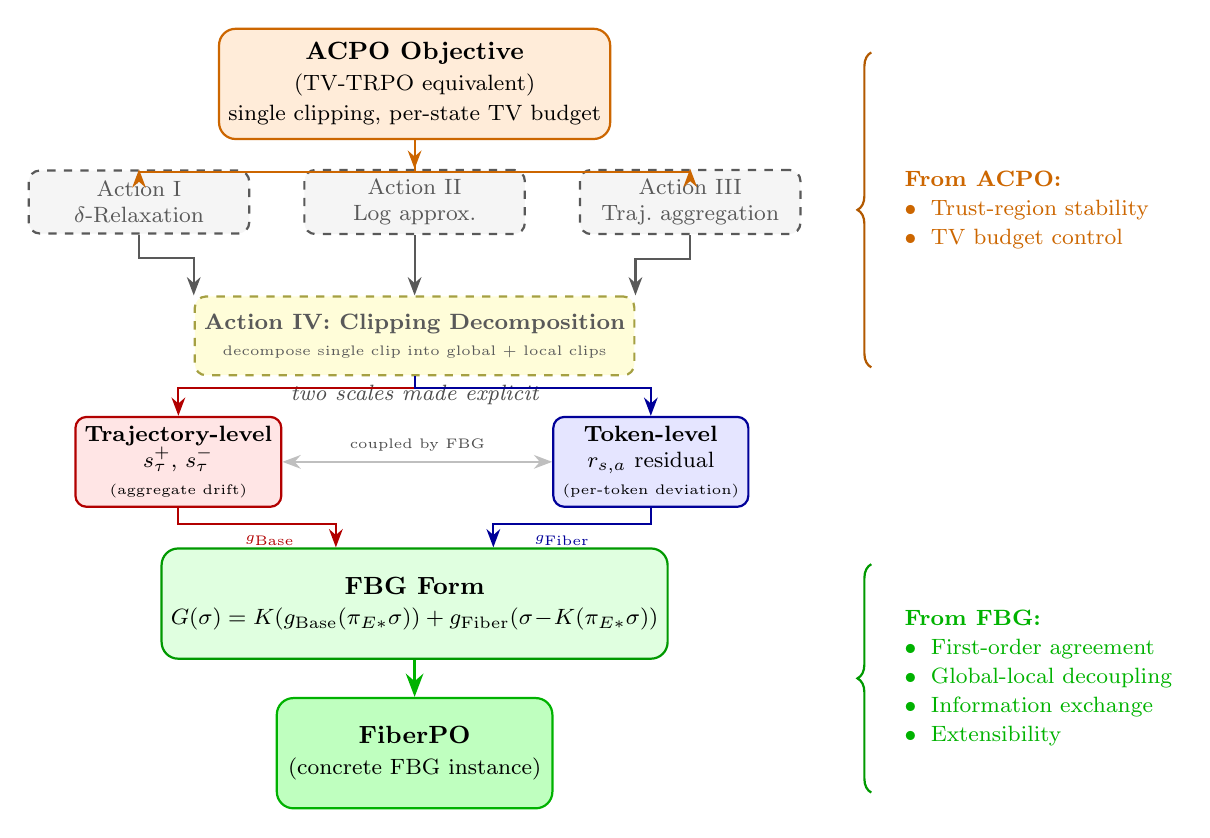}
\caption{Derivation roadmap from APC-Obj to FiberPO.
The APC-Obj objective uses a single clipping with cross-token interaction to enforce a per-state TV budget; the two scales of stability information (trajectory-level aggregate ratios $s_\tau^\pm$ and per-token ratios $r_{s,a}$) are implicitly entangled within this clipping.
Through four successive relaxations, the key step (Derivation~IV) decomposes the single clipping to make the two scales explicit, yielding separate global and local gates that fit naturally into the FBG form.
The resulting FiberPO inherits trust-region stability from APC-Obj and structural guarantees (first-order agreement, global-local decoupling, information exchange, extensibility) from the FBG framework.}
\label{fig:apc_obj_to_fbg}
\end{figure}

\subsubsection{Derivation I: $\delta$-Relaxation}
\medskip\noindent

\label{app:fiberpo_derivation1}

Replace the vanishing $\delta^{\text{(APC-Obj)}}$ with a positive tunable hyperparameter $\delta$.
Since $\delta^{(\rm APC\text{-}Obj)} = \delta^{(\rm TRPO)} = 0$ at $\gamma = 1$ (Theorem~\ref{thm:trpo_vanishing}), this relaxation trades the exact TRPO trust-region guarantee for a controllable approximation: the resulting objective no longer enforces the classical bound, but $\delta$ directly quantifies the departure, and the TV-based stability structure is retained as a design guide (Remark~\ref{rem:trust_region_design}).

\subsubsection{Derivation II: Logarithmic Approximation}
\label{app:fiberpo_derivation2}

Apply the logarithmic approximation $\log r_i\approx r_i-1$ and $1+\operatorname{clip}(\cdots)\approx \exp(\operatorname{clip}(\cdots))$.
This provides two benefits:
\begin{enumerate}
    \item \textbf{Positivity by construction.} Policy ratios must satisfy $r > 0$. The map $\log r \in \mathbb R \mapsto \exp(\cdot) > 0$ guarantees this automatically: all subsequent gating operations---clipping, piecewise-linear gates $g^{\rm agg}$, additions---can be designed as unconstrained operations on $\mathbb R$ and exponentiated back, yielding valid positive ratios without explicit positivity constraints. This opens opportunity to many gating function design, not just a map from positive real line to positive real line.
    \item \textbf{Better algebraic properties} for subsequent simplifications, in particular the additive decomposition of the clipping bound in Derivation~IV.
    \item \textbf{Natural asymmetry.} The common asymmetry of clipping parameters on ratios (where $\epsilon_-<\epsilon_+$) is naturally captured: replacing $\operatorname{clip}(r;\,1-\epsilon_{-},\,1+\epsilon_+)$ by $\exp(\operatorname{clip}(\log r,\,\pm\epsilon))$, the symmetric ``log clipping'' naturally exhibits this asymmetry.
\end{enumerate}
Under this approximation, the cross-token TV budget terms $|r_{s,a'}-1|$ in the APC-Obj objective~\eqref{eq:apc_obj_main_obj} become $|\log r_{s,a'}|$, so the per-state estimated TV divergence $\hat{D}_{\rm TV}(s)$ is approximated by the mean absolute log-ratio, which we call the \emph{absolute log divergence}.
This identification is used in Derivation~IV to interpret the clipping bound as a TV-budget constraint expressed in log space.

\noindent The objective becomes:
\[
    {\hat J}(\theta |\theta_{\rm old}) =
     \frac{1}{T}\sum_{(s,a,\mathcal I)\in\mathcal E} \exp\!\left(  \operatorname{clip}\bigl(\log r_{s,a},\; \pm \bigl(T_s\delta  - \!\!\!\!\sum_{\substack{(s,a',\mathcal I')\in\mathcal E \\ (s,a',\mathcal I')\ne (s,a,\mathcal I)}}\!\!\!\!|\log r_{s,a'}|\bigr)\bigr) \right)  \hat A_{s,a}.
\]

\subsubsection{Derivation III: Sequence-Level Aggregation}
\label{app:fiberpo_derivation3}

Two changes restrict the cross-token coupling to within each trajectory.
First, the summation inside the clip bound changes from all sampled tokens $\sum_{(s,a',\mathcal I')\in\mathcal E,\,(a',\mathcal I')\ne(a,\mathcal I)}$ to only those sharing the same trajectory $\sum_{(s',a',\mathcal I')\in\mathcal E,\,(a',\mathcal I')\ne(a,\mathcal I),\,\tau(\mathcal I')=\tau(\mathcal I)}$.
Second, the per-state TV budget $T_s\delta$ is replaced by a per-trajectory budget $T_{\tau(\mathcal I)}\delta$, distributing the total budget proportionally to trajectory length.
The objective becomes:
\begin{multline*}
    {\hat J}(\theta |\theta_{\rm old}) =
     \frac{1}{|\mathrm{Tj}^{\theta_{\rm old}}|}\sum_{(s,a,\mathcal I)\in\mathcal E} \frac{1}{T_{\tau(\mathcal I)}} \\
     \times\exp\!\left(  \operatorname{clip}\bigl(\log r_{s,a},\; \pm \bigl(T_{\tau(\mathcal I)}\delta  - \!\!\!\!\sum_{\substack{(s',a',\mathcal I')\in\mathcal E \\ (s',a',\mathcal I')\ne (s,a,\mathcal I),\; \tau(\mathcal I')=\tau(\mathcal I)}}\!\!\!\!|\log r_{s',a'}|\bigr)\bigr) \right)  \hat A_{s,a},
\end{multline*}
where $\tau(\mathcal I)$ is the trajectory given by the extra information $\mathcal I$ (see Definition~\ref{def:rgf}) associated with the pair $(s,a,\mathcal I)\in\mathcal E$, and $T_{\tau(\mathcal I)}$ is its length.
As with GRPO, the $1/T_{\tau(\mathcal I)}$ weighting also rewards shorter responses.

\subsubsection{Derivation IV: Clipping Decomposition into Fiber Bundle Gating Form}
\label{app:fiberpo_derivation4}

This is the most important and distinctive derivation step compared to the derivation of PPO, GRPO, and GSPO.
The goal is to decompose the current clipping, which enforces only a trajectory-level TV budget ($\sum_{(s',a',\mathcal I')\in\mathcal E,\;\tau(\mathcal I')=\tau(\mathcal I)}|\log r_{s',a'}|\leq T_{\tau(\mathcal I)}\delta$, i.e.\ the sum of token log-ratios within the same trajectory is bounded), into explicit local and global clipping.
We focus on the clipping term inside the objective:
\[
    \operatorname{clip}\bigl(\log r_{s,a},\; \pm \bigl(T_{\tau(\mathcal I)}\delta  - \!\!\!\!\sum_{\substack{(s',a',\mathcal I')\ne (s,a,\mathcal I) \\ \tau(\mathcal I')=\tau(\mathcal I)}}\!\!\!\!|\log r_{s',a'}|\bigr)\bigr).
\]

\paragraph{Introducing aggregate ratios.}
Define the positive and negative (trajectory) aggregate ratios:
\[
    \log (s^+_\tau) :=\frac{1}{T_\tau}\sum_{t=0}^{T_\tau-1} \max(\log(r_{s_t(\tau),a_t(\tau)}),0),
    \qquad
    \log (s^-_\tau) :=\frac{1}{T_\tau}\sum_{t=0}^{T_\tau-1} \max(\log(r^{-1}_{s_t(\tau),a_t(\tau)}),0),
\]
and denote $l_{s,a}:=\operatorname{sign}(\log r_{s,a})$.
The clipping term can be re-expressed as:
\[
    \operatorname{clip}\bigl(\log r_{s,a},\; \pm (T_{\tau}\delta  - T_\tau\log s^+_\tau-T_\tau\log s^-_\tau+|\log r_{s,a}|)\bigr).
\]

\paragraph{Negative skew of unsplit $\hat{\bar{D}}_{\rm TV}^{(\tau)}$ gating.}
Using a single, unsplit estimate of $\hat{\bar{D}}_{\rm TV}^{(\tau)}$ (Appendix~\ref{app:divergences}) as the global gating quantity introduces a systematic negative skew.
Concretely, the absolute log divergence $\log s^+_\tau + \log s^-_\tau = (1/T_\tau)\sum_{t}|\log r_{s_t,a_t}|$ approximates the trajectory-level TV distance in log-ratio coordinates (Derivation~II, Section~\ref{app:fiberpo_derivation2}).
By definition, $\log s^+_\tau, \log s^-_\tau \geq 0$, so their sum is always nonnegative.
In the FBG framework (Definition~\ref{def:fbg}), local-global decoupling subtracts the pushed-forward aggregate $K(\pi_{E*}\sigma)$ from every fiber.
When this aggregate is the unsplit (nonneg.)\ TV estimate, the subtraction removes a nonnegative quantity from every per-token log-ratio, imparting a systematic negative bias to all fiber residuals.
The remedy is to split by sign into the positive channel $s^+_\tau$ and negative channel $s^-_\tau$: a token with $\log r_{s,a} > 0$ is then gated only by $\log s^+_\tau$, and one with $\log r_{s,a} < 0$ only by $\log s^-_\tau$, so the same-sign component is subtracted rather than the combined nonnegative total.

\paragraph{Per-channel budgets and the $C^- < C^+$ condition.}
After splitting, each channel is gated independently, eliminating the sign-mixing skew.
However, an intrinsic asymmetry between the two channels remains.
Observe that
\[
    \log s^-_\tau - \log s^+_\tau \;=\; -\frac{1}{T_\tau}\sum_{t=0}^{T_\tau-1}\log r_{s_t(\tau),a_t(\tau)},
\]
which is a trajectory-averaged Monte Carlo estimate of the reverse KL divergence $\bar{D}^{\rm (tj)}_{\rm KL}(\pi_{\theta_{\rm old}}\|\pi_\theta)$.
Since this quantity is nonnegative, we have $\log s^-_\tau \geq \log s^+_\tau$ in expectation, so the negative channel generically carries more aggregate drift.
Imposing a symmetric budget $\delta/2$ on each channel therefore insufficiently constrains the faster-drifting negative channel: because $\log s^-$ grows faster, it reaches the shared threshold $\delta/2$ sooner and dominates the base weight $w^{\rm base}_\tau$ (Eq.~\ref{eq:fiberpo_gating_decomp}), skewing the effective trust region toward the negative direction.
To compensate, we decompose $\delta$ into per-channel budgets $C^+$ and $C^-$ satisfying $C^+ + C^- = \delta$, $C^+, C^- > 0$, with the recommendation $C^- < C^+$: the tighter negative budget ensures that the faster-drifting channel reaches rollback earlier, preventing the overall gated ratio $\mathcal G(r_\bullet)_i = w^{\rm base}_\tau \cdot \tilde{r}_i^{\rm fiber}$ (Eq.~\ref{eq:fiberpo_gating_decomp}) from being biased toward policy-ratio-decreasing updates.
The symmetric case $C^+ = C^- = \delta/2$ is recovered as a special case.

We decompose the above into a stricter clipping by clipping positive and negative log ratios separately:
\begin{align*}
    \text{For }\log r_{s,a}> 0\;(l_{s,a}=1):\quad &
    \operatorname{clip}\bigl(\log r_{s,a},\; \pm (T_{\tau}C^+  - T_\tau\log s^+_\tau +\log r_{s,a})\bigr), \\[0.5em]
    \text{For }\log r_{s,a}< 0\;(l_{s,a}=-1):\quad &
    \operatorname{clip}\bigl(\log r_{s,a},\; \pm (T_{\tau}C^-  - T_\tau\log s^-_\tau -\log r_{s,a})\bigr).
\end{align*}
Or equivalently:
\[
    \operatorname{clip}\bigl(\log r_{s,a},\; \pm (T_{\tau}C^{(l_{s,a})}  - T_\tau\log s^{(l_{s,a})}_\tau +l_{s,a}\log r_{s,a})\bigr).
\]

\paragraph{Decomposing aggregate and individual clipping.}
The channel-split clipping above still acts jointly on the aggregate and per-token ratios.
To introduce \emph{local} clipping behavior that constrains each token individually and to further structure the objective into the fiber and base components required by the FBG framework, we decompose the clipping of the aggregate ratio explicitly out of the clipping of individual policy ratios, yielding a stricter clipping (absolute value strictly less than the above), where $l_{s,a}:= \operatorname{sign}(\log r_{s,a})$:
\[
    l_{s,a}\operatorname{clip}(l_{s,a}\log r_{s,a}-\log s_\tau^{(l_{s,a})},\,\pm \epsilon)
    \;+\;
    \operatorname{clip}\bigl(l_{s,a}\log s^{(l_{s,a})}_\tau,\;\pm (T_\tau C^{(l_{s,a})}-T_\tau\log s^{(l_{s,a})}_\tau -( \epsilon - l_{s,a}\log r_{s,a}) )\bigr).
\]
The decomposition naturally introduces an extra constant $\epsilon$, representing the clipping bound for individual policy ratios. 
Now we are ready to define an intermediate object, object $g^{\rm agg}$, which will be used to define $g_{{\rm Base}}$ in FBG, the above is equivalent to:
\begin{equation}\label{eq:fiberpo_decomposed}
    l_{s,a}\,\operatorname{clip}(l_{s,a}\log r_{s,a}-\log s_\tau^{(l_{s,a})},\,\pm \epsilon)
    \;+\;
    l_{s,a}\,g^{\rm agg}\!\left( \log s^{(l_{s,a})}_\tau ,\; \frac{T_\tau C^{(l_{s,a})}-(\epsilon-l_{s,a}\log r_{s,a})}{T_\tau +1},\; T_\tau\right),
\end{equation}
where
\begin{equation}\label{eq:gagg}
    g^{\rm agg}(x,C,k):= \left\{\begin{array}{ll}   x&\text{if }|x|\leq C\\[0.5em]  \operatorname{sign}(x)(k+1)C-kx & \text{if }C<|x|<(1+k^{-1})C  \\[0.5em]  0 &\text{otherwise}   \end{array}\right.
\end{equation}

Removing the nonnegative factor $l_{s,a}\log r_{s,a}$ from $g^{\rm agg}$, and noting that $T_\tau$ is typically large in the LLM context so that $\frac{\epsilon}{T_\tau+1}$ quickly vanishes and $\frac{T_\tau C^{(l_{s,a})}}{T_\tau + 1}\approx C^{(l_{s,a})}$, we obtain the simplified clipping term:
\begin{equation}\label{eq:fiberpo_L3}
    l_{s,a}\,\operatorname{clip}(l_{s,a}\log r_{s,a}-\log s_\tau^{(l_{s,a})},\,\pm \epsilon)
    \;+\;
    l_{s,a}\,g^{\rm agg}\!\left( \log s^{(l_{s,a})}_\tau ,\; C^{(l_{s,a})},\; T_\tau\right).
\end{equation}

\begin{remark}[Three-regime structure of $g^{\rm agg}$]\label{rem:gagg_regimes}
The function $g^{\rm agg}$ that emerges from this decomposition has three regimes: pass-through when $|\log s_\tau^{(l)}| \leq C^{(l)}$, linear rollback when $C^{(l)} < |\log s_\tau^{(l)}| < (1+T_\tau^{-1})C^{(l)}$, and zeroing beyond that threshold.
See Figure~\ref{fig:fiber_weight} for a visualization.
\end{remark}

\paragraph{Fiber Bundle Gating objects.}
The two terms in~\eqref{eq:fiberpo_L3} correspond to fiber and base gating components respectively; we now identify the FBG objects that realize this decomposition.
By defining $E=\bar{\mathcal X}\times\{-1,+1\}$, $(s,a,\tau,t;\,l)\in E$,
\[
    \mathcal F(r_{\bullet})_{(s,a,\tau,t;\,l)}:= \frac{1}{T_\tau}\max(l\log r_{s,a},0)=\mu_{\bar{\mathcal X}}^{(\tau)}\max(l\log r_{s,a},0)\in \mathcal D^1_{(s,a,\tau,t;\,l)}E\simeq \mathbb R
\]
as local information, and $B=\mathrm{Tj}^{\theta_{\rm old}}\times\{-1,+1\}$ with the projection $\pi_E:E\to B$, $\pi_E(s,a,\tau,t;\,l):=(\tau,l)$, one finds that
\[
    \log s^{(l)}_\tau = \pi_{E*}\mathcal F(r_\bullet)\big|_{(\tau,l)\in B}.
\]
The reflecting Markov kernel is defined as $K(\cdot|(\tau,l)):=\mu_{\bar{\mathcal X}}^{(\tau)}$ (see Appendix~\ref{app:estimation}), so that $K((s,a,\tau,t;\,l)|(\tau,l)):=\mu_{\bar{\mathcal X}}^{(\tau)}(\{(s,a,\tau,t)\})=\frac{1}{T_\tau}$.
One can verify that $\pi_{E*} K(p_B)=p_B=\mathrm{id}_{\mathbf B}(p_B)$.

The fibration recovery map is defined as
\[
    \mathcal R(\sigma_\bullet)_{(s,a,\tau,t)}:= \exp\bigl(T_\tau (\sigma_{(s,a,\tau,t;\,+1)}-\sigma_{(s,a,\tau,t;\,-1)})\bigr).
\]
The atomic gating functions are:
\[
    g_{{\rm Fiber},i,p_B}(\sigma_i):= \operatorname{clip}\!\left(\sigma_i,\,\pm \frac{\epsilon}{T_{\tau(i)}}\right),\quad i\in E,
    \qquad
    g_{{\rm Base},(\tau,l)}(\sigma_{(\tau,l)}):= g^{\rm agg}\!\left(\sigma_{(\tau,l)},\,C^{(l)},\,T_{\tau}\right),\quad (\tau,l)\in B.
\]

\paragraph{Verification.}
The expression~\eqref{eq:fiberpo_L3} fits exactly into the fiber bundle gating formalism:
\begin{align*}
    &\operatorname{clip}(l_{s,a}\log r_{s,a}-\log s_\tau^{(l_{s,a})},\,\pm \epsilon) +  g^{\rm agg}\!\left( \log s^{(l_{s,a})}_\tau ,\; C^{(l_{s,a})},\; T_\tau\right)\\
    &\quad= g_{{\rm Fiber},i,\pi_{E*}\sigma}\bigl(\sigma-K(\pi_{E*}\sigma)\bigr)_{i}+K\bigl(g_{{\rm Base},(\tau,l_{s,a})}(\pi_{E*}\sigma)\bigr)_i\\
    &\quad\equiv G(\sigma)_i,
\end{align*}
where $\sigma\equiv\mathcal F(r_\bullet)\in\mathbf E$ and $i\equiv (s,a,\tau,t;\,l_{s,a})\in E$.

The clipping-decomposed objective is therefore in standard Fiber Bundle Gating form:
\[
    \hat J_{(G;\mathcal F,\mathcal R)}(\theta|\theta_{\rm old}):=\sum_{\tau\in \mathrm{Tj}^{\pi_{\theta_{\rm old}}}}\left[ \sum_{t=0}^{T_\tau-1}  {\frac{1}{T_\tau}\, \mathcal R\circ G\circ\mathcal F(r_{\bullet})_{s_t,a_t}\; \hat A_{s_t,a_t} }  \right].
\]

\subsubsection{The FiberPO Objective}
\label{app:fiberpo_objective}

By fully expanding the FBG form, we obtain the FiberPO objective.

\begin{definition}[Fiber-Aware Clipping Policy Optimization (FiberPO) objective]\label{def:fiberpo}
In standard RGF form:
\begin{multline}\label{eq:fiberpo_rgf}
    {\hat J^{{\text{FiberPO}}}}(\theta |\theta_{\rm old}) = \sum_{(s,a,\tau,t)\in\bar{\mathcal X}}\Biggl[ {\frac{1}{{|\mathrm{Tj}^{\theta_{\rm old}}|}}{\frac{1}{{T_\tau}}}  \cdot{\frac{\exp\circ \;g^{\rm agg}( \log s^{+}_\tau , C^+, T_\tau)}{\exp\circ \;g^{\rm agg}( \log s^{-}_\tau , C^-, T_\tau)}}} \\
    \times\frac{
    {\text{logclip}}\bigl({{\left( {{{(s_\tau^{(l_{s,a})})}^{ - l_{s,a}}}r_{s,a}} \right)}},{\epsilon}\bigr)
    }{
    {\text{logclip}}\bigl({{\left( {{{(s_\tau^{(-l_{s,a})})}^{  -l_{s,a}}}} \right)}},{\epsilon}\bigr)
    } \cdot \hat A_{s,a} \Biggr],
\end{multline}

where:
\begin{itemize}
    \item $l_{s,a}:= \operatorname{sign}(\log r_{s,a})$.
    \item $\log (s^+_\tau) :=\frac{1}{T_\tau}\sum_{t=0}^{T_\tau-1} \max(\log(r_{s_t(\tau),a_t(\tau)}),0)$,
    \item $\log (s^-_\tau) :=\frac{1}{T_\tau}\sum_{t=0}^{T_\tau-1} \max(\log(r^{-1}_{s_t(\tau),a_t(\tau)}),0)$,
    \item $\operatorname{logclip}(x,a):=\exp(\operatorname{clip}(\log (x),\,\pm a))$,
    \item $g^{\rm agg}(x,C,k)$ is as defined in~\eqref{eq:gagg},
    \item $\mathcal E=\bar{\mathcal X}$ is the augmented state-action pair space (see Appendix~\ref{app:estimation}),
    \item $C^+$ and $C^-$ are the per-channel trust-region budgets with $C^+ + C^- = \delta$; $\epsilon$ and $\delta$ are hyperparameters.
\end{itemize}
\end{definition}

\subsection{FiberPO Objective in LLM Notation}
\label{app:fiberpo_llm_notation}

We restate the FiberPO objective (Definition~\ref{def:fiberpo_main}) using standard LLM notation $(g,j,i)$, where $g$ indexes the query group, $j$ indexes the $j$-th trajectory within group $g$, and $i$ indexes the $i$-th token within trajectory $j$.
The translation from RL notation is $\tau,t \mapsto g,j,i := s_0(\tau),\tau,t$ (see Appendix~\ref{app:llm_rl_notation}).

The FiberPO objective in LLM notation is:
\begin{multline}\label{eq:fiberpo_main_llm}
    {\hat J^{{\text{FiberPO}}}}(\theta |\theta_{\rm old}) = {\mathbb E}_{g\sim{\mathcal D}}\Biggl[ {\frac{1}{{|\mathrm{Tj}^{\theta_{\rm old}}(g)|}}\sum_{j\in\mathrm{Tj}^{\theta_{\rm old}}(g)} {\frac{1}{{T_j^{(g)}}}}} \sum_{i = 1}^{T_j^{(g)}} {\frac{\exp\circ \;g^{\rm agg}( \log s^{(g)+}_{j} , C^+, T_j^{(g)})}{\exp\circ \;g^{\rm agg}( \log s^{(g)-}_{j} , C^-, T_j^{(g)})}} \\
    \times\frac{
    {\text{logclip}}\bigl({{\left( {{{(s_j^{(g,\,l_{j,i}^{(g)})})}^{ - l_{j,i}^{(g)}}}r_{j,i}^{(g)}} \right)}},{\epsilon}\bigr)
    }{
    {\text{logclip}}\bigl({{\left( {{{(s_j^{(g,\,-l_{j,i}^{(g)})})}^{  -l_{j,i}^{(g)}}}} \right)}},{\epsilon}\bigr)
    } \cdot \hat A_{j,i}^{(g)} \Biggr],
\end{multline}
where $l_{j,i}^{(g)}:=\operatorname{sign}(\log r_{j,i}^{(g)})$, and the constituent quantities are:
\begin{itemize}
    \item Positive and negative trajectory aggregate ratios:
    \[
        \log s^{(g)+}_{j} :=\frac{1}{T^{(g)}_j}\sum_{i=0}^{T_j^{(g)}-1} \max(\log r^{(g)}_{j,i},\,0), \qquad
        \log s^{(g)-}_{j} :=\frac{1}{T^{(g)}_j}\sum_{i=0}^{T_j^{(g)}-1} \max(-\log r^{(g)}_{j,i},\,0).
    \]
    \item Log-clipping function: $\operatorname{logclip}(x,\epsilon):=\exp(\operatorname{clip}(\log x,\,\pm \epsilon))$.
    \item Aggregate gating function: $g^{\rm agg}(x,C,k)$ as defined in~\eqref{eq:gagg_main}.
    \item $C^+$ and $C^-$ are the per-channel trust-region budgets with $C^+ + C^- = \delta$.
    \item $T_j^{(g)}$ is the sequence length of the $j$-th trajectory in group $g$; $\epsilon$ and $\delta$ are hyperparameters.
\end{itemize}

\subsection{FiberPO-Domain: FGH Objects}
\label{app:fiberpo_domain_fgh}

The FiberPO-Domain objective~\eqref{eq:domain_fiberpo} can equivalently be written
in the FGH gating form~\eqref{eq:fgh_objective}:
\begin{equation}\label{eq:domain_fiberpo_rgf}
\hat{J}^{\text{FiberPO-Domain}}(\theta|\theta_{\rm old})
= \sum_{i \in \bar{\mathcal{X}}} \frac{1}{|\mathrm{Tj}^{\theta_{\rm old}}|}
\,\frac{1}{T_\tau}\; \mathcal{R} \circ G \circ \mathcal{F}(r_\bullet)_i \;\hat{A}_i,
\end{equation}
where $\mathcal{G}(r_\bullet)_i = w_i^{\rm Base} \cdot
\tilde{r}_i^{\rm Fiber} = \mathcal{R} \circ G \circ
\mathcal{F}(r_\bullet)_i$ corresponds to the FGH case $n = 3$.
The constituent FGH objects are:

\begin{enumerate}
    \item \textbf{Fibration hierarchy ($n{=}3$).}
    The four strata are:
    \begin{alignat*}{2}
    B_0 &:= \mathrm{Domain} \times \{-1,+1\}, &\qquad
    B_1 &:= \mathrm{PromptGroup} \times \{-1,+1\}, \\
    B_2 &:= \mathrm{Tj}^{\theta_{\rm old}} \times \{-1,+1\}, &\qquad
    B_3 \equiv E &:= \bar{\mathcal{X}} \times \{-1,+1\},
    \end{alignat*}
    with fibrations $\pi_3(s,a,\tau,t;\, l) := (\tau, l)$,
    $\pi_2(\tau, l) := (g_\tau, l)$,
    $\pi_1(g, l) := (D_g, l)$.

    \item \textbf{Reflecting Markov kernels.}
    \begin{align*}
    K_0(g,l \mid D,l') &:= \frac{|\mathrm{Tj}^{\theta_{\rm old}}(g)|}
    {|\mathrm{Tj}^{\theta_{\rm old}}(D)|}\,
    \mathbb{I}_{g \in D}\;\mathbb{I}_{l=l'}, \\
    K_1(\tau,l \mid g,l') &:= \frac{1}{|\mathrm{Tj}^{\theta_{\rm old}}(g)|}\,
    \mathbb{I}_{\tau \in g}\;\mathbb{I}_{l=l'}, \\
    K_2(i,l \mid \tau,l') &:= \frac{1}{T_\tau}\,
    \mathbb{I}_{i \in \tau}\;\mathbb{I}_{l=l'}.
    \end{align*}

    \item \textbf{Atomic gating functions.}
    At the domain level ($k{=}0$):
    \[
    g_{0,(D,l)}(\sigma_0) := {|\mathrm{Tj}^{\theta_{\rm old}}(D)|}\;
    g^{\rm agg}\!\left(\frac{1}{|\mathrm{Tj}^{\theta_{\rm old}}(D)|}\,(\sigma_0)_{(D,l)},\;
    C^{(l)},\; T_D\right).
    \]
    At the prompt group level ($k{=}1$):
    \[
    g_{1,p_{<1},(g,l)}(\sigma_1) := {|\mathrm{Tj}^{\theta_{\rm old}}(g)|}\;
    g^{\rm agg}\!\left(\frac{1}{|\mathrm{Tj}^{\theta_{\rm old}}(g)|}\,(\sigma_1)_{(g,l)},\;
    C^{(l)},\; T_g\right).
    \]
    At the trajectory level ($k{=}2$):
    \[
    g_{2,p_{<2},(\tau,l)}(\sigma_2) := g^{\rm agg}\!\left((\sigma_2)_{(\tau,l)},\;
    C^{(l)},\; T_\tau\right).
    \]
    At the token level ($k{=}3$):
    \[
    g_{3,p_{<3},(i,l)}(\sigma) := \frac{1}{T_{\tau_i}}\,
    \operatorname{clip}\!\left(T_{\tau_i}\,\sigma_{(i,l)},\; \epsilon\right).
    \]

    \item \textbf{Fibration decomposition and recovery.}
    $\mathcal{F}(r_\bullet)_{i,l} := \frac{\mathbb{I}_{l_i=l}}{T_{\tau_i}}\,
    l_i \log r_i$, \quad
    $\mathcal{R}(\sigma)_i := \exp\!\left(T_{\tau_i}\,
    (\sigma_{i,+1} - \sigma_{i,-1})\right)$.
\end{enumerate}

%% file: references.bib
@article{yu2025dapo,
  title={Dapo: An open-source llm reinforcement learning system at scale},
  author={Yu, Qiying and Zhang, Zheng and Zhu, Ruofei and Yuan, Yufeng and Zuo, Xiaochen and Yue, Yu and Dai, Weinan and Fan, Tiantian and Liu, Gaohong and Liu, Lingjun and others},
  journal={arXiv preprint arXiv:2503.14476},
  year={2025}
}

@article{shao2024deepseekmath,
  title={Deepseekmath: Pushing the limits of mathematical reasoning in open language models},
  author={Shao, Zhihong and Wang, Peiyi and Zhu, Qihao and Xu, Runxin and Song, Junxiao and Bi, Xiao and Zhang, Haowei and Zhang, Mingchuan and Li, YK and Wu, Yang and others},
  journal={arXiv preprint arXiv:2402.03300},
  year={2024}
}

@article{zheng2025group,
  title={Group sequence policy optimization},
  author={Zheng, Chujie and Liu, Shixuan and Li, Mingze and Chen, Xiong-Hui and Yu, Bowen and Gao, Chang and Dang, Kai and Liu, Yuqiong and Men, Rui and Yang, An and others},
  journal={arXiv preprint arXiv:2507.18071},
  year={2025}
}

@inproceedings{schulman2015trust,
  title={Trust region policy optimization},
  author={Schulman, John and Levine, Sergey and Abbeel, Pieter and Jordan, Michael and Moritz, Philipp},
  booktitle={International Conference on Machine Learning},
  pages={1889--1897},
  year={2015},
  organization={PMLR}
}

@article{kim2022adaptive,
  title={Adaptive discount factor for deep reinforcement learning in continuing tasks with uncertainty},
  author={Kim, MyeongSeop and Kim, Jung-Su and Choi, Myoung-Su and Park, Jae-Han},
  journal={Sensors},
  volume={22},
  number={19},
  pages={7266},
  year={2022},
  publisher={MDPI}
}

@article{schulman2017proximal,
  title={Proximal policy optimization algorithms},
  author={Schulman, John and Wolski, Filip and Dhariwal, Prafulla and Radford, Alec and Klimov, Oleg},
  journal={arXiv preprint arXiv:1707.06347},
  year={2017}
}

@inproceedings{kakade2002approximately,
  title={Approximately optimal approximate reinforcement learning},
  author={Kakade, Sham and Langford, John},
  booktitle={International Conference on Machine Learning},
  pages={267--274},
  year={2002}
}

@article{ouyang2022training,
  title={Training language models to follow instructions with human feedback},
  author={Ouyang, Long and Wu, Jeffrey and Jiang, Xu and Almeida, Diogo and Wainwright, Carroll and Mishkin, Pamela and Zhang, Chong and Agarwal, Sandhini and Slama, Katarina and Ray, Alex and others},
  journal={Advances in Neural Information Processing Systems},
  volume={35},
  pages={27730--27744},
  year={2022}
}

@inproceedings{wang2020truly,
  title={Truly Proximal Policy Optimization},
  author={Wang, Yuhui and He, Hao and Tan, Xiaoyang},
  booktitle={Proceedings of the 35th Conference on Uncertainty in Artificial Intelligence},
  pages={113--122},
  year={2020},
  organization={PMLR}
}
